\documentclass[twoside]{article}

\usepackage{aistats2020}
\usepackage{natbib}

\usepackage[T1]{fontenc}

\usepackage{amsmath}
\usepackage{amsfonts}
\usepackage{amssymb}
\usepackage{amsthm}
\usepackage{mathtools}

\newcommand{\BlackBox}{\rule{1.5ex}{1.5ex}}  

\newtheorem{theorem}{Theorem}
\newtheorem{lemma}{Lemma}

\newtheorem{definition}{Definition}

\usepackage{enumitem}
\usepackage{booktabs}
\usepackage{url}
\usepackage{bbold}
\usepackage{dsfont}
\usepackage{thmtools,thm-restate}
\usepackage{subcaption}
\usepackage{graphicx}
\usepackage{algorithm}
\usepackage{algpseudocode}
\usepackage{hyperref}
\usepackage{bm}

\usepackage{tikz}
\usetikzlibrary{decorations.pathreplacing,calc}
\newcommand{\tikzmark}[1]{\tikz[overlay,remember picture] \node (#1) {};}

\newcommand*{\AddNote}[4]{%
    \begin{tikzpicture}[overlay, remember picture]
        \draw [decoration={brace,amplitude=0.5em},decorate,thick,blue]
            ($(#3)!(#1.north)!($(#3)-(0,1)$)$) --  
            ($(#3)!(#2.south)!($(#3)-(0,1)$)$)
                node [align=center, text width=2.5cm, pos=0.5, anchor=west] {#4};
    \end{tikzpicture}
}%
\newcommand{\CommentX}[1]{\unskip~~#1~}

\usepackage{enumitem}
\setitemize{noitemsep,topsep=0pt,parsep=0pt,partopsep=0pt}

\usepackage[colorinlistoftodos, textwidth=22mm, shadow]{todonotes}
\definecolor{blued}{RGB}{70,197,221}
\definecolor{pearOne}{HTML}{2C3E50}
\definecolor{pearTwo}{HTML}{A9CF54}
\definecolor{pearTwoT}{HTML}{C2895B}
\definecolor{pearThree}{HTML}{FF69B4}
\colorlet{titleTh}{pearOne}
\colorlet{bull}{pearTwo}
\definecolor{pearcomp}{HTML}{B97E29}
\definecolor{pearFour}{HTML}{588F27}
\definecolor{pearFith}{HTML}{ECF0F1}
\definecolor{pearDark}{HTML}{2980B9}
\definecolor{pearDarker}{HTML}{F330DB}
\hypersetup{
	colorlinks,
	citecolor=pearDark,
	linkcolor=pearThree,
	urlcolor=pearDarker}

\DeclareMathOperator*{\argmax}{arg\,max}
\DeclareMathOperator*{\argmin}{arg\,min}

\newcommand\ddfrac[2]{\frac{\displaystyle #1}{\displaystyle #2}}

\newcommand*\diff{\mathop{}\!\mathrm{d}}

\renewcommand{\d}[1]{\ensuremath{\operatorname{d}\!{#1}}}

\newcommand{\floor}[1]{\left\lfloor#1\right\rfloor}

\newcommand{\R}{\mathbb{R}}

\newcommand{\NN}{{\mathbb N}}
\newcommand{\1}{\mathds{1}}

\newcommand{\EE}[1]{\mathbb{E}\left[#1\right]}

\newcommand{\PP}[1]{\mathbb{P}\left[#1\right]}

\newcommand{\expp}[1]{\exp\left\{#1\right\}}

\newcommand*{\eqdef}{\triangleq}

\newcommand{\cA}{\mathcal{A}}
\newcommand{\cB}{\mathcal{B}}
\newcommand{\cC}{\mathcal{C}}

\newcommand{\cF}{\mathcal{F}}

\newcommand{\cH}{\mathcal{H}}
\newcommand{\cI}{\mathcal{I}}

\newcommand{\cN}{\mathcal{N}}
\newcommand{\cO}{\mathcal{O}}

\newcommand{\cU}{\mathcal{U}}

\newcommand{\bT}{{\bf T}}

\newcommand{\bY}{{\bf Y}}

\renewcommand{\epsilon}{\varepsilon}
\renewcommand{\hat}{\widehat}
\renewcommand{\tilde}{\widetilde}
\renewcommand{\bar}{\overline}

\newcommand{\btheta}{{\boldsymbol \theta}}

\newcommand{\bmu}{{\boldsymbol \mu}}

\newcommand{\bomega}{{\boldsymbol \omega}}

\newcommand{\nothere}[1]{}

\usepackage{xspace}

\newcommand{\UCB}{\texttt{UCB}\xspace}

\newcommand{\SE}{\texttt{SuccessiveElimination}\xspace}

\newcommand{\LUCB}{\texttt{LUCB}\xspace}

\newcommand{\DT}{\texttt{D-Tracking}\xspace}

\newcommand{\EGE}{\texttt{ExponentialGapElimination}\xspace}
\newcommand{\LIL}{\texttt{lil'UCB}\xspace}

\newcommand{\UGapE}{\texttt{UGapE}\xspace}

\newcommand{\TTTS}{\texttt{TTTS}\xspace}

\newcommand{\TCC}{\texttt{T3C}\xspace}
\newcommand{\TTPS}{\texttt{TTPS}\xspace}
\newcommand{\TTVS}{\texttt{TTVS}\xspace}
\newcommand{\TTEI}{\texttt{TTEI}\xspace}
\newcommand{\BC}{\texttt{BC}\xspace}

\definecolor{misocolor}{rgb}{0.16,0.27,0.86}

\definecolor{graphicbackground}{rgb}{0.96,0.96,0.8}
\definecolor{rouge1}{RGB}{226,0,38}  
\definecolor{orange1}{RGB}{243,154,38}  
\definecolor{jaune}{RGB}{254,205,27}  
\definecolor{blanc}{RGB}{255,255,255} 
\definecolor{rouge2}{RGB}{230,68,57}  
\definecolor{orange2}{RGB}{236,117,40}  
\definecolor{taupe}{RGB}{134,113,127} 
\definecolor{gris}{RGB}{91,94,111} 
\definecolor{bleu1}{RGB}{38,109,131} 
\definecolor{bleu2}{RGB}{28,50,114} 
\definecolor{vert1}{RGB}{133,146,66} 
\definecolor{vert3}{RGB}{20,200,66} 
\definecolor{vert2}{RGB}{157,193,7} 
\definecolor{darkyellow}{RGB}{233,165,0}  
\definecolor{lightgray}{rgb}{0.9,0.9,0.9}
\definecolor{darkgray}{rgb}{0.6,0.6,0.6}
\definecolor{babyblue}{rgb}{0.54, 0.81, 0.94}
\definecolor{citrine}{rgb}{0.89, 0.82, 0.04}
\definecolor{misogreen}{rgb}{0.25,0.6,0.0}

\begin{document}

\twocolumn[

\aistatstitle{Fixed-Confidence Guarantees for Bayesian Best-Arm Identification}

\aistatsauthor{Xuedong Shang${}^{1,2}$ \ \ Rianne de Heide${}^{3,4}$ \ \  Emilie Kaufmann${}^{1,2,5}$ \ \ Pierre M\'enard${}^1$ \ \ Michal Valko${}^{1,6}$}

\aistatsaddress{${}^1$Inria Lille Nord Europe \ ${}^2$Universit\'e de Lille \ ${}^3$Leiden University \ ${}^4$CWI \ ${}^5$CNRS \  ${}^6$DeepMind Paris} ]


\begin{abstract}
We investigate and provide new insights on the sampling rule called Top-Two Thompson Sampling (\TTTS). In particular, we justify its use for \emph{fixed-confidence best-arm identification}. We further propose a variant of \TTTS called Top-Two Transportation Cost (\TCC), which disposes of the computational burden of \TTTS. As our main contribution, we provide the first sample complexity analysis of \TTTS and \TCC when coupled with a very natural Bayesian stopping rule, for bandits with Gaussian rewards, solving one of the open questions raised by ~\citet{russo2016ttts}. 
We also provide new posterior convergence results for \TTTS under two models that are commonly used in practice: bandits with Gaussian and Bernoulli rewards and conjugate priors.
\end{abstract}


\section{Introduction}\label{sec:intro}

In multi-armed bandits, a learner repeatedly chooses an \emph{arm} to play, and receives a reward from the associated unknown probability distribution. When the task is \emph{best-arm} identification (BAI), the learner is not only asked to sample an arm at each stage, but is also asked to output a recommendation (i.e., a guess for the arm with the largest mean reward) after a certain period. Unlike in another well-studied bandit setting, the learner is not interested in maximizing the sum of rewards gathered during the exploration (or minimizing \emph{regret}), but only cares about the quality of her recommendation. As such, BAI is a particular \emph{pure exploration} setting~\citep{bubeck2009pure}.

Formally, we consider a finite-arm bandit model, which is a collection of $K$ probability distributions, called arms $\cA\eqdef\{1,\ldots,K\}$, parametrized by their means $\mu_1, \ldots, \mu_K$. We assume the (unknown) best arm is unique and we denote it by $I^\star \eqdef \argmax_i \mu_i$. A best-arm identification strategy $(I_n, J_n, \tau)$ consists of three components. The first is a \emph{sampling rule}, which selects an arm $I_n$ at round $n$. At each round $n$, a vector of rewards $\bY_n = (Y_{n,1},\cdots,Y_{n,K})$ is generated for all arms independently from past observations, but only $Y_{n,I_n}$ is revealed to the learner. Let $\cF_n$ be the $\sigma$-algebra generated by  $(U_1,I_1,Y_{1,I_1},\cdots,U_n,I_{n},Y_{n,I_{n}})$, then $I_n$ is $\cF_{n-1}$-measurable, i.e., it can only depend on the past $n-1$ observations, and some exogeneous randomness, materialized into $U_{n-1} \sim \cU([0,1])$. The second component is a $\cF_{n}$-measurable \emph{recommendation rule} $J_n$, which returns a guess for the best arm, and thirdly, the  \emph{stopping rule}~$\tau$, a stopping time with respect to $\left(\cF_{n}\right)_{n \in \mathbb{N}}$, decides when the exploration is over.

BAI has already been studied within several theoretical frameworks. In this paper we consider the fixed-confidence setting, introduced by \cite{even-dar2003confidence}, in which given a risk parameter $\delta$, the goal is to ensure that probability to stop and recommend a wrong arm, $\PP{J_\tau \neq I^\star}$, is smaller than $\delta$, while minimizing the expected total number of samples to make this accurate recommendation, $\EE{\tau}$. The most studied alternative is the fixed-budget setting for which the stopping rule $\tau$ is fixed to some (known) maximal budget $n$, and the goal is to minimize the error probability $\PP{J_n \neq I^\star}$ \citep{audibert2010budget}. Note that these two frameworks are very different in general and do not share transferable regret bounds (see~\citealt{carpentier2016budget} for an additional discussion).

Most of the existing sampling rules for the fixed-confidence setting depend on the risk parameter $\delta$. Some of them rely on confidence intervals such as \LUCB~\citep{kalyanakrishnan2012lucb}, \UGapE~\citep{gabillon2012ugape}, 
or \LIL~\citep{jamieson2014lilucb}; others are based on eliminations such as \SE~\citep{even-dar2003confidence} and \EGE~\citep{karnin2013sha}. The first known sampling rule for BAI that does not depend on $\delta$ is the \emph{tracking} rule proposed by~\cite{garivier2016tracknstop}, which is proved to achieve the minimal sample complexity when combined with the Chernoff stopping rule when $\delta$ goes to zero. Such an \emph{anytime} sampling rule (neither depending on a risk $\delta$ or a budget $n$) is very appealing for applications, as advocated by \cite{jun2016atlucb}, who introduce the anytime best-arm identification framework. In this paper, we investigate another anytime sampling rule for BAI: \underline{T}op-\underline{T}wo \underline{T}hompson \underline{S}ampling (\TTTS), and propose a second anytime sampling rule: \underline{T}op-\underline{T}wo \underline{T}ransportation \underline{C}ost (\TCC).


Thompson Sampling~\citep{thompson1933} is a Bayesian algorithm well known for regret minimization, for which it is now seen as a major competitor to \UCB-typed approaches \citep{burnetas1996optimal,auer2002ucb,cappe2013klucb}. However, it is also well known that regret minimizing algorithms cannot yield optimal performance for BAI \citep{bubeck2011pure,kaufmann2017survey} and as we opt Thompson Sampling for BAI, then its adaptation is necessary. Such an adaptation, \TTTS, was given by \citet{russo2016ttts} along with the other top-two sampling rules \TTPS and \TTVS. By choosing between two different candidate arms in each round, these sampling rules enforce the exploration of sub-optimal arms, that would be under-sampled by vanilla Thompson sampling due to its objective of maximizing rewards.

While \TTTS appears to be a good anytime sampling rule for the fixed-confidence BAI when coupled with an appropriate stopping rule, so far there is no theoretical support for this employment. Indeed, the (Bayesian-flavored) asymptotic analysis of \cite{russo2016ttts} shows that under \TTTS, the posterior probability that $I^\star$ is the best arm converges almost surely to 1 at the best possible rate. However, this property does not by itself translate into sample complexity guarantees. Since the result of \cite{russo2016ttts}, \citet{qin2017ttei} proposed and analyzed \TTEI, another Bayesian sampling rule, both in the fixed-confidence setting and in terms of posterior convergence rate. Nonetheless, similar guarantees for \TTTS have been left as an open question by \cite{russo2016ttts}. In the present paper, we answer this open question. In addition, we propose \TCC, a computationally more favorable variant of \TTTS and  extend the fixed-confidence guarantees to \TCC as well.

\paragraph{Contributions} 
(1) We propose a new Bayesian sampling rule, \TCC, which is inspired by \TTTS but easier to implement and computationally advantageous (2) We investigate two Bayesian stopping and recommendation rules and establish their $\delta$-correctness for a bandit model with Gaussian rewards.\footnote{hereafter `Gaussian bandits' or `Gaussian model'} (3) We provide the first sample complexity analysis of \TTTS and \TCC for a Gaussian model and our proposed stopping rule. (4) \citeauthor{russo2016ttts}'s posterior convergence results for \TTTS were obtained under restrictive assumptions on the models and priors, which exclude the two mostly used in practice: Gaussian bandits with Gaussian priors and bandits with Bernoulli rewards\footnote{hereafter `Bernoulli bandits'} with Beta priors. We prove that optimal posterior convergence rates can be obtained for those two as well.

\paragraph{Outline} In Section~\ref{sec:algorithm}, we give a reminder of \TTTS and introduce \TCC along with our proposed recommendation and stopping rules. Then, in Section~\ref{sec:related}, we describe in detail two important notions of optimality that are invoked in this paper. The main fixed-confidence analysis follows in Section~\ref{sec:confidence}, and further Bayesian optimality results are given in Section~\ref{sec:bayesian}. Numerical illustrations are given in Section~\ref{sec:experiments}.

\section{Bayesian BAI Strategies}\label{sec:algorithm}

In this section, we give an overview of the sampling rule \TTTS and introduce \TCC. We provide details for Bayesian updating for Gaussian and Bernoulli models respectively, and introduce associated Bayesian stopping and recommendation rules. 

\subsection{Sampling rules}

Both \TTTS and \TCC employ a Bayesian machinery and make use of a prior distribution $\Pi_1$ over a set of parameters $\Theta$, that contains the unknown true parameter vector $\bmu$. Upon acquiring observations $(Y_{1,I_1},\cdots,Y_{n-1,I_{n-1}})$, we update our beliefs according to Bayes' rule and obtain a posterior distribution $\Pi_{n}$ which we assume to have density $\pi_n$ w.r.t.\,the Lebesgue measure. 
\citeauthor{russo2016ttts}'s analysis is restricted to strong regularity properties on the models and priors that exclude two important useful cases we consider in this paper: (1) the observations of each arm~$i$ follow a Gaussian distribution $\cN(\mu_i,\sigma^2)$ with common known variance $\sigma^2$, with imposed Gaussian prior $\cN(\mu_{1,i},\sigma_{1,i}^2)$, (2) all arms receive Bernoulli rewards with unknown means, with a uniform prior on each arm.

\paragraph{Gaussian model} For Gaussian bandits with a $\mathcal{N}(0,\kappa^2)$ prior on each mean, the posterior distribution of $\mu_i$ at round $n$ is Gaussian with mean and variance that are respectively given by
\[\frac{\sum_{\ell=1}^{n-1} \1\{I_\ell = i\} Y_{\ell,I_\ell}}{T_{n,i} + \sigma^2/\kappa^2}\quad\text{ and }\quad \frac{\sigma^2}{T_{n,i} + \sigma^2/\kappa^2},\]
where $T_{n,i}\eqdef\sum_{\ell=1}^{n-1} \1{\{ I_{\ell} = i \}}$ is the number of selections of arm $i$ before round $n$.
%
For the sake of simplicity, we consider improper Gaussian priors with $\mu_{1,i}=0$ and $\sigma_{1,i}=+\infty$ for all $i\in\cA$, for which
\[
    \mu_{n,i}  = \frac{1}{T_{n,i}} \sum_{\ell=1}^{n-1} \1\{I_\ell = i\}Y_{\ell,I_\ell} \quad \text{and} \quad \sigma_{n,i}^2 = \frac{\sigma^2}{T_{n,i}}.
\]
Observe that in that case the posterior mean $\mu_{n,i}$ coincides with the empirical mean.

\paragraph{Beta-Bernoulli model} For Bernoulli bandits with a uniform ($\cB eta(1,1)$) prior on each mean, the posterior distribution of $\mu_i$ at round $n$ is a Beta distribution with shape parameters $\alpha_{n,i} = \sum_{\ell=1}^{n-1} \1{\{ I_{\ell} = i \}} Y_{\ell,I_{\ell}} +1$ and $\beta_{n,i} = T_{n,i} - \sum_{\ell=1}^{n-1} \1{\{ I_{\ell} = i \}} Y_{\ell,I_{\ell}} + 1$. 

Now we briefly recall \TTTS and introduce \TCC.

\paragraph{Description of \TTTS} At each time step $n$, \TTTS has two potential actions: (1) with probability $\beta$, a parameter vector $\btheta$ is sampled from $\Pi_{n}$, and \TTTS chooses to play $I_n^{(1)} \eqdef \argmax_{i\in\cA} \theta_i$, (2) and with probability $1-\beta$, the algorithm continues sampling new $\btheta'$ until we obtain a \emph{challenger} $I_n^{(2)} \eqdef \argmax_{i\in\cA} \theta_i'$ that is different from $I_n^{(1)}$, and \TTTS then selects the challenger.

\paragraph{Description of \TCC} One drawback of \TTTS is that, in practice, when the posteriors become concentrated, it takes many Thompson samples before the challenger $I_n^{(2)}$ is obtained. We thus propose a variant of \TTTS, called \TCC, which alleviates this computational burden. Instead of re-sampling from the posterior until a different candidate appears, we define the challenger as the arm that has the lowest \emph{transportation cost} $W_n(I_n^{(1)},i)$ with respect to the first candidate (with ties broken uniformly at random). 

Let $\mu_{n,i}$ be the empirical mean of arm $i$ and $\mu_{n,i,j} \eqdef (T_{n,i}\mu_{n,i} +T_{n,j}\mu_{n,j})/(T_{n,i}+T_{n,j})$, then we define
\begin{equation}\label{def:Transportation}
	W_n(i,j) \eqdef
	\left\{ \begin{array}{ll}
				0 & \operatorname{if} \mu_{n,j} \geq \mu_{n,i},\\
				W_{n,i,j}+W_{n,j,i} & \operatorname{otherwise},
			\end{array}\right.
\end{equation}
where $W_{n,i,j}\eqdef T_{n,i} d\left(\mu_{n,i},\mu_{n,i,j}\right)$ for any $i,j$ and $d(\mu ; \mu' )$ denotes the Kullback-Leibler between the distribution with mean $\mu$ and that of mean $\mu'$. In the Gaussian case, $d(\mu;\mu') = (\mu-\mu')^2/(2\sigma^2)$ while in the Bernoulli case $d(\mu;\mu') = \mu \ln (\mu/\mu') + (1-\mu)\ln (1-\mu)/(1-\mu')$.
In particular, for Gaussian bandits 
\[
    W_n(i,j) = \dfrac{(\mu_{n,i}-\mu_{n,j})^2}{2\sigma^2(1/T_{n,i}+1/T_{n,j})}\1\{\mu_{n,j}<\mu_{n,i}\}.
\]

The pseudo-code of \TTTS and \TCC are shown in Algorithm~\ref{alg:sampling_rule}. Note that under the Gaussian model with improper priors, one should pull each arm once at the beginning for the sake of obtaining proper posteriors. 


\begin{algorithm}[ht]
\centering
\caption{Sampling rule (\textcolor{blue}{\TTTS}/\textcolor{red}{\TCC)}}
\label{alg:sampling_rule}
\footnotesize
\begin{algorithmic}[1]
   \State {\bfseries Input:} $\beta$ 
   \For{$n \leftarrow 1,2,\cdots$}
        \State \texttt{sample} $\btheta \sim \Pi_n$
        \State $I^{(1)} \leftarrow \argmax_{i\in\cA}\theta_i$
	    \State \texttt{sample} $b \sim \cB ern(\beta)$
	    \If{$b = 1$}
	        \State \texttt{evaluate arm} $I^{(1)}$
	    \Else
	        \State \textcolor{blue}{\texttt{repeat sample} $\btheta' \sim \Pi_n$}\tikzmark{top}
            \State \textcolor{blue}{$I^{(2)} \leftarrow \argmax_{i\in\cA}\theta_i'$}~~~~~~~~~~~~~~~~~~~~~~~~~~~\CommentX{\textcolor{blue}{\TTTS}}
	        \State \textcolor{blue}{\texttt{until} $I^{(2)} \neq I^{(1)}$}\tikzmark{bottom}
	        \State \textcolor{red}{$I^{(2)} \leftarrow \argmin_{i\neq I^{(1)}}W_n(I^{(1)},i), $ cf.\,\eqref{def:Transportation}} \tikzmark{right}~~~~~\CommentX{\textcolor{red}{\TCC}} 
		    \State \texttt{evaluate arm} $I^{(2)}$
	    \EndIf
	    \State \texttt{update mean and variance}
	    \State $t = t+1$
   \EndFor
\end{algorithmic}
\AddNote{top}{bottom}{right}{}
\end{algorithm}

$W_n$ in Line 12 of Algorithm~\ref{alg:sampling_rule} is the transportation cost defined in (\ref{def:Transportation}).


\subsection{Rationale for \TCC}

In order to explain how \TCC can be seen as an approximation of the re-sampling performed by \TTTS, we first need to define the \emph{optimal action probabilities}. 

\paragraph{Optimal action probability} The optimal action probability $a_{n,i}$ is defined as the posterior probability that arm $i$ is optimal. Formally, letting $\Theta_i$ be the subset of $\Theta$ such that arm $i$ is the optimal arm,
\[
    \Theta_i \eqdef \left\{ \btheta\in\Theta \biggm| \theta_i > \max_{j\neq i}\theta_j \right\},
\]
then we define
\[
   \quad a_{n,i} \eqdef \Pi_{n}(\Theta_i) = \int_{\Theta_i} \pi_n(\btheta) \text{d} \btheta.
\]
With this notation, one can show that under \TTTS, 
\begin{equation}\label{CondDist}
    \Pi_n\left(I_n^{(2)} =j | I_n^{(1)} = i\right) = \frac{a_{n,j}}{\sum_{k\neq i} a_{n,k}}.
\end{equation}
Furthermore, when $i$ coincides with the empirical best mean (and this will often be the case for $I_n^{(1)}$ when $n$ is large due to posterior convergence) one can write 
\[a_{n,j} \simeq \Pi_n\left(\theta_j \geq \theta_{i}\right) \simeq \exp\left(-W_n(i,j)\right),\]
where the last step is justified in Lemma~\ref{lemma:gaussiantails} in the Gaussian case (and Lemma~\ref{lemma:binomial_tail} in Appendix~\ref{app:posterior_beta.aux} in the Bernoulli case). Hence, \TCC replaces sampling from the distribution \eqref{CondDist} by an approximation of its mode which is \emph{easy to compute}. Note that directly computing the mode would require to compute $a_{n,j}$, which is much more costly than the computation of $W_{n}(i,j)$\footnote{the \TTPS sampling rule \citep{russo2016ttts} also requires the computation of $a_{n,i}$, thus we do not report simulations for this Bayesian sampling rule in Section~\ref{sec:experiments}}. 

\subsection{Stopping and recommendation rules}

In order to use \TTTS or \TCC as sampling rule for fixed-confidence BAI, we need to additionally define stopping and recommendation rules. While \cite{qin2017ttei} suggest to couple \TTEI with the ``frequentist'' Chernoff stopping rule \citep{garivier2016tracknstop}, we propose in this section natural Bayesian stopping and recommendation rule. They both rely on the optimal action probabilities defined above.

\paragraph{Bayesian recommendation rule} At time step $n$, a natural candidate for the best arm is the arm with largest optimal action probability, hence we define 
\[
    J_n \eqdef \argmax_{i\in\cA} a_{n,i}.
\]

\paragraph{Bayesian stopping rule}
In view of the recommendation rule, it is natural to stop when the posterior probability that the recommended action is optimal is large, and exceeds some threshold $c_{n,\delta}$ which gets close to 1. Hence our Bayesian stopping rule is \begin{equation}\label{eq:stopping}
    \tau_{\delta} \eqdef \inf \left\{ n\in\NN:\max_{i\in\cA} a_{n,i} \geq c_{n,\delta} \right\}.
\end{equation}

\paragraph{Links with frequentist counterparts} Using the transportation cost $W_n(i,j)$ defined in \eqref{def:Transportation}, the Chernoff stopping rule of~\cite{garivier2016tracknstop} can actually be rewritten as
\begin{equation}\label{eq:chernoffstoppingtime}
\hspace{-0.2cm}\tau_\delta^{\text{Ch.}} \eqdef \inf \left\lbrace n \in \mathbb{N} : \max_{i \in \cA} \min_{j \in \cA \setminus \{i\} } W_{n}(i,j) > d_{n,\delta} \right\rbrace.
\end{equation}
This stopping rule coupled with the recommendation rule $J_n = \argmax_{i} \mu_{n,i}$. 

As explained in that paper, $W_{n}(i,j)$ can be interpreted as a (log) Generalized Likelihood Ratio statistic for rejecting the hypothesis $\cH_0 : (\mu_i < \mu_j)$. Through our Bayesian lens, we rather have in mind the approximation $\Pi_n(\theta_j > \theta_i) \simeq \expp{-W_n(i,j)}$, valid when $\mu_{n,i}> \mu_{n,j}$, which permits to analyze the two stopping rules using similar tools, as will be seen in the proof of Theorem~\ref{thm:pac_gaussian}. 

As shown later in Section~\ref{sec:confidence}, $\tau_\delta$ and $\tau_\delta^{\text{Ch.}}$ prove to be fairly similar for some corresponding choices of the thresholds $c_{n,\delta}$ and $d_{n,\delta}$. This endorses the use of the Chernoff stopping rule in practice, which does not require the (heavy) computation of optimal action probabilities. Still, our sample complexity analysis applies to the two stopping rules, and we believe that a frequentist sample complexity analysis of a fully Bayesian BAI strategy is a nice theoretical contribution.

\paragraph{Useful notation}



We follow the notation of \citet{russo2016ttts} and define the following measures of effort allocated to arm $i$ up to time $n$,
\begin{align*}
    \psi_{n,i} \eqdef \PP{I_n = i | \cF_{n-1}}\quad \text{and} \quad \Psi_{n,i} \eqdef \sum_{l=1}^n \psi_{l,i}.
\end{align*}

In particular, for \TTTS we have
\[
    \psi_{n,i} =  \beta a_{n,i} + (1-\beta) a_{n,i}\sum_{j\neq i} \frac{a_{n,j}}{1-a_{n,j}},
\]
while for \TCC
{\small
\[
    \psi_{n,i} = \beta a_{n,i} + (1-\beta) \sum_{j\neq i} a_{n,j}\frac{\1\{W_n(j,i)=\min_{k\neq j} W_n(j,k)\}}{\#\left|\argmin_{k\neq j } W_n(j,k)\right|}.
\]}

\section{Two Related Optimality Notions}\label{sec:related}

In the fixed-confidence setting, we aim for building $\delta$-correct strategies, i.e.\ strategies that identify the best arm with high confidence on any problem instance. 

\begin{definition} A strategy $(I_n,J_n,\tau)$ is $\delta$-correct if for all bandit models $\bmu$ with a unique optimal arm, it holds that $\mathbb{P}_{\bmu}\left[J_{\tau} \neq I^\star\right] \leq \delta$. 
\end{definition}

Among $\delta$-correct strategies, seek the one with the smallest sample complexity $\EE{\tau_\delta}$. So far, \TTTS has not been analyzed in terms of sample complexity; \citet{russo2016ttts} focusses on posterior consistency and optimal convergence rates. Interestingly, both the smallest possible sample complexity and the fastest rate of posterior convergence can be expressed in terms of the following quantities.




\begin{definition} Let $\Sigma_K = \{\bomega : \sum_{k=1}^K \omega_k = 1\}$ and define for all $i\neq I^\star$
\[
    C_i(\omega,\omega') \eqdef \min_{x\in \cI} \ \omega d(\mu_{I^\star};x) + \omega' d(\mu_i;x),
\]
where $d(\mu,\mu')$ is the KL-divergence defined above and $\cI = \R$ in the Gaussian case and $\cI = [0,1]$ in the Bernoulli case. We define
\begin{eqnarray}
    \Gamma^\star &\eqdef& \max_{\bomega \in \Sigma_K}\min_{i\neq I^\star} C_i(\omega_{I^\star},\omega_i),\nonumber\\
    \Gamma_{\beta}^\star &\eqdef& \max_{\substack{\bomega \in \Sigma_K\\\omega_{I^\star}=\beta}}\min_{i\neq I^\star} C_i(\omega_{I^\star},\omega_i).\label{def:GammaBeta}\end{eqnarray}
\end{definition}

The quantity $C_i(\omega_{I^\star},\omega_i)$ can be interpreted as a ``transportation cost''\footnote{for which $W_n(I^\star,i)$ is an empirical counterpart} from the original bandit instance $\bm\mu$ to an alternative instance in which the mean of arm $i$ is larger than that of $I^\star$, when the proportion of samples allocated to each arm is given by the vector $\bomega \in \Sigma_K$. As shown by~\cite{russo2016ttts}, the $\bomega$ that maximizes \eqref{def:GammaBeta} is unique, which allows us to define the $\beta$-\emph{optimal allocation} $\bomega^\beta$ in the following proposition.
\begin{restatable}{proposition}{restateoptim}\label{prop:optim}
There is a unique solution $\bomega^\beta$ to the optimization problem \eqref{def:GammaBeta}
satisfying $\omega_{I^\star}^\beta = \beta$, and for all $i,j \neq I^\star$, $C_i(\beta,\omega_i^\beta) = C_j(\beta,\omega_j^\beta)$.
\end{restatable}

For models with more than two arms, there is no closed form expression for $\Gamma_\beta^\star$ or $\Gamma^\star$, even for Gaussian bandits  with variance $\sigma^2$ for which we have
\[
    \Gamma_{\beta}^\star = \max_{\bomega:\omega_{I^\star}=\beta}\min_{i\neq I^\star} \frac{(\mu_{I^\star}-\mu_i)^2}{2\sigma^2(1/\omega_i+1/\beta)}.
\]


\paragraph{Bayesian $\beta$-optimality} \citet{russo2016ttts} proves that  any sampling rule allocating a fraction $\beta$ to the optimal arm ($\Psi_{n,I^\star}/n \rightarrow \beta$) satisfies 
$1-a_{n, I^\star} \geq e^{-n(\Gamma_{\beta}^\star + o(1))}$ (a.s.) for large values of $n$. We define a  \emph{Bayesian $\beta$-optimal} sampling rule as a sampling rule matching this lower bound, i.e. satisfying $\Psi_{n,I^\star}/n \rightarrow \beta$ and $1- a_{n, I^\star} \leq e^{-n(\Gamma_{\beta}^\star + o(1))}$.

\citet{russo2016ttts} proves that \TTTS with parameter $\beta$ is Bayesian $\beta$-optimal.
However, the result is valid only under strong regularity assumptions, excluding the two practically important cases of Gaussian and Bernoulli bandits. In this paper, we complete the picture by establishing Bayesian $\beta$-optimality for those models in Section~\ref{sec:bayesian}. For the Gaussian bandit, Bayesian $\beta$-optimality was established for \TTEI by~\cite{qin2017ttei} with Gaussian priors, but this remained an open problem for \TTTS.

A fundamental ingredient of these proofs is to establish the convergence of the allocation of measurement effort to the $\beta$-optimal allocation: $\Psi_{n,i}/n \rightarrow \omega_{i}^\beta$ for all $i$, which is equivalent to $T_{n,i}/n \rightarrow \omega_{i}^\beta$ (cf.\ Lemma~\ref{lemma:link}).

\paragraph{$\beta$-optimality in the fixed-confidence setting} In the fixed confidence setting, the performance of an algorithm is evaluated in terms of sample complexity. A lower bound given by \cite{garivier2016tracknstop} states that any $\delta$-correct strategy satisfies $\EE{\tau_\delta} \geq (\Gamma^\star)^{-1} \ln \left({1}/(3\delta)\right)$. 

Observe that $\Gamma^\star = \max_{\beta \in [0,1]} \Gamma_\beta^\star$. Using the same lower bound techniques, one can also prove that under any $\delta$-correct strategy satisfying $T_{n,I^\star}/n \rightarrow \beta$,
\[\liminf_{\delta \rightarrow 0}\frac{\EE{\tau_\delta}}{\ln(1/\delta)} \geq \frac{1}{\Gamma^\star_\beta}.\]
This motivates the relaxed optimality notion that we introduce in this paper: A BAI strategy is called \emph{asymptotically $\beta$-optimal} if it satisfies 
\[\frac{T_{n,I^\star}}{n}\rightarrow \beta \ \ \ \text{and} \ \ \ \limsup_{\delta \rightarrow 0}\frac{\EE{\tau_\delta}}{\ln(1/\delta)} \leq \frac{1}{\Gamma^\star_\beta}.\]
In the paper, we provide the first sample complexity analysis of a BAI algorithm based on \TTTS (with the stopping and recommendation rules described in Section~\ref{sec:algorithm}), establishing its asymptotic $\beta$-optimality.

As already observed by \cite{qin2017ttei}, any sampling rule converging to the $\beta$-optimal allocation (i.e.\ satisfying $T_{n,i}/n \rightarrow w_i^\beta$ for all $i$) can be shown to satisfy $\limsup_{\delta \rightarrow 0} \tau_\delta/\ln(1/\delta) \leq (\Gamma_\beta^\star)^{-1}$ almost surely,  when coupled with the Chernoff stopping rule. The fixed confidence optimality that we define above is stronger as it provides guarantees on $\EE{\tau_\delta}$.

\section{Fixed-Confidence Analysis}\label{sec:confidence}

In this section, we consider Gaussian bandits and the Bayesian rules using an improper prior on the means.
We state our main result below, showing that \TTTS and \TCC are asymptotically $\beta$-optimal in the fixed confidence setting, when coupled with appropriate stopping and recommendation rules. 

\begin{theorem}\label{thm:confidence_main} With $\cC^{g_G}$ the function defined by \cite{kaufmann2018mixture}, which satisfies $\cC^{g_G}(x) \simeq x+\ln(x)$, we introduce the threshold
\begin{equation}d_{n,\delta} = 4\ln(4+\ln(n)) + 2 \cC^{g_G}\left(\frac{\ln((K-1)/\delta)}{2}\right).\label{def:thresholdD}\end{equation}
The \TTTS and \TCC sampling rules coupled with either   
\begin{itemize}
 \item the Bayesian stopping rule \eqref{eq:stopping} with threshold \[c_{n,\delta} = 1 - \frac{1}{\sqrt{2\pi}} e^{-\left(\sqrt{d_{n,\delta}} + \frac{1}{\sqrt{2}}\right)^2}\]
 and the recommendation rule $J_t = \argmax_{i} a_{n,i}$
  \item or the Chernoff stopping rule \eqref{eq:chernoffstoppingtime} with threshold $d_{n,\delta}$
 and recommendation rule $J_t = \argmax_i \mu_{n,i}$,
\end{itemize}
form a $\delta$-correct BAI strategy. Moreover, if all the arms means are distinct, it satisfies  
    \[
        \limsup_{\delta\rightarrow{0}} \frac{\EE{\tau_{\delta}}}{\log(1/\delta)} \leq \frac{1}{\Gamma_{\beta}^\star}.
    \]
\end{theorem}

We now give the proof of Theorem~\ref{thm:confidence_main}, which is divided into three parts. The \textbf{first step} of the analysis is to prove the $\delta$-correctness of the studied BAI strategies.

\begin{restatable}{theorem}{restatepac}\label{thm:pac_gaussian}
    Regardless of the sampling rule, the stopping rule~(\ref{eq:stopping}) with the threshold $c_{n,\delta}$ and the Chernoff stopping rule with threshold $d_{n,\delta}$ defined in Theorem~\ref{thm:confidence_main} satisfy $        \PP{\tau_{\delta} < \infty \wedge J_{\tau_{\delta}} \neq I^\star} \leq \delta$.
\end{restatable}

To prove that \TTTS and \TCC allow to reach a $\beta$-optimal sample complexity, one needs to quantify how fast the measurement effort for each arm is concentrating to its corresponding optimal weight. For this purpose,  we introduce the random variable
\[
    T_{\beta}^\epsilon \eqdef \inf \left\{ N\in\NN: \max_{i\in\cA} \vert T_{n,i}/n-\omega_i^\beta \vert \leq \epsilon, \forall n \geq N \right\}.
\]
The \textbf{second step} of our analysis is a sufficient condition for $\beta$-optimality, stated in Lemma~\ref{lemma:confidence}. Its proof is given in Appendix~\ref{app:confidence}. The same result was proven for the Chernoff stopping rule by \cite{qin2017ttei}.

\begin{restatable}{lemma}{restatefixedconfidence}\label{lemma:confidence}
    Let $\delta,\beta\in (0,1)$. For any sampling rule which satisfies $\EE{T_{\beta}^\epsilon} < \infty$ for all $\epsilon > 0$, we have
    \[
        \limsup_{\delta\rightarrow{0}} \frac{\EE{\tau_{\delta}}}{\log(1/\delta)} \leq \frac{1}{\Gamma_{\beta}^\star},
    \]
    if the sampling rule is coupled with stopping rule~(\ref{eq:stopping}), 
\end{restatable}

Finally, it remains to show that \TTTS and \TCC meet the sufficient condition, and therefore the \textbf{last step}, which is the core component and the most technical part our analysis, consists of showing the following.

\begin{theorem}\label{thm:sufficient_condition}
    Under \TTTS or \TCC, $\EE{T_{\beta}^\epsilon} < +\infty$.
\end{theorem}


In the rest of this section, we prove Theorem~\ref{thm:pac_gaussian} and sketch the proof of Theorem~\ref{thm:sufficient_condition}. But we first highlight some important ingredients for these proofs.

\subsection{Core ingredients}

Our analysis hinges on properties of the Gaussian posteriors, in particular on the following tails bounds, which follow from Lemma 1 of \cite{qin2017ttei}.

\begin{lemma}\label{lemma:gaussiantails}
For any $i,j\in\cA$, if $\mu_{n,i}\leq \mu_{n,j}$
{\small
\begin{flalign}
    \Pi_n\left[\theta_i\geq \theta_j\right] &\leq \frac{1}{2} \expp{-\frac{\left( \mu_{n,j}-\mu_{n,i} \right)^2}{2\sigma_{n,i,j}^2}},\label{gaussian_upper}\\
    \Pi_n\left[\theta_i\geq\theta_j\right] &\geq \frac{1}{\sqrt{2\pi}} \exp \left\{-\frac{\left(\mu_{n,j}-\mu_{n,i} +  \sigma_{n,i,j}\right)^2}{2\sigma_{n,i,j}^2}\right\}, \label{gaussian_lower}
\end{flalign}
}%
where $\sigma_{n,i,j}^2 \eqdef \sigma^2/T_{n,i} + \sigma^2/T_{n,j}$.
\end{lemma}

This lemma is crucial to control $a_{n,i}$ and $\psi_{n,i}$, the optimal action and selection probabilities. 

\subsection{Proof of Theorem~\ref{thm:pac_gaussian}} \label{subsec:proofPAC}

We upper bound the desired probability as follows
{\small
\begin{flalign*}
&\PP{\tau_{\delta} < \infty \wedge J_{\tau_{\delta}} \neq I^\star}  \leq  \sum_{i\neq I^\star} \PP{\exists n \in \NN : \alpha_{i,n} > c_{n,\delta}} \\
&\leq  \sum_{i\neq I^\star}\!\PP{\exists n \in \NN : \Pi_n(\theta_i \geq \theta_{I_\star}) > c_{n,\delta}, \mu_{n,I^\star} \!\leq \mu_{n,i}}\\
&\leq  \sum_{i\neq I^\star}\!\PP{\exists n \in \NN : 1-c_{n,\delta} > \Pi_n(\theta_{I^\star}\!\! > \theta_i), \mu_{n,I^\star} \!\leq \mu_{n,i}}.
\end{flalign*}
}%
The second step uses the fact that as $c_{n,\delta}\geq 1/2$, a necessary condition for $\Pi_n(\theta_i \geq \theta_{I_\star}) \geq c_{n,\delta}$ is that $\mu_{n,i} \geq \mu_{n,I_\star}$. Now using the lower bound \eqref{gaussian_lower}, if $ \mu_{n,I^\star}\leq \mu_{n,i}$, the inequality  $1-c_{n,\delta} > \Pi_n(\theta_{I^\star} > \theta_i)$  implies
\[
    \displaystyle \frac{(\mu_{n,i}-\mu_{n,I^\star}\!)^2}{2\sigma_{n,i,I^\star}^2} \! \geq \! \left(\sqrt{\ln{\frac{1}{\sqrt{2\pi}(1-c_{n,\delta})}}}\! - \frac{1}{\sqrt{2}}\! \right)^2 = d_{n,\delta},
\]
where the equality follows from the expression of $c_{n,\delta}$ as function of $d_{n,\delta}$. Hence to conclude the proof it remains to check that 
{\scriptsize
\begin{flalign}\label{ToProveDef}
   \PP{\exists n \! \in \!\NN: \!\mu_{n,i} \geq \mu_{n,I^\star}\!,\!\frac{(\mu_{n,i}\!-\!\mu_{n,I^\star}\!)^2}{2\sigma_{n,i,I^\star}^2} \! \geq \! d_{n,\delta}}\! \leq \!\frac{\delta}{K\!-\!1}.\!\!
\end{flalign}
}%
To prove this, we observe that for $\mu_{n,i} \geq \mu_{n,I^\star}$, 
{\small
\begin{align*}
\frac{(\mu_{n,i}-\mu_{n,I^\star}\!)^2}{2\sigma_{n,i,I^\star}^2} &= \inf_{\theta_i<\theta_{I^\star}} T_{n,i}d(\mu_{n,i};\theta_i) + T_{n,I^\star}d(\mu_{n,I^\star}\!;\theta_{I^\star}\!)\\
&\leq T_{n,i}d(\mu_{n,i};\mu_i) + T_{n,I^\star}d(\mu_{n,I^\star}\!;\mu_{I^\star}\!).
\end{align*}
}%

Corollary 10 of~\cite{kaufmann2018mixture} then allows us to upper bound the probability
{\small
\[
    \PP{\exists n \in \NN: T_{n,i}d(\mu_{n,i};\mu_i) + T_{n,I^\star}d(\mu_{n,I^\star},\mu_{I^\star}) \geq d_{n,\delta}} 
\]
}%
by $\delta/(K-1)$ for the choice of threshold given in \eqref{def:thresholdD},
which completes the proof that the stopping rule \eqref{eq:stopping} is $\delta$-correct. The fact that the Chernoff stopping rule with the above threshold $d_{n,\delta}$ given above is $\delta$-correct straightforwardly follows from \eqref{ToProveDef}.  

%

\subsection{Sketch of the proof of Theorem~\ref{thm:sufficient_condition}}

We present a unified proof sketch of Theorem~\ref{thm:sufficient_condition} for \TTTS and \TCC. While the two analyses follow the same steps, some of the lemmas given below have different proofs for \TTTS and \TCC, which can be found in Appendix~\ref{app:confidence_ttts} and Appendix~\ref{app:confidence_t3c} respectively.

We first state two important concentration results, that hold under any sampling rule. 

\begin{restatable}{lemma}{restatewone}[Lemma 5 of~\citealt{qin2017ttei}]\label{lemma:means}
    There exists a random variable $W_1$, such that for all $i\in\cA$,
    \[
        \forall n\in\NN, \quad |\mu_{n,i} - \mu_{i}| \leq \sigma W_1 \sqrt{\frac{\log(e+T_{n,i})}{1+T_{n,i}}} \text{ a.s.},
    \]
    and $\EE{e^{\lambda W_1}} < \infty$ for all $\lambda > 0$.
\end{restatable}

\begin{restatable}{lemma}{restatewtwo}\label{lemma:link}
There exists a random variable $W_2$, such that for all $i\in\cA$,
    \[
        \forall n\in\NN, |T_{n,i}-\Psi_{n,i}| \leq W_2\sqrt{(n+1)\log(e^2+n)} \text{ a.s.},
    \]
and $\EE{e^{\lambda W_2}} < \infty$ for any $\lambda > 0$.
\end{restatable}

Lemma~\ref{lemma:means} controls the concentration of the posterior means towards the true means and Lemma~\ref{lemma:link} establishes that $T_{n,i}$ and $\Psi_{n,i}$ are close. Both results rely on uniform deviation inequalities for martingales.

Our analysis uses the same principle as that of \TTEI: We establish that $T_\beta^\epsilon$ is upper bounded by some random variable $N$ which is a polynomial of the random variables $W_1$ and $W_2$ introduced in the above lemmas, denoted by $\text{Poly}(W_1,W_2) \eqdef \cO(W_1^{c_1}W_2^{c_2})$, where $c_1$ and $c_2$ are two constants (that may depend on arms' means and the constant hidden in the $\cO$). As all exponential moments of $W_1$ and $W_2$ are finite, $N$ has a finite expectation as well, which concludes the proof.

The first step to exhibit such an upper bound $N$ is to establish that every arm is pulled sufficiently often. 

\begin{restatable}{lemma}{restatesuffexploration}\label{lemma:sufficient_exploration}
Under \TTTS or \TCC, there exists $N_1 = \text{Poly}(W_1,W_2)$ s.t. $\forall n \geq N_1$, for all $i$, $\ T_{n,i} \geq \sqrt{{n}/{K}}$, almost surely.
\end{restatable}

Due to the randomized nature of \TTTS and \TCC, the proof of Lemma~\ref{lemma:sufficient_exploration} is significantly more involved than for a deterministic rule like \TTEI. Intuitively, the posterior of each arm would be well concentrated once the arm is sufficiently pulled. If the optimal arm is under-sampled, then it would be chosen as the first candidate with large probability. If a sub-optimal arm is under-sampled, then its posterior distribution would possess a relatively wide tail that overlaps with or cover the somehow narrow tails of other overly-sampled arms. The probability of that sub-optimal arm being chosen as the challenger would be large enough then.


Combining Lemma~\ref{lemma:sufficient_exploration} with Lemma~\ref{lemma:means} straightforwardly leads to the following result.

\begin{restatable}{lemma}{restatemeans}\label{lemma:tracking_means}
    Under \TTTS or \TCC, fix a constant $\epsilon > 0$, there exists $N_2 = \text{Poly}(1/\epsilon,W_1,W_2)$ s.t. $\forall n \geq N_2$,
    \[
        \forall i \in \cA, \quad |\mu_{n,i}-\mu_i| \leq \epsilon.
    \]
\end{restatable}

We can then deduce a very nice property about the optimal action probability for sub-optimal arms from the previous two lemmas. Indeed, we can show that
\[
    \forall i\neq I^\star, \quad a_{n,i} \leq \expp{-\frac{\Delta_{\text{min}}^2}{16\sigma^2}\sqrt{\frac{n}{K}}}
\]
for $n$ larger than some $\text{Poly}(W_1,W_2)$. In the previous inequality, $\Delta_{\text{min}}$ is the smallest mean difference among all the arms.

Plugging this in the expression of $\psi_{n,i}$, one can easily quantify how fast $\psi_{n,I^\star}$ converges to $\beta$, which eventually yields the following result. 

\begin{restatable}{lemma}{restatetrackingbest}\label{lemma:tracking_best}
    Under \TTTS or \TCC, fix $\epsilon > 0$, then there exists $N_3 = \text{Poly}(1/\epsilon,W_1,W_2)$ s.t. $\forall n \geq N_3$,
    \[
        \left|\frac{T_{n,I^\star}}{n} - \beta\right| \leq \epsilon. 
    \]
\end{restatable}

The last, more involved, step is to establish that the fraction of measurement allocation to every sub-optimal arm $i$ is indeed similarly close to its optimal proportion $\omega_i^\beta$.

\begin{restatable}{lemma}{restatetrackingother}\label{lemma:tracking_other}
    Under \TTTS or \TCC, fix a constant $\epsilon > 0$, there exists $N_4 = \text{Poly}(1/\epsilon,W_1,W_2)$ s.t. $\forall n \geq N_4$,
    \[
        \forall i\neq I^\star, \quad \left|\frac{T_{n,i}}{n} - \omega_i^\beta\right| \leq \epsilon. 
    \]
\end{restatable}

The major step in the proof of Lemma~\ref{lemma:tracking_other} for each sampling rule, is to establish that if some arm is over-sampled, then its probability to be selected is exponentially small. Formally, we show that for $n$ larger than some $\text{Poly}(1/\epsilon,W_1,W_2)$,
\[
    \frac{\Psi_{n,i}}{n} \geq \omega_{i}^\beta + \xi \ \ \ \Rightarrow \ \ \ \psi_{n,i} \leq \expp{- f(n,\xi)},
\]
for some function $f(n,\xi)$ to be specified for each sampling rule, satisfying $f(n)\geq C_\xi\sqrt{n}$ (a.s.). This result leads to the concentration of $\Psi_{n,i}/n$, thus can be easily converted to the concentration of $T_{n,i}/n$ by Lemma~\ref{lemma:link}.


Finally, Lemma~\ref{lemma:tracking_best} and Lemma~\ref{lemma:tracking_other} show that $T_\beta^\epsilon$ is upper bounded by $N \eqdef \max(N_3,N_4)$, which yields $\mathbb{E}[{T_\beta^\epsilon}] \leq \max(\EE{N_3},\EE{N_4}) < \infty$.

\begin{figure*}[t!]
\centering
\includegraphics[clip, width= 0.24\textwidth]{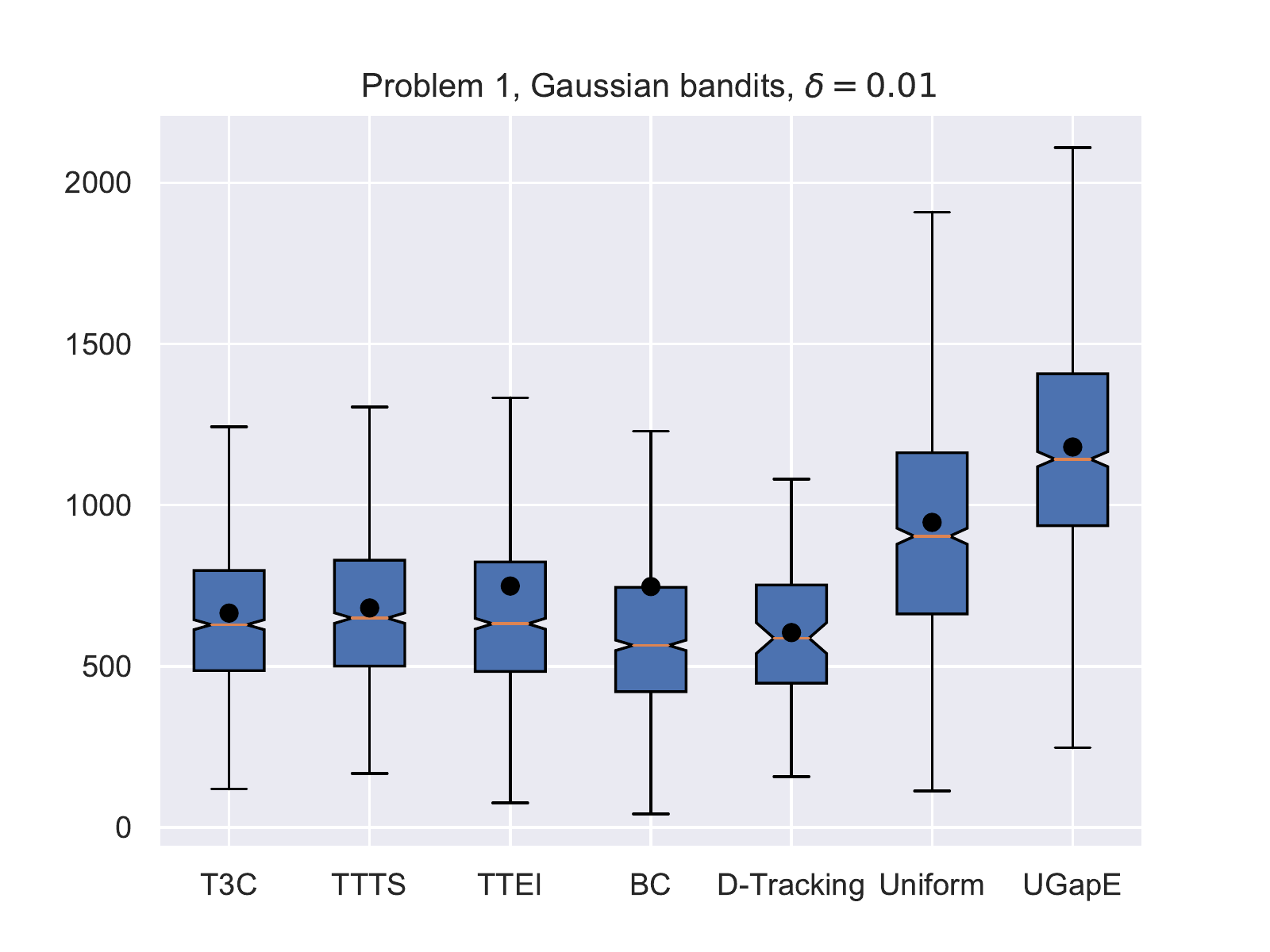}
\includegraphics[clip, width= 0.24\textwidth]{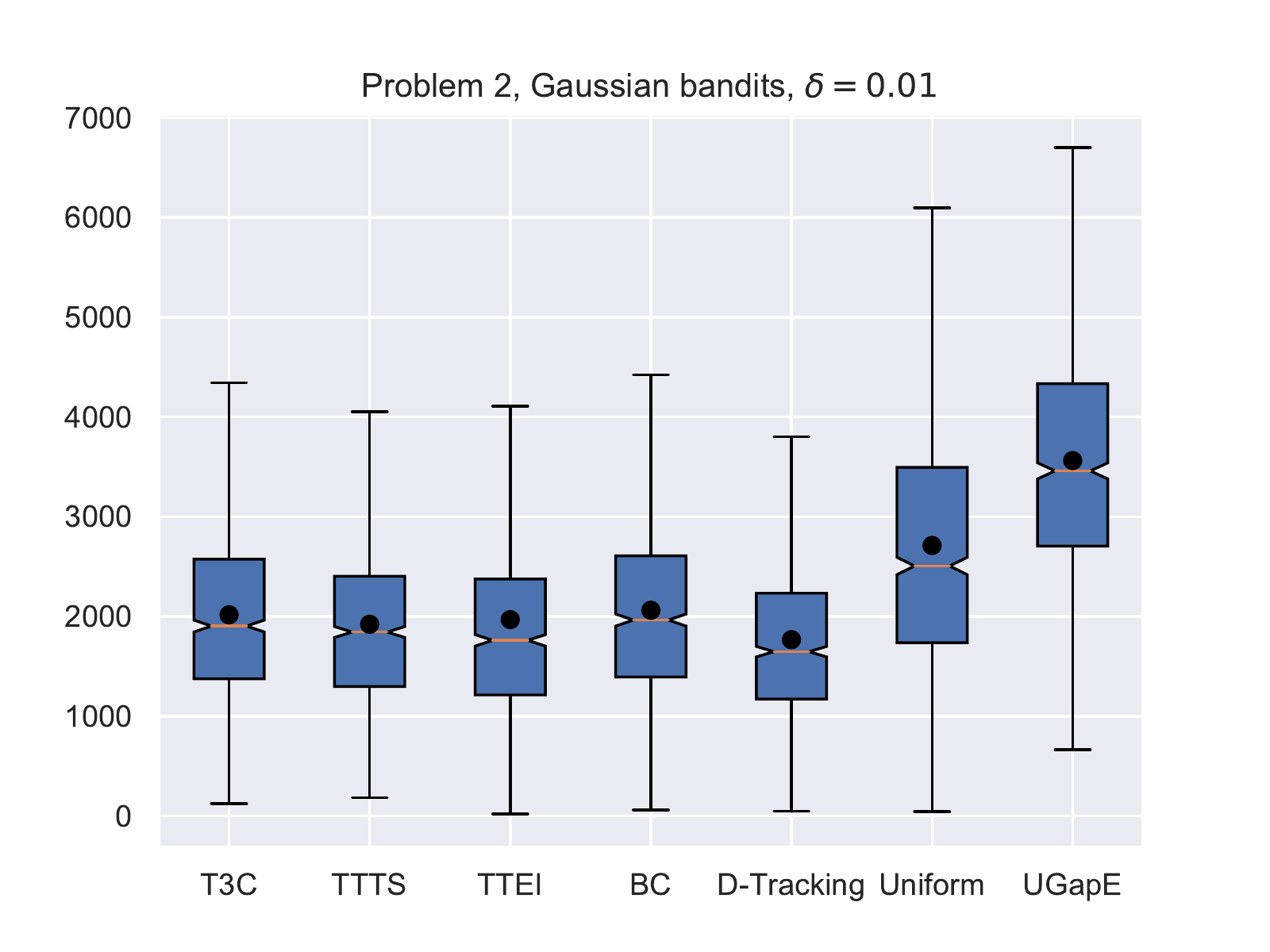}
\includegraphics[clip, width= 0.24\textwidth]{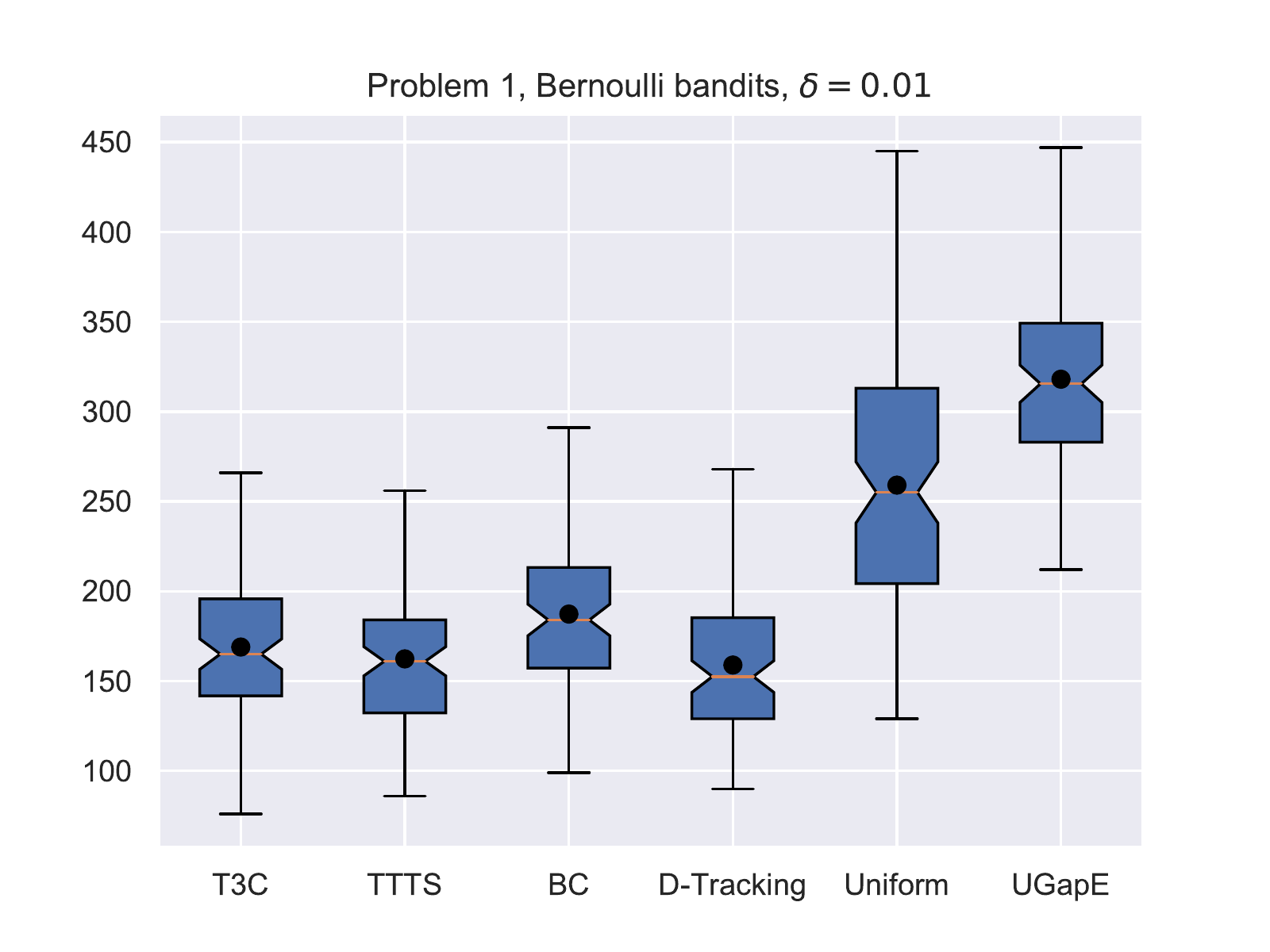}
\includegraphics[clip, width= 0.24\textwidth]{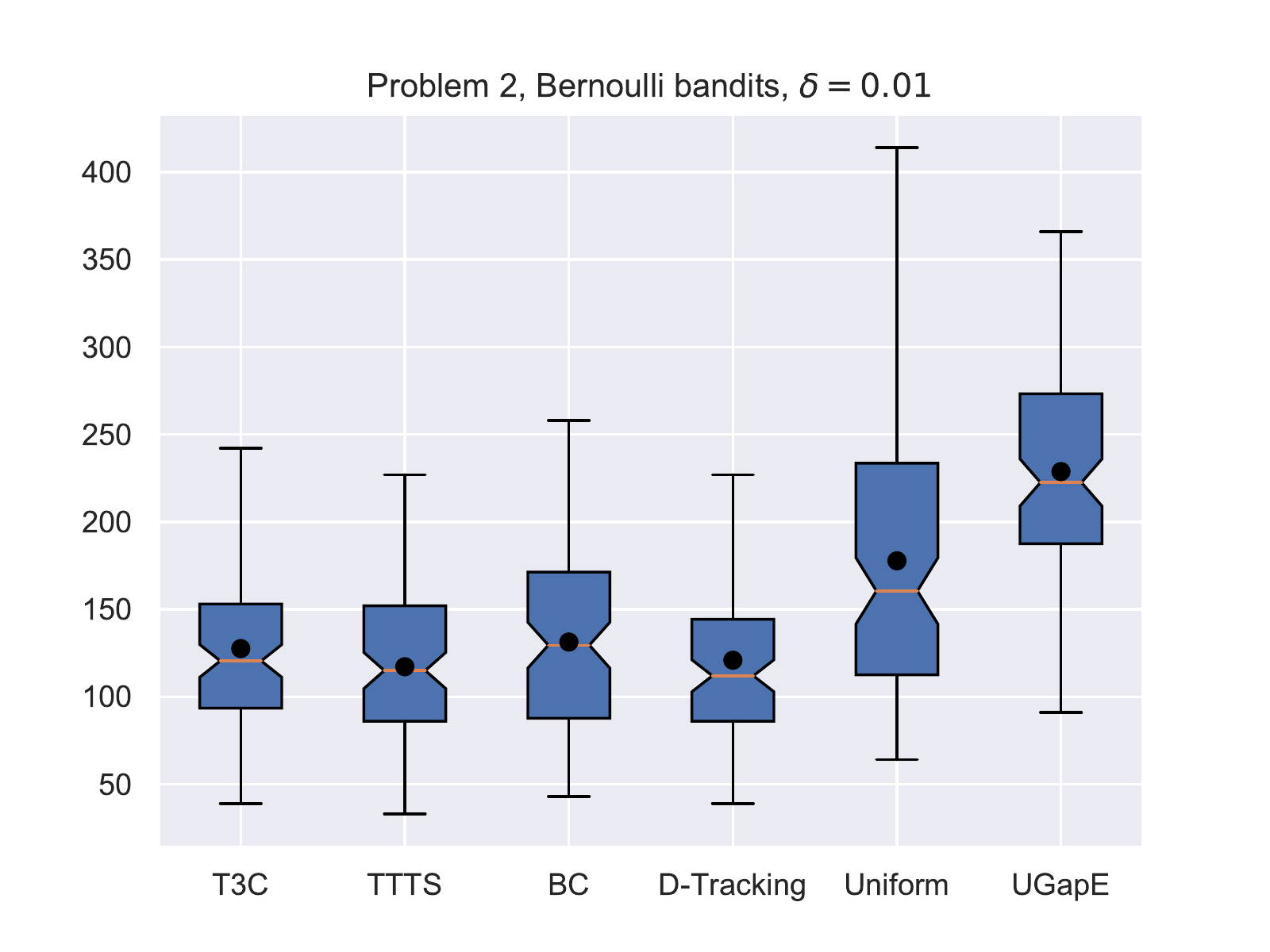}
\caption{black dots represent means and oranges lines represent medians.}
\label{fig:confidence}
\end{figure*}

\begin{table*}[t!]
\centering
\def\arraystretch{1.2}
\begin{tabular}{|c|c|c|c|c|c|c|c|}
 \hline
 \textbf{Sampling rule} & \TCC & \TTTS & \TTEI & \BC & \DT & \texttt{Uniform} & \UGapE \\
 \hline
 \textbf{Execution time (s)} & $1.6\times 10^{-5}$ & $2.3\times 10^{-4}$ & $1\times 10^{-5}$ & $1.4\times 10^{-5}$ & $1.3\times 10^{-3}$ & $6\times 10^{-6}$ & $5\times 10^{-6}$ \\
 \hline
\end{tabular}
\caption{average execution time in seconds for different sampling rules.}
\label{table:time}
\end{table*}

\section{Optimal Posterior Convergence
}\label{sec:bayesian}

Recall that $a_{n, I^\star}$ denotes the posterior mass assigned to the event that action $I^\star$ (i.e.\ the true optimal arm) is optimal at time $n$. As the number of observations tends to infinity, we desire that the posterior distribution converges to the truth. In this section we show equivalently that the posterior mass on the complementory event, $1 - a_{n, I^\star}$, the event that arm $I^\star$ is not optimal, converges to zero at an exponential rate, and that it does so at optimal rate $\Gamma_{\beta}^\star$. 

\citet{russo2016ttts} proves a similar theorem under three confining boundedness assumptions (cf.\,\citealt{russo2016ttts}, Asssumption 1) on the parameter space, the prior density and the (first derivative of the) log-normalizer of the exponential family. Hence, the theorems in \cite{russo2016ttts} do not apply to the two bandit models most used in practise, which we consider in this paper: the Gaussian and Bernoulli model. 

In the first case, the parameter space is unbounded, in the latter model, the derivative of the log-normalizer (which is $e^{\eta} / (1 + e^\eta)$) is unbounded. Here we provide two theorems, proving that under \TTTS, the optimal, exponential posterior convergence rates are obtained for the Gaussian model with uninformative (improper) Gaussian priors (proof given in Appendix~\ref{app:posterior_gaussian}), and the Bernoulli model with $\cB eta(1,1)$ priors (proof given in Appendix~\ref{app:posterior_beta}).

\begin{restatable}{theorem}{restateposteriorgaussian}\label{thm:posterior_gaussian}
    Under \TTTS, for Gaussian bandits with improper Gaussian priors, it holds almost surely that 
    \[
        \lim_{n\rightarrow{\infty}} -\frac{1}{n}\log(1-a_{n,I^\star}) = \Gamma_{\beta}^\star.
    \]
\end{restatable}

\begin{restatable}{theorem}{restateposteriorbernoulli}\label{thm:posterior_bernoulli}
	Under \TTTS, for Bernoulli bandits and uniform priors, it holds almost surely that
	\[
	\lim_{n\rightarrow{\infty}} -\frac{1}{n}\log(1-a_{n,I^\star}) = \Gamma_{\beta}^\star.
	\]
\end{restatable}

\section{Numerical Illustrations}\label{sec:experiments}
This section is aimed at illustrating our theoretical results and supporting the practical use of Bayesian sampling rules for fixed-confidence BAI.   

We experiment with three different Bayesian sampling rules: \TCC, \TTTS and \TTEI, and we also include the Direct Tracking (\DT) rule of~\cite{garivier2016tracknstop} (which is adaptive to $\beta$), the \UGapE~\citep{gabillon2012ugape} algorithm, and a uniform baseline. In order to make a fair comparison, we use the Chernoff stopping rule~(\ref{eq:chernoffstoppingtime}) and associated recommendation rule for all of the sampling rules, including the uniform one, except for \UGapE which has its own stopping rule. Furthermore, we include a top-two variant of the Best Challenger (\BC) heuristic (see, e.g., \citealp{menard2019lma}). 

\BC selects the empirical best arm $\hat{I}_n$ with probability $\beta$ and the maximizer of $W_n(\hat{I}_n,j)$ with probability $1-\beta$, but also performs forced exploration (selecting any arm sampled less than $\sqrt{n}$ times at round $n$). \TCC can thus be viewed as a variant of \BC in which no forced exploration is needed to converge to $\bomega^\beta$, due to the noise added by replacing $\hat{I}_n$ with $I_n^{(1)}$.

We consider two simple instances with arms means given by $\bmu_1 = [0.5 \ 0.9 \ 0.4 \ 0.45 \ 0.44999]$, and $\bmu_2 = [1 \ 0.8 \ 0.75 \ 0.7]$ respectively. We run simulations for both Gaussian (with $\sigma=1$) and Bernoulli bandits with a risk parameter $\delta=0.01$. 
Figure~\ref{fig:confidence} reports the empirical distribution of $\tau_\delta$ under the different sampling rules, estimated over 1000 independent runs. 

These figures provide several insights: (1) \TCC is competitive with, and sometimes slightly better than \TTTS and \TTEI in terms of sample complexity. (2) The \UGapE algorithm has a larger sample complexity than the uniform sampling rule, which highlights the importance of the stopping rule in the fixed-confidence setting. (3) The fact that \DT performs best is not surprising, since it converges to $\bomega^{\beta^\star}$ and achieves minimal sample complexity. However, in terms of computation time, \DT is much worse than other sampling rules, as can be seen in Table~\ref{table:time}, which reports the average execution time of one step of each sampling rule for $\mu_1$ in the Gaussian case. (4) \TTTS also suffers from computational costs, whose origins are explained in Section~\ref{sec:algorithm}, unlike \TCC and \TTEI. 
Although \TTEI is already computationally more attractive than \TTTS, its practical benefits are limited to the Gaussian case, since the \emph{Expected Improvement} (EI) does not have a closed form beyond this case and its approximation would be costly. In contrast, \TCC can be applied for other distributions.

\section{Conclusion}
We have advocated the use of a Bayesian sampling rule for BAI. In particular, we proved that \TTTS and a computationally advantageous approach \TCC, are both $\beta$-optimal in the fixed-confidence setting, for Gaussian bandits. 
We further extended the Bayesian optimality properties established by \cite{russo2016ttts} to more practical choices of models and prior distributions.

In order to be optimal, the sampling rules studied in this paper would need to use the oracle tuning $\beta^\star =\argmax_{\beta \in [0,1]} \Gamma_\beta^\star$, which is not feasible. In future work, we will investigate an efficient online tuning of~$\beta$ to circumvent this issue. We also plan to investigate the extension of \TCC to more general pure exploration problems, as an alternative to approaches recently proposed by~\citet{menard2019lma,degenne2019game}.

Finally, it is also important to study Bayesian sampling rules in the fixed-budget setting which is more plausible in many application scenarios such as applying BAI for automated machine learning~\citep{hoffman2014bayesgap,li2017hyperband,shang2019dttts}.
\vfil


\paragraph{Acknowledgement}
The research presented was supported by European CHIST-ERA project DELTA, French National Research Agency projects BADASS (ANR-16-CE40-0002) and BOLD (ANR-19-CE23-0026-04), the Inria-CWI associate team 6PAC and LUF travel grant number W19204-1-35.

\bibliography{Major}
\bibliographystyle{apalike}

\newpage
\onecolumn
\appendix

\section{Outline}\label{app:outline}

The appendix of this paper is organized as follows:
\begin{itemize}[label=$\square$]
    \item Appendix~\ref{app:confidence_ttts} provides the complete fixed-confidence analysis of \TTTS (Gaussian case).
    \item Appendix~\ref{app:confidence_t3c} provides the complete fixed-confidence analysis of \TCC (Gaussian case).
    \item Appendix~\ref{app:confidence} is dedicated to Lemma~\ref{lemma:confidence}.
    \item Appendix~\ref{app:lemmas} is dedicated to crucial technical lemmas.
    \item Appendix~\ref{app:posterior_gaussian} is the proof to the posterior convergence Theorem~\ref{thm:posterior_gaussian} (Gaussian case).
    \item Appendix~\ref{app:posterior_beta} is the proof to the posterior convergence Theorem~\ref{thm:posterior_bernoulli} (Beta-Bernoulli case).
\end{itemize}

\section{Useful Notation for the Appendices}\label{app:notation}

In this section, we provide a list of useful notation that is applied in appendices (including reminders of previous notation in the main text and some new ones).

\begin{itemize}
    \item Recall that $d(\mu_1;\mu_2)$ denotes the \texttt{KL}-divergence between two distributions parametrized by their means $\mu_1$ and $\mu_2$. For Gaussian distributions, we know that
    \[
        d(\mu_1;\mu_2) = \frac{(\mu_1-\mu_2)^2}{2\sigma^2}.
    \]
    When it comes to Bernoulli distributions, we denote this with $kl$, i.e.\
    \[
        kl(\mu_1;\mu_2) = \mu_1 \ln \left( \frac{\mu_1}{\mu_2} \right) + (1-\mu_1) \ln  \left( \frac{1-\mu_1}{1-\mu_2} \right).
    \]
    \item $\cB eta(\cdot,\cdot)$ denotes a Beta distribution.
    \item $\cB ern(\cdot)$ denotes a Bernoulli distribution.
    \item $\cB(\cdot)$ denotes a Binomial distribution.
    \item $\cN(\cdot,\cdot)$ denotes a normal distribution.
    \item $Y_{n,i}$ is the reward of arm $i$ at time $n$.
    \item $Y_{n,I_n}$ is the observation of the sampling rule at time $n$.
    \item $\cF_n \eqdef \sigma(I_1,Y_{1,I_1}, I_2, Y_{2,I_2}, \cdots, I_n, Y_{n,I_n})$ is the filtration generated by the first $n$ observations.
    \item $\psi_{n,i} \eqdef \PP{I_n = i | \cF_{n-1}}$.
    \item $\Psi_{n,i} \eqdef \sum_{l=1}^n \psi_{l,i}$.
    \item For the sake of simplicity, we further define $\bar{\psi}_{n,i} \eqdef \frac{\Psi_{n,i}}{n}$.
    \item $T_{n,i}$ is the number of pulls of arm $i$ before round $n$.
    \item $\bT_n$ denotes the vector of the number of arm selections. 
    \item $I_n^\star \eqdef \argmax_{i\in\cA} \mu_{n,i}$ denotes the empirical best arm at time $n$.
    \item For any $a, b > 0$, define a function $C_{a,b}$ s.t. $\forall y$,
    \[
        C_{a,b}(y) \eqdef (a+b-1) kl (\frac{a-1}{a+b-1}; y).
    \]
    \item We define the minimum and the maximum means gap as
    \[
        \Delta_{\text{min}} \eqdef \min_{i \neq j}|\mu_i-\mu_j|; \Delta_{\text{max}} \eqdef \max_{i \neq j}|\mu_i-\mu_j|.
    \]
    \item We introduce two indices
    \[
        J_n^{(1)} \eqdef \argmax_{j} a_{n,j}, J_n^{(2)} \eqdef \argmax_{j\neq J_n^{(1)}} a_{n,j}.
    \]
    Note that $J_n^{(1)}$ coincides with the Bayesian recommendation index $J_n$.
    \item Two real-valued sequences $(a_n)$ and $(b_n)$ are are said to be logarithmically equivalent if
    \[
        \lim_{n\rightarrow\infty}\frac{1}{n}\log\left(\frac{a_n}{b_n}\right) =0,
    \]
    and we denote this by $a_n \doteq b_n$.
\end{itemize}

\section{Fixed-Confidence Analysis for \TTTS}\label{app:confidence_ttts}

This section is entirely dedicated to \TTTS.

\subsection{Sufficient exploration of all arms, proof of Lemma~\ref{lemma:sufficient_exploration} under \TTTS}\label{app:confidence_ttts.exploration}

To prove this lemma, we introduce the two following sets of indices for a given $L>0$: $\forall n\in\NN$ we define
\[
    U_n^L \eqdef \{i: T_{n,i} < \sqrt{L}\},
\]
\[
    V_n^L \eqdef \{i: T_{n,i} < L^{3/4}\}.
\]
It is seemingly non trivial to manipulate directly \TTTS's candidate arms, we thus start by connecting \TTTS with \TTPS (top two probability sampling). \TTPS is another sampling rule presented by~\cite{russo2016ttts} for which the two candidate samples are defined as in Appendix~\ref{app:notation}, we recall them in the following.
\[
    J_n^{(1)} \eqdef \argmax_{j} a_{n,j}, J_n^{(2)} \eqdef \argmax_{j\neq J_n^{(1)}} a_{n,j}.
\]
Lemma~\ref{lemma:sufficient_exploration} is proved via the following sequence of lemmas.

\begin{lemma}\label{lemma:link_ttps}
    There exists $L_1 = \text{Poly}(W_1)$ s.t. if $L > L_1$, for all $n$, $U_n^L \neq \emptyset$ implies $J_n^{(1)} \in V_n^L$ or $J_n^{(2)} \in V_n^L$.
\end{lemma}

\begin{proof}
    If $J_n^{(1)} \in V_n^L$, then the proof is finished. Now we assume that $J_n^{(1)} \in \bar{V_n^L}$, and we prove that $J_n^{(2)} \in V_n^L$.
    \paragraph{Step 1} According to Lemma~\ref{lemma:means}, there exists $L_2 = \text{Poly}(W_1)$ s.t. $\forall L > L_2, \forall i \in \bar{U_n^L}$,
    \begin{align*}
        |\mu_{n,i} - \mu_{i}| &\leq \sigma W_1 \sqrt{\frac{\log(e+T_{n,i})}{1+T_{n,i}}}\\
                              &\leq \sigma W_1 \sqrt{\frac{\log(e+\sqrt{L})}{1+\sqrt{L}}}\\
                              &\leq \sigma W_1 \frac{\Delta_{\text{min}}}{4\sigma W_1} = \frac{\Delta_{\text{min}}}{4}.
    \end{align*}
    The second inequality holds since $x\mapsto \frac{\log(e+x)}{1+x}$ is a decreasing function. The third inequality holds for a large $L>L_2$ with $L_2 = \ldots$.
    
    \paragraph{Step 2} We now assume that $L > L_2$, and we define 
    \[
        \bar{J_n^\star} \eqdef \argmax_{j\in\bar{U_n^L}} \mu_{n,j} = \argmax_{j\in\bar{U_n^L}} \mu_j.
    \]
    The last equality holds since $\forall j \in \bar{U_n^L}$, $|\mu_{n,i}-\mu_i| \leq \Delta_{\text{min}}/4$. We show that there exists $L_3 = \text{Poly}(W_1)$ s.t. $\forall L > L_3$, 
    \[
        \bar{J_n^\star} = J_n^{(1)}.
    \] We proceed by contradiction, and suppose that $\bar{J_n^\star} \neq J_n^{(1)}$, then $\mu_{n,J_n^{(1)}} < \mu_{n,\bar{J_n^\star}}$, since $J_n^{(1)} \in \bar{V_n^L} \subset \bar{U_n^L}$. However, we have
    \begin{align*}
        a_{n,J_n^{(1)}} &= \Pi_{n}\left[\theta_{J_n^{(1)}} > \max_{j\neq J_n^{(1)}}\theta_j\right]\\
                             &\leq \Pi_{n}\left[\theta_{J_n^{(1)}} > \theta_{\bar{J_n^\star}}\right]\\
                             &\leq \frac{1}{2}\expp{-\frac{(\mu_{n,J_n^{(1)}}-\mu_{n,\bar{J_n^\star}})^2}{2\sigma^2(1/T_{n,J_n^{(1)}}+1/T_{n,\bar{J_n^\star}})}}.
    \end{align*}
    The last inequality uses the Gaussian tail inequality (\ref{gaussian_upper}) of Lemma~\ref{lemma:gaussiantails}. On the other hand,
    \begin{align*}
        |\mu_{n,J_n^{(1)}} - \mu_{n,\bar{J_n^\star}}| &= |\mu_{n,J_n^{(1)}} - \mu_{J_n^{(1)}} + \mu_{J_n^{(1)}} - \mu_{\bar{J_n^\star}} + \mu_{\bar{J_n^\star}} -\mu_{n,\bar{J_n^\star}}|\\
                                                      &\geq
        |\mu_{J_n^{(1)}} - \mu_{\bar{J_n^\star}}| - |\mu_{n,J_n^{(1)}} - \mu_{J_n^{(1)}} + \mu_{\bar{J_n^\star}} -\mu_{n,\bar{J_n^\star}}|\\
                                                      &\geq
        \Delta_{\text{min}} - (\frac{\Delta_{\text{min}}}{4} + \frac{\Delta_{\text{min}}}{4})\\
                                                      &=
        \frac{\Delta_{\text{min}}}{2},
    \end{align*}
    and
    \[
        \frac{1}{T_{n,J_n^{(1)}}}+\frac{1}{T_{n,\bar{J_n^\star}}} \leq \frac{2}{\sqrt{L}}.
    \]
    Thus, if we take $L_3$ s.t. 
    \[
        \expp{-\frac{\sqrt{L_3}\Delta_{\text{min}}^2}{16\sigma^2}} \leq \frac{1}{2K},
    \]
    then for any $L > L_3$, we have
    \[
        a_{n,J_n^{(1)}} \leq \frac{1}{2K} < \frac{1}{K},
    \]
    which contradicts the definition of $J_n^{(1)}$. We now assume that $L > L_3$, thus $J_n^{(1)}=\bar{J_n^\star}$.
    
    \paragraph{Step 3} We finally show that for $L$ large enough, $J_n^{(2)} \in V_n^L$. First note that $\forall j \in \bar{V_n^L}$, we have
    \begin{equation}
                a_{n,j} \leq \Pi_{n}\left[\theta_j \geq \theta_{\bar{J_n^{\star}}}\right] \leq \expp{-\frac{L^{3/4}\Delta_{\text{min}}^2}{16\sigma^2}}. \label{eq:upper_bound_anj_explo}
    \end{equation}

    This last inequality can be proved using the same argument as Step 2. Now we define another index $J_n^\star \eqdef \argmax_{j\in U_n^L}\mu_{n,j}$ and the quantity $c_n \eqdef \max(\mu_{n,J_n^\star},\mu_{n,\bar{J_n^\star}})$. We can lower bound $a_{n,J_n^\star}$ as follows:
    \begin{align*}
        a_{n,J_n^\star} &\geq \Pi_{n}\left[\theta_{J_n^\star}\geq c_n\right]\prod_{j\neq J_n^\star}\Pi_{n}\left[\theta_j\leq c_n\right]\\
                             &= \Pi_{n}\left[\theta_{J_n^\star}\geq c_n\right]\prod_{j\neq J_n^\star;j\in U_n^L}\Pi_{n}\left[\theta_j\leq c_n\right]\prod_{j\in \bar{U_n^L}}\Pi_{n}\left[\theta_j\leq c_n\right]\\
                             &\geq \Pi_{n}\left[\theta_{J_n^\star}\geq c_n\right] \frac{1}{2^{K-1}}.
    \end{align*}
    Now there are two cases:
    \begin{itemize}
        \item If $\mu_{n,J_n^\star} > \mu_{n,\bar{J_n^\star}}$, then we have
        \[
            \Pi_{n}\left[\theta_{J_n^\star}\geq c_n\right] = \Pi_{n}\left[\theta_{J_n^\star}\geq \mu_{n,J_n^\star}\right] \geq \frac{1}{2}.
        \]
        \item If $\mu_{n,J_n^\star} < \mu_{n,\bar{J_n^\star}}$, then we can apply the Gaussian tail bound (\ref{gaussian_lower}) of Lemma~\ref{lemma:gaussiantails}, and we obtain
        \begin{align*}
            \Pi_{n}\left[\theta_{J_n^\star}\geq c_n\right] &= \Pi_{n}\left[\theta_{J_n^\star}\geq \mu_{n,\bar{J_n^\star}}\right] =
            \Pi_{n}\left[\theta_{J_n^\star}\geq \mu_{n,J_n^\star} + (\mu_{n,\bar{J_n^\star}}-\mu_{n,J_n^\star})\right]\\
                                            &\geq
            \frac{1}{\sqrt{2\pi}} \expp{-\frac{1}{2}\left( 1-\frac{\sqrt{T_{n,J_n^\star}}}{\sigma}(\mu_{n,J_n^\star}-\mu_{n,\bar{J_n^\star}}) \right)^2}\\
                                            &=
            \frac{1}{\sqrt{2\pi}} \expp{-\frac{1}{2}\left( 1+\frac{\sqrt{T_{n,J_n^\star}}}{\sigma}(\mu_{n,\bar{J_n^\star}}-\mu_{n,J_n^\star}) \right)^2}.
        \end{align*}
        On the other hand, by Lemma~\ref{lemma:means}, we know that
        \begin{align*}
            |\mu_{n,J_n^\star} - \mu_{n,\bar{J_n^\star}}| &= |\mu_{n,J_n^\star} - \mu_{J_n^\star} + \mu_{J_n^\star} - \mu_{\bar{J_n^\star}} + \mu_{\bar{J_n^\star}} - \mu_{n,\bar{J_n^\star}}|\\
                                                          &\leq
            |\mu_{J_n^\star} - \mu_{\bar{J_n^\star}}| + \sigma W_1 \sqrt{\frac{\log(e+T_{n,J_n^\star})}{1+T_{n,J_n^\star}}} + \sigma W_1 \sqrt{\frac{\log(e+T_{n,\bar{J_n^\star}})}{1+T_{n,\bar{J_n^\star}}}}\\
                                                          &\leq
            |\mu_{J_n^\star} - \mu_{\bar{J_n^\star}}| + 2\sigma W_1 \sqrt{\frac{\log(e+T_{n,J_n^\star})}{1+T_{n,J_n^\star}}}\\
                                                          &\leq
            \Delta_{\max} + 2\sigma W_1 \sqrt{\frac{\log(e+T_{n,J_n^\star})}{1+T_{n,J_n^\star}}}.
        \end{align*}
        Therefore,
        \begin{align*}
            \Pi_{n}\left[\theta_{J_n^\star}\geq c_n\right] &\geq \frac{1}{\sqrt{2\pi}} \expp{-\frac{1}{2}\left( 1+\frac{\sqrt{T_{n,J_n^\star}}}{\sigma}\left(\Delta_{\max} + 2\sigma W_1 \sqrt{\frac{\log(e+T_{n,J_n^\star})}{1+T_{n,J_n^\star}}}\right) \right)^2}\\
                                            &\geq
            \frac{1}{\sqrt{2\pi}} \expp{-\frac{1}{2}\left( 1+\frac{\sqrt{\sqrt{L}}}{\sigma}\left(\Delta_{\max} + 2\sigma W_1 \sqrt{\frac{\log(e+\sqrt{L})}{1+\sqrt{L}}}\right) \right)^2}\\
                                            &\geq
            \frac{1}{\sqrt{2\pi}} \expp{-\frac{1}{2}\left( 1+\frac{L^{1/4}\Delta_{\max}}{\sigma} + 2 W_1 \sqrt{\log(e+\sqrt{L})} \right)^2}.
        \end{align*}
    \end{itemize}
    Now we have
    \[
        a_{n,J_n^\star} \geq \max\left( \left(\frac{1}{2}\right)^K, \left(\frac{1}{2}\right)^{K-1}\frac{1}{\sqrt{2\pi}}\expp{-\frac{1}{2}\left( 1+\frac{L^{1/4}\Delta_{\max}}{\sigma} + 2 W_1 \sqrt{\log(e+\sqrt{L})} \right)^2} \right),
    \]
    and we have $\forall j\in \bar{V_n^L}$, $a_{n,j}\leq \expp{-L^{3/4}\Delta_{\text{min}}^2/(16\sigma^2)}$, thus there exists $L_4 = \text{Poly}(W_1)$ s.t. $\forall L > L_4$, $\forall j \in \bar{V_n^L}$,
    \[
        a_{n,j} \leq \frac{a_{n,J_n^\star}}{2},
    \]
    and by consequence, $J_n^{(2)}\in V_n^L$.
    
    Finally, taking $L_1 = \max(L_2, L_3, L_4)$, we have $\forall L > L_1$, either $J_n^{(1)} \in V_n^L$ or $J_n^{(2)} \in V_n^L$.
\end{proof}

Next we show that there exists at least one arm in $V_n^L$ for whom the probability of being pulled is large enough. More precisely, we prove the following lemma.

\begin{lemma}\label{lemma:psi_min_ttts}
    There exists $L_1 = \text{Poly}(W_1)$ s.t. for $L > L_1$ and for all $n$ s.t. $U_n^L \neq \emptyset$, then there exists $J_n \in V_n^L$ s.t.
    \[
        \psi_{n,J_n} \geq \frac{\min(\beta,1-\beta)}{K^2} \eqdef \psi_{\min}.
    \]
\end{lemma}

\begin{proof}
    Using Lemma~\ref{lemma:link_ttps}, we know that $J_n^{(1)}$ or $J_n^{(2)} \in V_n^L$. On the other hand, we know that
    \[
        \forall i\in\cA, \psi_{n,i} = a_{n,i} \left(\beta + (1-\beta) \sum_{j\neq i} \frac{a_{n,j}}{1-a_{n,j}}\right).
    \]
    Therefore we have
    \[
        \psi_{n,J_n^{(1)}} \geq \beta a_{n,J_n^{(1)}} \geq \frac{\beta}{K},
    \]
    since $\sum_{i\in\cA} a_{n,i} = 1$, and
    \begin{align*}
        \psi_{n,J_n^{(2)}} &\geq (1-\beta) a_{n,J_n^{(2)}} \frac{a_{n,J_n^{(1)}}}{1-a_{n,J_n^{(1)}}}\\
                           &= (1-\beta) a_{n,J_n^{(1)}} \frac{a_{n,J_n^{(2)}}}{1-a_{n,J_n^{(1)}}}\\
                           &\geq \frac{1-\beta}{K^2},
    \end{align*}
    since $a_{n,J_n^{(1)}} \geq 1/K$ and $\sum_{i\neq J_n^{(1)}} a_{n,i}/(1-a_{n,J_n^{(1)}}) = 1 $, thus $a_{n,J_n^{(2)}}/(1-a_{n,J_n^{(1)}}) \geq 1/K$.
\end{proof}

The rest of this subsection is quite similar to that of~\cite{qin2017ttei}. Indeed, with the above lemma, we can show that the set of poorly explored arms $U_n^L$ is empty when $n$ is large enough.

\begin{lemma}\label{lemma:poorly_explored_ttts}
    Under \TTTS, there exists $L_0 = \text{Poly}(W_1,W_2)$ s.t. $\forall L > L_0$, $U_{\floor{KL}}^L = \emptyset$.
\end{lemma}

\begin{proof}
    We proceed by contradiction, and we assume that $U_{\floor{KL}}^L$ is not empty. Then for any $1 \leq \ell \leq \floor{KL}$, $U_{\ell}^L$ and $V_{\ell}^L$ are non empty as well.
    
    There exists a deterministic $L_5$ s.t. $\forall L > L_5$,
    \[
        \floor{L} \geq KL^{3/4}.
    \]
    Using the pigeonhole principle, there exists some $i \in \cA$ s.t. $T_{\floor{L},i} \geq L^{3/4}$. Thus, we have $|V_{\floor{L}}^L| \leq K-1$.
    
    Next, we prove $|V_{\floor{2L}}^L| \leq K-2$. Otherwise, since $U_{\ell}^L$ is non-empty for any $\floor{L}+1 \leq \ell \leq \floor{2L}$, thus by Lemma~\ref{lemma:psi_min_ttts}, there exists $J_{\ell}\in V_\ell^L$ s.t. $\psi_{\ell,J_{\ell}} \geq \psi_{\min}$. Therefore,
    \[
        \sum_{i\in V_{\ell}^L} \psi_{\ell,i} \geq \psi_{\min},
    \]
    and
    \[
        \sum_{i\in V_{\floor{L}}^L} \psi_{\ell,i} \geq \psi_{\min}
    \]
    since $V_{\ell}^L \subset V_{\floor{L}}^L$. Hence, we have
    \[
        \sum_{i\in V_{\floor{L}}^L} (\Psi_{\floor{2L},i} - \Psi_{\floor{L},i}) = \sum_{\ell = \floor{L}+1}^{\floor{2L}} \sum_{i\in V_{\floor{L}}^L} \psi_{\ell,i} \geq \psi_{\min}\floor{L}.
    \]
    Then, using Lemma~\ref{lemma:link}, there exists $L_6 = \text{Poly}(W_2)$ s.t. $\forall L > L_6$, we have
    \begin{align*}
        \sum_{i\in V_{\floor{L}}^L} (T_{\floor{2L},i} - T_{\floor{L},i}) &\geq \sum_{i\in V_{\floor{L}}^L} (\Psi_{\floor{2L},i} - \Psi_{\floor{L},i} - 2W_2\sqrt{\floor{2L}\log(e^2+\floor{2L})})\\
                          &\geq 
        \sum_{i\in V_{\floor{L}}^L} (\Psi_{\floor{2L},i} - \Psi_{\floor{L},i}) - 2KW_2\sqrt{\floor{2L}\log(e^2+\floor{2L})}\\
                          &\geq \psi_{\min}\floor{L} - 2KW_2C_2\floor{L}^{3/4}\\
                          &\geq KL^{3/4},
    \end{align*}
    where $C_2$ is some absolute constant. Thus, we have one arm in $V_{\floor{L}}^L$ that is pulled at least $L^{3/4}$ times between $\floor{L}+1$ and $\floor{2L}$, thus $|V_{\floor{2L}}^L| \leq K-2$.
    
    By induction, for any $1 \leq k \leq K$, we have $|V_{\floor{kL}}^L| \leq K-k$, and finally if we take $L_0=\max(L_1,L_5,L_6)$, then $\forall L > L_0$, $U_{\floor{KL}}^L = \emptyset$.
\end{proof}

We can finally conclude the proof of Lemma~\ref{lemma:sufficient_exploration} for \TTTS.

\paragraph{Proof of Lemma~\ref{lemma:sufficient_exploration}}
Let $N_1 = KL_0$ where $L_0 = \text{Poly}(W_1,W_2)$ is chosen according to Lemma~\ref{lemma:poorly_explored_ttts}. For all $n > N_1$, we let $L=n/K$, then by Lemma~\ref{lemma:poorly_explored_ttts}, we have $U_{\floor{KL}}^L = U_n^{n/K}$ is empty, which concludes the proof.

\hfill\BlackBox\\[2mm]

\subsection{Concentration of the empirical means, proof of Lemma~\ref{lemma:tracking_means} under \TTTS}\label{app:confidence_ttts.means}

As a corollary of the previous section, we can show the concentration of $\mu_{n,i}$ to $\mu_i$ for \TTTS\footnote{this proof is the same as Proposition 3 of~\cite{qin2017ttei}}.

By Lemma~\ref{lemma:means}, we know that $\forall i\in\cA$ and $n\in\NN$,
\[
    |\mu_{n,i}-\mu_i| \leq \sigma W_1 \sqrt{\frac{\log(e+T_{n,i})}{T_{n,i}+1}}.
\]
According to the previous section, there exists $N_1 = \text{Poly}(W_1,W_2)$ s.t. $\forall n \geq N_1$ and $\forall i\in\cA$, $T_{n,i} \geq \sqrt{n/K}$. Therefore,
\[
    |\mu_{n,i}-\mu_i| \leq \sqrt{\frac{\log(e+\sqrt{n/K})}{\sqrt{n/K}+1}},
\]
since $x \mapsto \log(e+x)/(x+1)$ is a decreasing function. There exists $N_2' = \text{Poly}(\epsilon,W_1)$ s.t. $\forall n \geq N_2'$,
\[
    \sqrt{\frac{\log(e+\sqrt{n/K})}{\sqrt{n/K}+1}} \leq \sqrt{\frac{2(n/K)^{1/4}}{\sqrt{n/K}+1}} \leq \frac{\epsilon}{\sigma W_1}.
\]
Therefore, $\forall n \geq N_2 \eqdef \max\{N_1,N_2'\}$, we have
\[
    |\mu_{n,i}-\mu_i| \leq \sigma W_1 \frac{\epsilon}{\sigma W_1}.
\]

\subsection{Measurement effort concentration of the optimal arm, proof of Lemma~\ref{lemma:tracking_best} under \TTTS}\label{app:confidence_ttts.best_arm}

In this section we show that the empirical arm draws proportion of the true best arm for \TTTS concentrates to $\beta$ when the total number of arm draws is sufficiently large.

The proof is established upon the following lemmas. First, we prove that the empirical best arm coincides with the true best arm when the total number of arm draws goes sufficiently large.

\begin{lemma}\label{lemma:empirical_best}
    Under \TTTS, there exists $M_1 = \text{Poly}(W_1,W_2)$ s.t. $\forall n > M_1$, we have $I_n^\star = I^\star = J_n^{(1)}$ and $\forall i \neq I^\star$,
    \[
        a_{n,i} \leq \expp{-\frac{\Delta_{\text{min}}^2}{16\sigma^2}\sqrt{\frac{n}{K}}}.
    \]
\end{lemma}

\begin{proof}
    Using Lemma~\ref{lemma:tracking_means} with $\epsilon = \Delta_{\min}/4$, there exists $N_1' = \text{Poly}(4/\Delta_{\min},W_1,W_2)$ s.t. $\forall n > N_1'$,
    \[
        \forall i\in\cA, |\mu_{n,i} - \mu_i| \leq \frac{\Delta_{\min}}{4}, 
    \]
    which implies that starting from a known moment, $\mu_{n,I^\star} > \mu_{n,i}$ for all $i\neq I^\star$, hence $I_n^\star = I^\star$. Thus, $\forall i \neq I^\star$,
    \begin{align*}
        a_{n,i} &= \Pi_{n}\left[\theta_i > \max_{j\neq i} \theta_j\right]\\
                             &\leq \Pi_{n}\left[\theta_i > \theta_{I^\star}\right]\\
                             &\leq \frac{1}{2}\expp{-\frac{(\mu_{n,i}-\mu_{n,I^\star})^2}{2\sigma^2(1/T_{n,i}+1/T_{n,I^\star})}}.
    \end{align*}
    The last inequality uses the Gaussian tail inequality of (\ref{gaussian_upper}) Lemma~\ref{lemma:gaussiantails}. Furthermore,
    \begin{align*}
        (\mu_{n,i} - \mu_{n,I^\star})^2 &= (|\mu_{n,i} - \mu_{n,I^\star}|)^2\\
                                        &= (|\mu_{n,i} - \mu_i + \mu_i - \mu_{I^\star} + \mu_{I^\star} -\mu_{n,I^\star}|)^2\\
                                        &\geq (|\mu_i - \mu_{I^\star}| - |\mu_{n,i} - \mu_i + \mu_{I^\star} -\mu_{n,I^\star}|)^2\\
                                        &\geq \left(\Delta_{\text{min}} - \left(\frac{\Delta_{\text{min}}}{4} + \frac{\Delta_{\text{min}}}{4}\right)\right)^2 = \frac{\Delta_{\text{min}}^2}{4},
    \end{align*}
    and according to Lemma~\ref{lemma:sufficient_exploration}, we know that there exists $M_2 = \text{Poly}(W_1,W_2)$ s.t. $\forall n > M_2$,
    \[
        \frac{1}{T_{n,i}}+\frac{1}{T_{n,I^\star}} \leq \frac{2}{\sqrt{n/K}}.
    \]
    Thus, $\forall n > \max\{N_1',M_2\}$, we have
    \[
        \forall i\neq I^\star, a_{n,i} \leq \expp{-\frac{\Delta_{\text{min}}^2}{16\sigma^2}\sqrt{\frac{n}{K}}}.
    \]
    Then, we have
    \[
        a_{n,I^\star} = 1 - \sum_{i\neq I^\star} a_{n,i} \geq 1-(K-1)\expp{-\frac{\Delta_{\text{min}}^2}{16\sigma^2}\sqrt{\frac{n}{K}}}.
    \]
    There exists $M_2'$ s.t. $\forall n > M_2'$, $a_{n,I^\star}>1/2$, and by consequence $I^\star = J_n^{(1)}$. Finally taking $M_1 \eqdef \max\{N_1', M_2, M_2'\}$ concludes the proof.
\end{proof}

Before we prove Lemma~\ref{lemma:tracking_best}, we first show that $\Psi_{n,I^\star}/n$ concentrates to $\beta$.

\begin{lemma}\label{lemma:psi_best}
    Under \TTTS, fix a constant $\epsilon>0$, there exists $M_3 = \text{Poly}(\epsilon,W_1,W_2)$ s.t. $\forall n > M_3$, we have
    \[
        \left| \frac{\Psi_{n,I^\star}}{n}-\beta \right| \leq \epsilon.
    \]
\end{lemma}

\begin{proof}
    By Lemma~\ref{lemma:empirical_best}, we know that there exists $M_1' = \text{Poly}(W_1,W_2)$ s.t. $\forall n > M_1'$, we have $I_n^\star = I^\star = J_n^{(1)}$ and $\forall i \neq I^\star$,
    \[
        a_{n,i} \leq \expp{-\frac{\Delta_{\text{min}}^2}{16\sigma^2}\sqrt{\frac{n}{K}}}.
    \]
    Note also that $\forall n\in\NN$, we have
    \[
        \psi_{n,I^\star} = a_{n,I^\star} \left(\beta + (1-\beta) \sum_{j\neq I^\star} \frac{a_{n,j}}{1-a_{n,j}}\right).
    \]
    We proceed the proof with the following two steps.
    
    \paragraph{Step 1} We first lower bound $\Psi_{n,I^\star}$ for a given $\epsilon$. Take $M_4 > M_1'$ that we decide later, we have $\forall n > M_4$,
    \begin{align*}
        \frac{\Psi_{n,I^\star}}{n} &= \frac{1}{n}\sum_{l=1}^{n}\psi_{l,I^\star} = \frac{1}{n}\sum_{l=I^\star}^{M_4}\psi_{l,I^\star} + \frac{1}{n}\sum_{l=M_4+1}^{n}\psi_{l,I^\star}\\
                             &\geq \frac{1}{n}\sum_{l=M_4+1}^{n}\psi_{l,I^\star} \geq \frac{1}{n}\sum_{l=M_4+1}^{n} a_{l,I^\star}\beta\\
                             &= \frac{\beta}{n}\sum_{l=M_4+1}^{n} \left(1-\sum_{j\neq I^\star}a_{l,j}\right)\\
                             &\geq \frac{\beta}{n}\sum_{l=M_4+1}^{n} \left(1-(K-1)\expp{-\frac{\Delta_{\text{min}}^2}{16\sigma^2}\sqrt{\frac{l}{K}}}\right)\\
                             &= \beta - \frac{M_4}{n}\beta - \frac{\beta}{n}\sum_{l=M_4+1}^{n} (K-1)\expp{-\frac{\Delta_{\text{min}}^2}{16\sigma^2}\sqrt{\frac{l}{K}}}\\
                             &\geq \beta - \frac{M_4}{n}\beta - \frac{(n-M_4)}{n}\beta(K-1)\expp{-\frac{\Delta_{\text{min}}^2}{16\sigma^2}\sqrt{\frac{M_4}{K}}}\\
                             &\geq \beta - \frac{M_4}{n}\beta - \beta(K-1)\expp{-\frac{\Delta_{\text{min}}^2}{16\sigma^2}\sqrt{\frac{M_4}{K}}}.
    \end{align*}
    For a given constant $\epsilon>0$, there exists $M_5$ s.t. $\forall n > M_5$,
    \[
        \beta(K-1)\expp{-\frac{\Delta_{\text{min}}^2}{16\sigma^2}\sqrt{\frac{n}{K}}} < \frac{\epsilon}{2}.
    \]
    Furthermore, there exists $M_6 = \text{Poly}(\epsilon/2,M_5)$ s.t. $\forall n > M_6$,
    \[
        \frac{M_5}{n}\beta < \frac{\epsilon}{2}.
    \]
    Therefore, if we take $M_4 \eqdef \max\{M_1', M_5, M_6\}$, we have $\forall n > M_4$,
    \[
        \frac{\Psi_{n,I^\star}}{n} \geq \beta - \epsilon.
    \]
    
    \paragraph{Step 2} On the other hand, we can also upper bound $\Psi_{n,I^\star}$. We have $\forall n > M_3$,
    \begin{align*}
        \frac{\Psi_{n,I^\star}}{n} &= \frac{1}{n}\sum_{l=1}^{n}\psi_{l,I^\star}\\
                             &= \frac{1}{n}\sum_{l=1}^{n}a_{l,I^\star}\left( \beta + (1-\beta)\sum_{j\neq I^\star}\frac{a_{l,j}}{1-a_{l,j}} \right)\\
                             &\leq \frac{1}{n}\sum_{l=1}^{n}a_{l,I^\star}\beta + \frac{1}{n}\sum_{l=1}^{n}a_{l,I^\star}(1-\beta)\sum_{j\neq I^\star}\frac{a_{l,j}}{1-a_{l,j}}\\
                             &\leq \beta + \frac{1}{n}\sum_{l=1}^{n}(1-\beta)\sum_{j\neq I^\star}\frac{a_{l,j}}{1-a_{l,j}}\\
                             &\leq \beta + \frac{1}{n}\sum_{l=1}^{n}(1-\beta)\sum_{j\neq I^\star}\frac{\expp{-\frac{\Delta_{\text{min}}^2}{16\sigma^2}\sqrt{\frac{l}{K}}}}{1-\expp{-\frac{\Delta_{\text{min}}^2}{16\sigma^2}\sqrt{\frac{l}{K}}}}.
    \end{align*}
    Since, for a given $\epsilon > 0$, there exists $M_8$ s.t. $\forall n > M_8$,
    \[
        \expp{-\frac{\Delta_{\text{min}}^2}{16\sigma^2}\sqrt{\frac{n}{K}}} < \frac{1}{2},
    \]
    and there exists $M_9$ s.t. $\forall n > M_9$,
    \[
        (1-\beta)(K-1)\expp{-\frac{\Delta_{\text{min}}^2}{16\sigma^2}\sqrt{\frac{n}{K}}} < \frac{\epsilon}{4}.
    \]
    Thus, $\forall n > M_{10} \eqdef \max\{M_8,M_9\}$,
    \begin{align*}
        \frac{\Psi_{n,I^\star}}{n} &\leq \beta + \frac{1-\beta}{n}\left(\sum_{l=1}^{M_{10}}\sum_{j\neq I^\star}\frac{\expp{-\frac{\Delta_{\text{min}}^2}{16\sigma^2}\sqrt{\frac{l}{K}}}}{1-\expp{-\frac{\Delta_{\text{min}}^2}{16\sigma^2}\sqrt{\frac{l}{K}}}} + \sum_{l=M_{10}+1}^{n}\sum_{j\neq I^\star}\frac{\expp{-\frac{\Delta_{\text{min}}^2}{16\sigma^2}\sqrt{\frac{l}{K}}}}{1-\expp{-\frac{\Delta_{\text{min}}^2}{16\sigma^2}\sqrt{\frac{l}{K}}}}\right)\\
                             &\leq \beta + \frac{1-\beta}{n}\sum_{l=1}^{M_{10}}\sum_{j\neq I^\star}\frac{\expp{-\frac{\Delta_{\text{min}}^2}{16\sigma^2}\sqrt{\frac{l}{K}}}}{1-\expp{-\frac{\Delta_{\text{min}}^2}{16\sigma^2}\sqrt{\frac{l}{K}}}} + 2(1-\beta)(K-1)\expp{-\frac{\Delta_{\text{min}}^2}{16\sigma^2}\sqrt{\frac{M_{10}}{K}}}\\
                             &\leq \beta + \frac{1-\beta}{n}\sum_{l=1}^{M_{10}}\sum_{j\neq I^\star}\frac{\expp{-\frac{\Delta_{\text{min}}^2}{16\sigma^2}\sqrt{\frac{l}{K}}}}{1-\expp{-\frac{\Delta_{\text{min}}^2}{16\sigma^2}\sqrt{\frac{l}{K}}}} + \frac{\epsilon}{2}.
    \end{align*}
    There exists $M_{11} = \text{Poly}(\epsilon/2,M_{10})$ s.t. $\forall n > M_{11}$,
    \[
        \frac{1-\beta}{n}\sum_{l=1}^{M_{10}}\sum_{j\neq I^\star}\frac{\expp{-\frac{\Delta_{\text{min}}^2}{16\sigma^2}\sqrt{\frac{l}{K}}}}{1-\expp{-\frac{\Delta_{\text{min}}^2}{16\sigma^2}\sqrt{\frac{l}{K}}}} < \frac{\epsilon}{2}.
    \]
    Therefore, $\forall n > M_7 \eqdef \max\{M_3,M_{11}\}$, we have
    \[
        \frac{\Psi_{n,I^\star}}{n} \leq \beta + \epsilon.
    \]
    
    \paragraph{Conclusion} Finally, combining the two steps and define $M_3 \eqdef \max\{M_4,M_7\}$, we have $\forall n > M_3$,
    \[
        \left|\frac{\Psi_{n,I^\star}}{n}-\beta\right| \leq \epsilon.
    \]
\end{proof}

With the help of the previous lemma and Lemma~\ref{lemma:link}, we can finally prove Lemma~\ref{lemma:tracking_best}.

\paragraph{Proof of Lemma~\ref{lemma:tracking_best}}
Fix an $\epsilon > 0$. Using Lemma~\ref{lemma:link}, we have $\forall n\in\NN$,
\[
    \left|\frac{T_{n,I^\star}}{n}-\frac{\Psi_{n,I^\star}}{n}\right| \leq \frac{W_2\sqrt{(n+1)\log(e^2+n)}}{n}.
\]
Thus there exists $M_{12}$ s.t. $\forall n > M_{12}$,
\[
    \left|\frac{T_{n,I^\star}}{n}-\frac{\Psi_{n,I^\star}}{n}\right| \leq \frac{\epsilon}{2}.
\]
And using Lemma~\ref{lemma:psi_best}, there exists $M_3' = \text{Poly}(\epsilon/2,W_1,W_2)$ s.t. $\forall n > M_3'$,
\[
    \left|\frac{\Psi_{n,I^\star}}{n}-\beta\right| \leq \frac{\epsilon}{2}.
\]
Again, according to Lemma~\ref{lemma:psi_min_ttts}, there exists $M_3'$ s.t. $\forall n > M_3'$,
\[
    \frac{\Psi_{n,I^\star}}{n} \leq \beta+\frac{\epsilon}{2}.
\]
Thus, if we take $N_3 \eqdef \max\{M_3',M_{12}\}$, then $\forall n > N_3$, we have
\[
    \left| \frac{T_{n,I^\star}}{n}-\beta \right| \leq \epsilon.
\]

\hfill\BlackBox\\[2mm]

\subsection{Measurement effort concentration of other arms, proof of Lemma~\ref{lemma:tracking_other} under \TTTS}\label{app:confidence_ttts.other_arms}

In this section, we show that, for \TTTS, the empirical measurement effort concentration also holds for other arms than the true best arm.
We first show that if some arm is overly sampled at time $n$, then its probability of being picked is reduced exponentially.

\begin{lemma}\label{lemma:over_allocation_finite_ttts} 
    Under \TTTS, for every $\xi \in (0,1)$, there exists $S_1 = \text{Poly}(1/\xi,W_1,W_2)$ such that for all $n > S_1$, for all $i\neq I^\star$, 
    \[
        \frac{\Psi_{n,i}}{n} \geq \omega_{i}^\beta + \xi  \ \ \Rightarrow \ \ \psi_{n,i} \leq \expp{-\epsilon_0(\xi) n}\,,
    \]
    where $\epsilon_0$ is defined in~\eqref{eq:def_epsilon0} below.
\end{lemma}

\begin{proof}
First, by Lemma~\ref{lemma:empirical_best}, there exists $M_1'' = \text{Poly}(W_1,W_2)$ s.t. $\forall n > M_1''$, 
\[
    I^\star = I_n^\star = J_n^{(1)}.
\]
Then, following the similar argument as in Lemma~\ref{lemma:sufficient_optimality}, one can show that for all $i\neq I^\star$ and for all $n > M_1''$,
\begin{align*}
	\psi_{n,i} &=a_{n,i} \left( \beta  + (1-\beta) \sum_{j \neq i} \frac{a_{n,j}}{1-a_{n,j}} \right)\\
	           &\leq a_{n,i} \beta  + a_{n,i} (1-\beta) \frac{\sum_{j \neq i} a_{n,j}}{1-a_{n,J_n^{(1)}}}\\
	           &= a_{n,i} \beta  + a_{n,i} (1-\beta) \frac{\sum_{j \neq i} a_{n,j}}{1-a_{n,I^\star}}\\
	           &\leq a_{n,i}\beta  + a_{n,i} (1-\beta) \frac{1}{1-a_{n,I^\star}}\\
	           &\leq \frac{a_{n,i}}{1-a_{n,I^\star}}\\
	           &\leq \frac{\Pi_{n}\left[ \theta_i\geq\theta_{I^\star} \right]}{\Pi_{n}\left[ \cup_{j\neq I^\star}\theta_j\geq\theta_{I^\star} \right]}\\
	           &\leq \frac{\Pi_{n}\left[ \theta_i\geq\theta_{I^\star} \right]}{\max_{j\neq I^\star} \Pi_{n}\left[ \theta_j\geq\theta_{I^\star} \right]}.
\end{align*}

Using the upper and lower Gaussian tail bounds from Lemma~\ref{lemma:gaussiantails}, we have
\begin{align*}
    \psi_{n,i} &\leq \ddfrac{\expp{- \frac{(\mu_{n,I^\star} - \mu_{n,i})^2}{2\sigma^2\left(1/T_{n,I^\star} + 1/T_{n,i}\right)}}}{\expp{- \min_{j\neq I^\star} \frac{1}{2}\left(\frac{(\mu_{n,I^\star} - \mu_{n,j})}{\sigma\sqrt{\left(1/T_{n,I^\star} + 1/T_{n,j}\right)}} -1\right)^2}}\\ 
               &=  \ddfrac{\expp{- n\frac{(\mu_{n,I^\star} - \mu_{n,i})^2}{2\sigma^2\left(n/T_{n,I^\star} + n/T_{n,i}\right)}}}{\expp{- {n}\left(\min_{j\neq I^\star} \frac{(\mu_{n,I^\star} - \mu_{n,j})}{\sqrt{2\sigma^2\left(n/T_{n,I^\star} + n/T_{n,j}\right)}} -\frac{1}{\sqrt{2n}}\right)^2}},
\end{align*}

where we assume that $n > S_2 = \text{Poly}(W_1,W_2)$ for which 
\[
    \frac{(\mu_{n,I^\star} - \mu_{n,i})^2}{\sigma^2\left(1/T_{n,I^\star} + 1/T_{n,i}\right)} \geq 1
\]
according to Lemma~\ref{lemma:sufficient_exploration}. From there we take a supremum over the possible allocations to lower bound the denominator and write  
\begin{align*}
    \psi_{n,i} &\leq \ddfrac{\expp{-n\frac{(\mu_{n,I^\star} - \mu_{n,i})^2}{2\sigma^2\left(n/T_{n,I^\star} + n/T_{n,i}\right)}}}{\expp{-n\left(\underset{\bomega : \omega_{I^\star} = T_{n,I^\star}/n}{\sup} \min_{j\neq I^\star}\frac{(\mu_{n,I^\star} - \mu_{n,i})}{\sqrt{2\sigma^2\left(1/\omega_{I^\star} + 1/\omega_{j}\right)}} -\frac{1}{\sqrt{2n}}\right)^2}} \\
               &= \ddfrac{\expp{- n\frac{(\mu_{n,I^\star} - \mu_{n,i})^2}{2\sigma^2\left(n/T_{n,I^\star} + n/T_{n,i}\right)}}}{\expp{-n\left( \sqrt{\Gamma^\star_{T_{n,I^\star}/n}\left(\bmu_n\right)} -\frac{1}{\sqrt{2n}}\right)^2}},
\end{align*}
where $\bmu_n \eqdef (\mu_{n,1},\cdots,\mu_{n,K})$, and $(\beta,\bmu)\mapsto \Gamma_{\beta}^\star(\bmu)$ represents a function that maps $\beta$ and $\bmu$ to the parameterized optimal error decay that any allocation rule can reach given parameter $\beta$ and a set of arms with means $\bmu$. Note that this function is continuous with respect to $\beta$ and $\bmu$ respectively.

Now, assuming $\Psi_{n,i}/n \geq \omega_{i}^\beta + \xi$ yields that there exists $S_2'\eqdef\text{Poly}(2/\xi,W_2)$ s.t. for all $n>S_2'$, $T_{n,i}/n \geq \omega_{i}^\beta + \xi/2$, and by consequence,
\[
    \psi_{n,i} \leq \expp{-n\underbrace{\left(\frac{(\mu_{n,I^\star} - \mu_{n,i})^2}{2\sigma^2\left(n/T_{n,I^\star} + 1/(\omega_i^\beta + \xi/2)\right)} -\Gamma^\star_{T_{n,I^\star}/n}\left(\bmu_{n}\right) - \frac{1}{2n} + \sqrt{\frac{2\Gamma^\star_{T_{n,I^\star}/n}\left(\bmu_{n}\right)}{n}}\right)}_{\epsilon_n(\xi)}}.
\]
Using Lemma~\ref{lemma:tracking_best}, we know that for any $\epsilon$, there exists $S_3 = \text{Poly}(1/\epsilon,W_1,W_2)$ s.t. $\forall n > S_3$, $|T_{n,I^\star}/n - \beta| \leq \epsilon$, and $\forall j\in\cA, |\mu_{n,j}-\mu_j| \leq \epsilon$. Furthermore, $(\beta,\bmu)\mapsto \Gamma_{\beta}^\star(\bmu)$ is continuous with respect to $\beta$ and $\bmu$, thus for a given $\epsilon_0$, there exists $S_3' = \text{Poly}(1/\epsilon_0,W_1,W_2)$ s.t. $\forall n > S_3'$, we have
\[
    \left|\epsilon_n(\xi) - \left(\frac{(\mu_{I^\star} - \mu_{i})^2}{2\sigma^2\left(1/\beta + 1/(\omega_i^\beta + \xi/2)\right)} - \Gamma_{\beta}^\star\right)\right| \leq \epsilon_0.
\]

Finally, define $S_1 \eqdef \max\{S_2,S_2',S_3'\}$, we have $\forall n > S_1$,
\[
    \psi_{n,i} \leq \expp{-\epsilon_0(\xi) n},
\]
where
\begin{equation}
    \label{eq:def_epsilon0}
    \epsilon_0(\xi) = \frac{(\mu_{I^\star} - \mu_{i})^2}{2\sigma^2\left(1/\beta + 1/(\omega_i^\beta + \xi/2)\right)} - \Gamma_{\beta}^\star + \epsilon_0\,.
\end{equation}

\end{proof}

Next, starting from some known moment, no arm is overly allocated. More precisely, we show the following lemma.

\begin{lemma}\label{lemma:psi_other_ttts}
    Under \TTTS, for every $\xi$, there exists $S_4 = \text{Poly}(1/\xi,W_1,W_2)$ s.t. $\forall n > S_4$,
    \[
        \forall i \in \cA, \ \ \frac{\Psi_{n,i}}{n} \leq \omega_{i}^\beta + \xi.
    \]
\end{lemma}

\begin{proof}
    From Lemma~\ref{lemma:over_allocation_finite_ttts}, there exists $S_1' = \text{Poly}(2/\xi,W_1,W_2)$ such that for all $n > S_1'$ and for all $i\neq I^\star$, 
    \[
        \frac{\Psi_{n,i}}{n} \geq \omega_{i}^\beta + \frac{\xi}{2}  \ \ \Rightarrow \ \ \psi_{n,i} \leq \expp{-\epsilon_0(\xi/2) n}.
    \] 
    Thus, for all $i \neq I^\star$,
    \begin{align*}
        \frac{\Psi_{n,i}}{n} 
        &\leq \frac{S_1'}{n} + \ddfrac{\sum_{\ell=S_1'+1}^n \psi_{\ell,i}\1{\left(\frac{\Psi_{\ell,i}}{n} \geq \omega_{i}^\beta+\frac{\xi}{2}\right)}}{n} + \ddfrac{\sum_{\ell=S_1'+1}^n \psi_{\ell,i}\1{\left(\frac{\Psi_{\ell,i}}{n} \leq \omega_{i}^\beta + \frac{\xi}{2}\right)}}{n} \\
        &\leq \frac{S_1'}{n} + \ddfrac{\sum_{\ell=1}^n \expp{-\epsilon_0(\xi/2) n}}{n} + \ddfrac{\sum_{\ell=S_1'+1}^{\ell_n(\xi)}\ \psi_{\ell,i}\1{\left(\frac{\Psi_{\ell,i}}{n} \leq \omega_{i}^\beta + \frac{\xi}{2}\right)}}{n},
    \end{align*}
    where we let $\ell_n(\xi) = \max\left\{ \ell \leq n : \Psi_{\ell,i}/n \leq \omega_{i}^\beta + \xi/2\right\}$. Then
    \begin{align*}
        \frac{\Psi_{n,i}}{n} 
        &\leq \frac{S_1'}{n} + \ddfrac{\sum_{\ell=1}^n \expp{-\epsilon_0(\xi/2) n}}{n} + \Psi_{\ell_n(\xi),i}\\
        &\leq \frac{S_1' + (1 - \exp(-\epsilon_0(\xi/2))^{-1}}{n}+ \omega_{i}^\beta + \frac{\xi}{2}
    \end{align*}
    Then, there exists $S_5$ such that for all $n \geq S_5$,
    \[
        \frac{S_1' + (1 - \exp(-\epsilon_0(\xi/2))^{-1}}{n} \leq \frac{\xi}{2}.
    \]
    Therefore, for any $n > S_4 \eqdef \max\{S_1',S_5\}$, $\Psi_{n,i} \leq \omega_i^\beta + \xi$ holds for all $i\neq I^\star$. For $i = I^\star$, it is already proved for the optimal arm.
\end{proof}

We now prove Lemma~\ref{lemma:tracking_other} under \TTTS.

\paragraph{Proof of Lemma~\ref{lemma:tracking_other}} 

From Lemma~\ref{lemma:psi_other_ttts}, there exists $S_4' = \text{Poly}((K-1)/\xi,W_1,W_2)$ such that for all $n > S_4'$,
\[
    \forall i\in\cA, \frac{\Psi_{n,i}}{n} \leq \omega_i^\beta + \frac{\xi}{K-1}.
\]
Using the fact that $\Psi_{n,i}/n$ and $\omega_{i}^\beta$ all sum to 1, we have $\forall i \in \cA$,
\begin{align*}
    \frac{\Psi_{n,i}}{n} &= 1 - \sum_{j\neq i} \frac{\Psi_{n,j}}{n}\\
                         &\geq 1 - \sum_{j\neq i} \left(\omega_j^\beta + \frac{\xi}{K-1}\right)\\
                         &= \omega_i^\beta - \xi.
\end{align*}

Thus, for all $n > S_4'$, we have
\[
    \forall i\in\cA, \left| \frac{\Psi_{n,i}}{n} - \omega_i^\beta \right| \leq \xi.
\]
And finally we use the same reasoning as the proof of Lemma~\ref{lemma:tracking_best} to link $T_{n,i}$ and $\Psi_{n,i}$. Fix an $\epsilon > 0$. Using Lemma~\ref{lemma:link}, we have $\forall n\in\NN$,
\[
    \forall i\in\cA, \left|\frac{T_{n,i}}{n}-\frac{\Psi_{n,i}}{n}\right| \leq \frac{W_2\sqrt{(n+1)\log(e^2+n)}}{n}.
\]
Thus there exists $S_5$ s.t. $\forall n > S_5$,
\[
    \left|\frac{T_{n,I^\star}}{n}-\frac{\Psi_{n,I^\star}}{n}\right| \leq \frac{\epsilon}{2}.
\]
And using the above result, there exists $S_4'' = \text{Poly}(2/\epsilon,W_1,W_2)$ s.t. $\forall n > S_4''$,
\[
    \left|\frac{\Psi_{n,i}}{n} - \omega_i^\beta\right| \leq \frac{\epsilon}{2}.
\]
Thus, if we take $N_4 \eqdef \max\{S_4'',S_5\}$, then $\forall n > N_4$, we have
\[
    \forall i\in\cA, \left|\frac{T_{n,i}}{n} - \omega_i^\beta\right| \leq \epsilon.
\]

\hfill\BlackBox\\[2mm]

\section{Fixed-Confidence Analysis for \TCC}\label{app:confidence_t3c}

This section is entirely dedicated to \TCC. Note that the analysis to follow share the same proof line with that of \TTTS, and some parts even completely coincide with those of \TTTS. For the sake of simplicity and clearness, we shall only focus on the parts that differ and skip some redundant proofs. 

\subsection{Sufficient exploration of all arms, proof of Lemma~\ref{lemma:sufficient_exploration} under \TCC}\label{app:confidence_t3c.exploration}

To prove this lemma, we still need the two sets of indices for under-sampled arms like in Appendix~\ref{app:confidence_ttts.exploration}. We recall that for a given $L>0$: $\forall n\in\NN$ we define
\[
    U_n^L \eqdef \{i: T_{n,i} < \sqrt{L}\},
\]
\[
    V_n^L \eqdef \{i: T_{n,i} < L^{3/4}\}.
\]
For \TCC however, we investigate the following two indices,
\[
    J_n^{(1)} \eqdef \argmax_{j} a_{n,j}, \tilde{J_n^{(2)}} \eqdef \argmin_{j\neq J_n^{(1)}} W_n(J_n^{(1)},j).
\]
Lemma~\ref{lemma:sufficient_exploration} is proved via the following sequence of lemmas.

\begin{lemma}\label{lemma:link_t3c}
    There exists $L_1 = \text{Poly}(W_1)$ s.t. if $L > L_1$, for all $n$, $U_n^L \neq \emptyset$ implies $J_n^{(1)} \in V_n^L$ or $\tilde{J_n^{(2)}} \in V_n^L$.
\end{lemma}

\begin{proof}
    If $J_n^{(1)} \in V_n^L$, then the proof is finished. Now we assume that $J_n^{(1)} \in \bar{V_n^L} \subset \bar{U_n^L}$, and we prove that $J_n^{(2)} \in V_n^L$.
    \paragraph{Step 1} Following the same reasoning as Step 1 and Step 2 of the proof of Lemma~\ref{lemma:link_ttps}, we know that there exists $L_2 = \text{Poly}(W_1)$ s.t. if $L>L_2$, then
    \[
        \bar{J_n^\star} \eqdef \argmax_{j\in\bar{U_n^L}} \mu_{n,j} = \argmax_{j\in\bar{U_n^L}} \mu_j = J_n^{(1)}.
    \]
    
    \paragraph{Step 2} Now assuming that $L>L_2$, and we show that for $L$ large enough, $\tilde{J_n^{(2)}} \in V_n^L$. In the same way that we proved~\eqref{eq:upper_bound_anj_explo} one can show that for all $\forall j \in \bar{V_n^L}$,
    \[
        W_n(J_n^{(1)},j) = \ddfrac{(\mu_{n,I^\star}-\mu_{n,j})^2}{2\sigma^2\left(\frac{1}{T_{n,I^\star}}+\frac{1}{T_{n,j}}\right)} \geq \frac{L^{3/4}\Delta_{\text{min}}^2}{16\sigma^2}.
    \]
    
    Again, denote $J_n^\star \eqdef \argmax_{j\in U_n^L} \mu_{n,j}$, we obtain
    \begin{equation*}
        W_n(J_n^{(1)},J_n^\star) = \begin{cases}
        0 &\text{if } \mu_{n,J_n^\star}\geq\mu_{n,J_n^{(1)}}, \\
        \ddfrac{(\mu_{n,J_n^{(1)}}-\mu_{n,J_n^\star})^2}{2\sigma^2\left(\frac{1}{T_{n,J_n^{(1)}}}+\frac{1}{T_{n,J_n^\star}}\right)} &\text{else}.
        \end{cases}
    \end{equation*}
    In the second case, as already shown in Step 3 of Lemma~\ref{lemma:link_ttps} we have that
    \begin{align*}
        |\mu_{n,J_n^\star} - \mu_{n,\bar{J_n^\star}}| 
        &\leq \Delta_{\max} + 2\sigma W_1 \sqrt{\frac{\log(e+T_{n,J_n^\star})}{1+T_{n,J_n^\star}}}\\
        &\leq \Delta_{\max} + 2\sigma W_1 \sqrt{\frac{\log(e+\sqrt{L})}{1+\sqrt{L}}},
    \end{align*}
    since $J_n^\star \in U_n^L$. We also know that
    \[
        2\sigma^2\left(\frac{1}{T_{n,J_n^{(1)}}}+\frac{1}{T_{n,J_n^\star}}\right) \geq \frac{2\sigma^2}{T_{n,J_n^\star}} \geq \frac{2\sigma^2}{\sqrt{L}}.
    \]
    Therefore, we get
    \[
        W_n(J_n^{(1)},J_n^\star) \leq \frac{\sqrt{L}}{2\sigma^2}\left(\Delta_{\max} + 2\sigma W_1 \sqrt{\frac{\log(e+\sqrt{L})}{1+\sqrt{L}}}\right)^2.
    \]
    On the other hand, we know that for all $j\in\bar{V_n^L}$,
    \[
        W_n(J_n^{(1)},j) \geq \frac{L^{3/4}\Delta_{\text{min}}^2}{16\sigma^2}.
    \]
    Thus, there exists $L_3$ s.t. if $L>L_3$, then
    \[
        \forall j\in \bar{V_n^L},\, W_n(J_n^{(1)},j) \geq 2W_n(J_n^{(1)},J_n^\star).
    \]
    
    That means $\tilde{J_n^{(2)}}\notin \bar{V_n^L}$ and by consequence, $\tilde{J_n^{(2)}}\in V_n^L$.
    
    Finally, taking $L_1 = \max(L_2, L_3)$, we have $\forall L > L_1$, either $J_n^{(1)} \in V_n^L$ or $\tilde{J_n^{(2)}} \in V_n^L$.
\end{proof}

Next we show that there exists at least one arm in $V_n^L$ for whom the probability of being pulled is large enough. More precisely, we prove the following lemma.

\begin{lemma}\label{lemma:psi_min_t3c}
    There exists $L_1 = \text{Poly}(W_1)$ s.t. for $L > L_1$ and for all $n$ s.t. $U_n^L \neq \emptyset$, then there exists $J_n \in V_n^L$ s.t.
    \[
        \psi_{n,J_n} \geq \frac{\min(\beta,1-\beta)}{K^2} \eqdef \psi_{\min}.
    \]
\end{lemma}

\begin{proof}
    Using Lemma~\ref{lemma:link_t3c}, we know that $J_n^{(1)}$ or $\tilde{J_n^{(2)}} \in V_n^L$. We also know that under \TCC, for any arm $i$, $\psi_{n,i}$ can be written as
    \[
        \psi_{n,i} = \beta a_{n,i} + (1-\beta) \sum_{j\neq i} a_{n,j}\frac{\1\{W_n(j,i)=\min_{k\neq j} W_n(j,k)\}}{\big|\argmin_{k\neq j } W_n(j,k)\big|}.
    \]
    Note that $(\psi_{n,i})_i$ sums to 1,
    \begin{align*}
        \sum_i  \psi_{n,i} &= \beta +(1-\beta) \sum_j a_{n,j} \sum_{i\neq j } \frac{\1\{W_n(j,i)=\min_{k\neq j} W_n(j,k)\}}{\big|\argmin_{k\neq j } W_n(j,k)\big|}\\
        &= \beta +(1-\beta) \sum_j a_{n,j} =1\,.
    \end{align*}
    Therefore, we have
    \[
        \psi_{n,J_n^{(1)}} \geq \beta a_{n,J_n^{(1)}} \geq \frac{\beta}{K}
    \]
    on one hand, since $\sum_{i\in\cA} a_{n,i} = 1$. On the other hand, we have
    \begin{align*}
        \psi_{n,\tilde{J_n^{(2)}}} &\geq (1-\beta) \frac{a_{n,J_n^{(1)}}}{K}\\
                           &\geq \frac{1-\beta}{K^2},
    \end{align*}
    which concludes the proof.
\end{proof}

The rest of this subsection is exactly the same to that of \TTTS. Indeed, with the above lemma, we can show that the set of poorly explored arms $U_n^L$ is empty when $n$ is large enough.

\begin{lemma}\label{lemma:poorly_explored_t3c}
    Under \TCC, there exists $L_0 = \text{Poly}(W_1,W_2)$ s.t. $\forall L > L_0$, $U_{\floor{KL}}^L = \emptyset$.
\end{lemma}

\begin{proof}
    See proof of Lemma~\ref{lemma:poorly_explored_ttts} in Appendix~\ref{app:confidence_ttts.exploration}.
\end{proof}

We can finally conclude the proof of Lemma~\ref{lemma:sufficient_exploration} for \TCC in the same way as for \TTTS in Appendix~\ref{app:confidence_ttts.exploration}.
\hfill\BlackBox\\[2mm]

\subsection{Concentration of the empirical means, proof of Lemma~\ref{lemma:tracking_means} under \TCC}\label{app:confidence_t3c.means}

As a corollary of the previous section, we can show the concentration of $\mu_{n,i}$ to $\mu_i$, and the proof remains the same as that of \TTTS in Appendix~\ref{app:confidence_ttts.means}.

\subsection{Measurement effort concentration of the optimal arm, proof of Lemma~\ref{lemma:tracking_best} under \TCC}\label{app:confidence_t3c.best_arm}

Next, we show that the empirical arm draws proportion of the true best arm for \TCC concentrates to $\beta$ when the total number of arm draws is sufficiently large. This proof also remains the same as that of \TTTS in Appendix~\ref{app:confidence_ttts.best_arm}.

\subsection{Measurement effort concentration of other arms, proof of Lemma~\ref{lemma:tracking_other} under \TCC}\label{app:confidence_t3c.other_arms}

In this section, we show that, for \TCC, the empirical measurement effort concentration also holds for other arms than the true best arm. Note that this part differs from that of \TTTS.

We again establish first an over-allocation implies negligible probability result as follow.

\begin{lemma}\label{lemma:over_allocation_finite_t3c} 
    Under \TCC, for every $\xi \leq \epsilon_0$ with $\epsilon_0$ problem dependent, there exists $S_1 = \text{Poly}(1/\xi,W_1,W_2)$ such that for all $n > S_1$, for all $i\neq I^\star$, 
    \[
        \frac{\Psi_{n,i}}{n} \geq \omega_{i}^\beta + 2\xi  \ \ \Rightarrow \ \ \psi_{n,i} \leq (K-1)\expp{-\frac{\Delta_{\text{min}}^2}{16\sigma^2}\sqrt{\frac{n}{K}}}\,.
    \]
\end{lemma}

\begin{proof}
    Fix $i\neq I^\star$ s.t. $\Psi_{n,i}/n\geq \omega_i^\beta+2\xi$, then using Lemma~\ref{lemma:link}, there exists $S_2=\text{Poly}(1/\xi,W_2)$ such that for any $n>S_2$, we have
    \[
        \frac{T_{n,i}}{n} \geq \omega_i^\beta + \xi.
    \]
    Then,
    \begin{align*}
        \psi_{n,i} &\leq \beta a_{n,i} + (1-\beta) \sum_{j\neq i} a_{n,j}\1\{W_n(j,i)=\min_{k\neq j} W_n(j,k)\}\\
                   &\leq \beta a_{n,i} + (1-\beta) \left(\sum_{j\neq i,I^\star} a_{n,j} + a_{n,I^\star}\1\{W_n(I^\star,i)=\min_{k\neq I^\star} W_n(I^\star,k)\}\right)\\
                   &\leq \sum_{j\neq I^\star} a_{n,j} + \1\{W_n(I^\star,i)=\min_{k\neq I^\star} W_n(I^\star,k)\}.
    \end{align*}
    Next we show that the indicator function term in the previous inequality equals to 0.
    
    Using Lemma~\ref{lemma:means} and Lemma~\ref{lemma:tracking_best} for \TCC, there exists $S_3 = \text{Poly}(1/\xi,W_1,W_2)$ such that for any $n>S_3$,
    \[
        \left|\frac{T_{n,I^\star}}{n} - \beta\right| \leq \xi^2 \text{ and } \forall j \in\cA, |\mu_{n,j}-\mu_j|\leq\xi^2.
    \]
    
    Now if $\forall j \neq I^\star,i$, we have $T_{n,j}/n>\omega_j^\beta$, then
    \begin{align*}
        \frac{n-1}{n} &= \sum_{j\in\cA}\frac{T_{n,j}}{n}\\
                      &= \frac{T_{n,I^\star}}{n} + \frac{T_{n,i}}{n} + \sum_{j\neq I^\star,i}\frac{T_{n,j}}{n}\\
                      &> \beta - \epsilon^2 + \omega_i^\beta + \epsilon + \sum_{j\neq I^\star,i} \omega_j^\beta \geq 1,
    \end{align*}
    which is a contradiction.
    
    Thus there exists at least one $j_0\neq I^\star,i$, such that $T_{n,j_0}/n \leq \omega_j^\beta$. Assuming $n>\max(S_2,S_3)$, we have
    \begin{align*}
        W_n(I^\star,i)-W_n(I^\star,j_0) 
        &= \ddfrac{(\mu_{n,I^\star}-\mu_{n,i})^2}{2\sigma^2\left(\frac{1}{T_{n,I^\star}}+\frac{1}{T_{n,i}}\right)} - \ddfrac{(\mu_{n,I^\star}-\mu_{n,j_0})^2}{2\sigma^2\left(\frac{1}{T_{n,I^\star}}+\frac{1}{T_{n,j_0}}\right)}\\
        &\geq \underbrace{\ddfrac{(\mu_{I^\star}-\mu_{i}-2\xi^2)^2}{2\sigma^2\left(\frac{1}{\beta-\xi^2}+\frac{1}{\omega_i^\beta+\xi}\right)} - \ddfrac{(\mu_{I^\star}-\mu_{j_0}+2\xi^2)^2}{2\sigma^2\left(\frac{1}{\beta+\xi^2}+\frac{1}{\omega_{j_0}^\beta}\right)}}_{W_{i,j_0}^\xi}.
    \end{align*}
    According to Proposition~\ref{prop:optim}, $W_{i,j_0}^\xi$ converges to 0 when $\xi$ goes to 0, more precisely we have 
    \[W_{i,j_0}^\xi = \frac{(\mu_{I^\star}-\mu_{i})^2}{2\sigma^2} \left(\frac{\beta}{\beta +\omega_i^\beta}\right)^2 \xi + O(\xi^2)\,,
    \]
    thus there exists a $\epsilon_0$ such that for all $\xi<\epsilon_0$ it holds for all $i,j_0\neq I^\star$, $W_{i, j_0}^\xi>0$. It follows then
    \[
        W_n(I^\star,i)-\min_{k\neq I^\star}W_n(I^\star,k) \geq W_n(I^\star,i)-W_n(I^\star,j_0) > 0,
    \]
    and $\1\{W_n(I^\star,i)=\min_{k\neq I^\star} W_n(I^\star,k)\}=0$.
    
    Knowing that Lemma~\ref{lemma:empirical_best} is also valid for \TCC, thus there exists $M_1 = \text{Poly}(4/\Delta_{\min},W_1,W_2)$ such that for all $n>M_1$,
    \[
        \forall j \neq I^\star, a_{n,j} \leq \expp{-\frac{\Delta_{\text{min}}^2}{16\sigma^2}\sqrt{\frac{n}{K}}},
    \]
    which then concludes the proof by taking $S_1\eqdef\max(M_1,S_2,S_3)$.
\end{proof}

The rest of this subsection almost coincides with that of \TTTS. We first show that, starting from some known moment, no arm is overly allocated. More precisely, we show the following lemma.

\begin{lemma}\label{lemma:psi_other_t3c}
    Under \TCC, for every $\xi$, there exists $S_4 = \text{Poly}(1/\xi,W_1,W_2)$ s.t. $\forall n > S_4$,
    \[
        \forall i \in \cA, \ \ \frac{\Psi_{n,i}}{n} \leq \omega_{i}^\beta + 2\xi.
    \]
\end{lemma}

\begin{proof}
    See proof of Lemma~\ref{lemma:psi_other_ttts} in Appendix~\ref{app:confidence_ttts.other_arms}. Note that the previous step does not match exactly that of \TTTS, so the proof would be slightly different. However, the difference is only a matter of constant, we thus still choose to skip this proof.
\end{proof}

It remains to prove Lemma~\ref{lemma:tracking_other} for \TCC, which stays the same as that of \TTTS.

\paragraph{Proof of Lemma~\ref{lemma:tracking_other} for \TCC} 

See proof of Lemma~\ref{lemma:tracking_other} for \TTTS in Appendix~\ref{app:confidence_ttts.other_arms}.

\hfill\BlackBox\\[2mm]

\section{Proof of Lemma~\ref{lemma:confidence}}\label{app:confidence}

Finally, it remains to prove Lemma~\ref{lemma:confidence} under the Gaussian case before we can conclude for Theorem~\ref{thm:confidence_main} for \TTTS or \TCC.

\restatefixedconfidence*

For the clarity, we recall the definition of generalized likelihood ratio. For any pair of arms $i, j$, We first define a weighted average of their empirical means,
\[
    \hat{\mu}_{n,i,j} \eqdef \frac{T_{n,i}}{T_{n,i}+T_{n,j}} \hat{\mu}_{n,i} + \frac{T_{n,j}}{T_{n,i}+T_{n,j}} \hat{\mu}_{n,j}.
\]
And if $\hat{\mu}_{n,i}\geq\hat{\mu}_{n,j}$, then the generalized likelihood ratio $Z_{n,i,j}$ for Gaussian noise distributions has the following analytic expression,
\[
    Z_{n,i,j} \eqdef T_{n,i}d(\hat{\mu}_{n,i};\hat{\mu}_{n,i,j}) + T_{n,j}d(\hat{\mu}_{n,j};\hat{\mu}_{n,i,j}).
\]
We further define a statistic $Z_n$ as
\[
    Z_n \eqdef \max_{i\in\cA}\min_{j\in\cA\backslash\left\{i\right\}} Z_{n,i,j}.
\]

The following lemma stated by~\citet{qin2017ttei} is needed in our proof.

\begin{lemma}\label{lemma:ttei}
    For any $\zeta > 0$, there exists $\epsilon$ s.t. $\forall n\geq T_{\beta}^\epsilon$, $Z_n\geq (\Gamma_{\beta}^\star-\zeta)n$.
\end{lemma}

To prove Lemma~\ref{lemma:confidence}, we need the Gaussian tail inequality (\ref{gaussian_upper}) of Lemma~\ref{lemma:gaussiantails}.

\begin{proof}
    We know that
    \begin{align*}
    \begin{split}
        1 - a_{n,I^\star} &= \sum_{i\neq I^\star} a_{n,i} \\
        &\leq \sum_{i\neq I^\star} \Pi_n\left[\theta_i > \theta_{I^\star}\right] \\
        &= \sum_{i\neq I^\star} \Pi_n\left[\theta_i - \theta_{I^\star} > 0 \right] \\
        &\leq (K-1)\max_{i\neq I^\star} \Pi_n\left[\theta_i - \theta_{I^\star} > 0 \right].
    \end{split}
    \end{align*}
    
    We can further rewrite $\Pi_n\left[\theta_i - \theta_{I^\star} > 0 \right]$ as
    \[
        \Pi_n\left[\theta_i - \theta_{I^\star} > \mu_{n,i} - \mu_{n,I^\star} + \mu_{n,I^\star} - \mu_{n,i} \right].
    \]
    We choose $\epsilon$ sufficiently small such that the empirical best arm $I_n^\star = I^\star$. Then, for all $n \geq T_{\beta}^n$ and for any $i\neq I^\star$, $\mu_{n,I^\star} \geq \mu_{n,i}$. Thus, fix any $\zeta\in (0,\Gamma_{\beta}^\star/2)$ and apply inequality (\ref{gaussian_upper}) of Lemma~\ref{lemma:gaussiantails} with $\mu_{n,I^\star}$ and $\mu_{n,i}$, we have for any $n \geq T_{\beta}^\epsilon$,
    \begin{align*}
    \begin{split}
        1 - a_{n,I^\star} &\leq (K-1)\max_{i\neq I^\star}\frac{1}{2} \expp{-\frac{\left(\mu_{n,I^\star}-\mu_{n,i} \right)^2}{2\sigma_{n,i,I^\star}^2}} \\
        &= \frac{(K-1)\expp{-Z_n}}{2} \\
        &\leq \frac{(K-1)\expp{-(\Gamma_{\beta}^\star-\zeta)n}}{2}.
    \end{split}
    \end{align*}
    
    The last inequality is deduced from Lemma~\ref{lemma:ttei}. By consequence,
    \[
        \forall n \geq T_{\beta}^\epsilon, \ln\left(1 - a_{n,I^\star}\right) \leq \ln{\frac{K-1}{2}} - (\Gamma_{\beta}^\star-\zeta)n.
    \]
    
    On the other hand, we have for any $n$,
    \[
        1 - c_{n,\delta} = \ddfrac{\delta}{2n(K-1)\sqrt{2\pi e}\expp{\sqrt{2\ln{\frac{2n(K-1)}{\delta}}}}}.
    \]
    Thus, there exists a deterministic time $N$ s.t. $\forall n\geq N$,
    \begin{align*}
    \begin{split}
        \ln\left(1 - c_{n,\delta}\right) &= \ln{\frac{\delta}{(K-1)\sqrt{8\pi e}}} - \ln{n} - \sqrt{2\ln{\frac{2n(K-1)}{\delta}}} \\
        &\geq \ln{\frac{\delta}{2(K-1)\sqrt{2\pi e}}} - \zeta n.
    \end{split}
    \end{align*}
    
    Let $C_3 \eqdef (K-1)^2\sqrt{2\pi e}$, we have for any $n \geq N_0 \eqdef T_{\beta}^\epsilon + N$,
    \begin{equation}\label{eq:intermediate}
        \ln\left(1 - a_{n,I^\star}\right) - \ln\left(1 - c_{n,\delta}\right) \leq \ln{\frac{C_3}{\delta}} - (\Gamma_{\beta}^\star-2\zeta)n,
    \end{equation}
    and it is clear that $\EE{N_0}<\infty$.
    
    Let us consider the following two cases:
    \paragraph{Case 1}
    There exists $n\in \left[1,N_0\right]$ s.t. $a_{n,I^\star} \geq c_{n,\delta}$, then by definition,
    \[
        \tau_{\delta} \leq n \leq N_1.
    \]
    \paragraph{Case 2}
    For any $n\in \left[1,N_0\right]$, we have $a_{n,I^\star} < c_{n,\delta}$, then $\tau_{\delta} \geq N_0+1$, thus by Equation~\ref{eq:intermediate},
    \begin{align*}
    \begin{split}
        0 &\leq \ln\left(1 - a_{\tau_{\delta}-1,I^\star}\right) - \ln\left(1 - c_{\tau_{\delta}-1,\delta}\right) \\
        &\leq \ln{\frac{C_3}{\delta}} - (\Gamma_{\beta}^\star-2\zeta)(\tau_{\delta}-1),
    \end{split}
    \end{align*}
    and we obtain
    \[
        \tau_{\delta} \leq \frac{\ln(C_3/\delta)}{\Gamma_{\beta}^\star-2\zeta}+1.
    \]
    
    Combining the two cases, and we have for any $\zeta\in (0,\Gamma_{\beta}^\star/2)$,
    \begin{align*}
    \begin{split}
        \tau_{\delta} &\leq \max\left\{ N_0,\frac{\ln(C_3/\delta)}{\Gamma_{\beta}^\star-2\zeta}+1 \right\} \\
        &\leq N_0+1+\frac{\ln(C_3)}{\Gamma_{\beta}^\star-2\zeta}+\frac{\ln(1/\delta)}{\Gamma_{\beta}^\star-2\zeta}.
    \end{split}
    \end{align*}
    
    Since $\EE{N_1}<\infty$, therefore
    \[
        \limsup_{\delta}\frac{\EE{\tau_{\delta}}}{\log(1/\delta)} \leq \frac{1}{\Gamma_{\beta}^\star-2\zeta}, \forall\zeta\in (0,\Gamma_{\beta}^\star/2),
    \]
    which concludes the proof.
\end{proof}

\section{Technical Lemmas}\label{app:lemmas}

The whole fixed-confidence analysis for the two sampling rules are both substantially based on two lemmas: Lemma 5 of~\cite{qin2017ttei} and Lemma~\ref{lemma:link}. We prove Lemma~\ref{lemma:link} in this section.

\restatewtwo*

\begin{proof}

The proof shares some similarities with that of Lemma 6 of~\cite{qin2017ttei}.
For any arm $i\in\cA$, define $\forall n\in\NN$,
\[
    D_n \eqdef T_{n,i} - \Psi_{n,i},
\]
\[
    d_n \eqdef \1{\{I_n=i\}} - \psi_{n,i}.
\]
It is clear that $D_n = \sum_{l=1}^{n-1} d_l$ and $\EE{d_n|\cF_{n-1}} = 0$. Indeed,
\begin{align*}
    \EE{d_n|\cF_{n-1}} &= \EE{\1{\{I_n=i\}} - \psi_{n,i}|\cF_{n-1}} \\
                       &= \PP{I_n=i|\cF_{n-1}} - \EE{\PP{I_n=i|\cF_{n-1}}|\cF_{n-1}} \\
                       &= \PP{I_n=i|\cF_{n-1}} - \PP{I_n=i|\cF_{n-1}} = 0.
\end{align*}
The second last equality holds since $\PP{I_n=i|\cF_{n-1}}$ is $\cF_{n-1}$-measurable. Thus $D_n$ is a martingale, whose increment are 1 sub-Gaussian as $d_n \in [-1,1]$ for all $n$. 

Applying Corollary 8 of~\cite{abbasi-yadkori2012}\footnote{but we could actually use several deviation inequalities that hold uniformly over time for martingales with sub-Gaussian increments}, it holds that, with probability larger than $1-\delta$, for all $n$,
\[
    |D_n| \leq \sqrt{2\left(1+n\right)\ln\left(\frac{\sqrt{1+n}}{\delta}\right)}
\]
which yields the first statement of Lemma~\ref{lemma:link}.

We now introduce the random variable
\[
    W_2 \eqdef \max_{n\in\NN}\max_{i\in\cA} \frac{|T_{n,i}-\Psi_{n,i}|}{\sqrt{(n+1)\ln(e^2+n)}}.
\]
Applying the previous inequality with $\delta = e^{-x^2 / 2}$ yields 
\begin{align*}
          \PP{\exists n \in \mathbb{N}^\star : |D_n | > \sqrt{\left(1+n\right)\left(\ln\left(1+n \right)+x^2\right)}} &\leq e^{-x^2 / 2}, \\
          \PP{\exists n \in \mathbb{N}^\star : |D_n | > \sqrt{\left(1+n\right)\ln\left(e^2+n \right)x^2}} &\leq e^{-x^2 / 2},
\end{align*}
where the last inequality uses that for all $a,b \geq 2$, we have $ab \geq a+b$. 

Consequently $\forall x\geq 2$, for all $i \in \cA$
\[
    \PP{\max_{n\in\NN}\frac{|T_{n,i}-\Psi_{n,i}|}{\sqrt{\left(n+1\right)\log\left({e^2+n}\right)}}\geq x} \leq e^{-x^2/2}.
\]
Now taking a union bound over $i\in\cA$, we have $\forall x\geq 2$,
\begin{align*}
    \PP{W_2\geq x} &\leq \PP{\max_{i\in\cA}\max_{n\in\NN}\frac{|T_{n,i}-\Psi_{n,i}|}{\left(n+1\right)\log\left(\sqrt{e^2+n}\right)}\geq x} \\ 
                 &\leq \PP{\bigcup_{i\in\cA}\max_{n\in\NN}\frac{|T_{n,i}-\Psi_{n,i}|}{\left(n+1\right)\log\left(\sqrt{e^2+n}\right)}\geq x} \\
                 &\leq \sum_{i\in\cA} \PP{\max_{n\in\NN}\frac{|T_{n,i}-\Psi_{n,i}|}{\left(n+1\right)\log\left(\sqrt{e^2+n}\right)}\geq x} \\
                 &\leq Ke^{-x^2/2}.
\end{align*}

The previous inequalities imply that $\forall i\in\cA$ and $\forall n\in\NN$, we have $|T_{n,i}-\Psi_{n,i}| \leq W_2\sqrt{(n+1)\log(e^2+n)}$ almost surely. Now it remains to show that $\forall \lambda > 0, \EE{e^{\lambda W_2}}<\infty$. Fix some $\lambda > 0$.
\begin{align*}
    \EE{e^{\lambda W}} &= \int_{x=1}^{\infty} \PP{e^{\lambda W}\geq x} \diff{x} = \int_{y=0}^{\infty} \PP{e^{\lambda W}\geq e^{2\lambda y}}2\lambda e^{2\lambda y} \diff{y} \\
                       &= 2\lambda \int_{y=0}^{2} \PP{W\geq 2y} e^{2\lambda y} \diff{y} + 2\lambda \int_{y=2}^{\infty} \PP{W\geq 2y} e^{2\lambda y} \diff{y} \\
                       &\leq \underbrace{2\lambda \int_{y=0}^{2} \PP{W\geq 2y} e^{2\lambda y} \diff{y}}_{=e^{4\lambda-1}} + \underbrace{2\lambda C_1 \int_{y=2}^{\infty} e^{-y^2/2} e^{2\lambda y} \diff{y}}_{<\infty} < \infty,
\end{align*}
where $C_1$ is some constant.


\end{proof}

\section{Proof of Posterior Convergence for the Gaussian Bandit}\label{app:posterior_gaussian}

\subsection{Proof of Theorem~\ref{thm:posterior_gaussian}}\label{app:posterior_gaussian.main}

\restateposteriorgaussian*

From Theorem 2 in \cite{qin2017ttei}, any allocation rule satisfying $T_{n, i} / n \rightarrow \omega_i^\beta$ for each $i \in \cA$, satisfies 
\begin{align*}
    \lim_{n \rightarrow \infty} - \frac{1}{n} \log(1 - a_{n,I^\star}) = \Gamma_{\beta}^\star.
\end{align*}
Therefore, to prove Theorem~\ref{thm:posterior_gaussian}, it is sufficient to prove that under \TTTS,
\begin{equation}
    \forall i \in \{1,\dots,K\}, \ \ \     \lim_{n\rightarrow\infty} \frac{T_{n,i}}{n}  \overset{a.s}{=} \omega_i^\beta\label{ToProveGaussian}.
\end{equation}
Due to the concentration result in Lemma~\ref{lemma:link} that we restate below (and proved in Appendix~\ref{app:confidence_ttts}), which will be useful at several places in the proof, observe that 
\[
    \lim_{n\rightarrow \infty} \frac{T_{n,i}}{n}  \overset{a.s}{=} \omega_i^\beta \ \ \Leftrightarrow \ \ \ \lim_{n\rightarrow \infty} \frac{\Psi_{n,i}}{n}  \overset{a.s}{=} \omega_i^\beta,
\]
therefore it suffices to establish the convergence of $\overline{\psi}_{n,i} = \Psi_{n,i}/n$ to $\omega_i^\beta$, which we do next. For that purpose, we need again the following maximality inequality lemma.

\restatewtwo*

\paragraph{Step 1: \TTTS draws all arms infinitely often and satisfies $T_{n,I^\star}/n \rightarrow \beta$.} More precisely, we prove the following lemma. 

\begin{lemma}\label{lemma:optimal_prop_istar_gaussian}
	Under \TTTS, it holds almost surely that
	\begin{enumerate}
	    \item for all $i \in \cA$, $\lim_{n\rightarrow \infty} T_{n,i} = \infty.$
	    \item $a_{n,I^\star} \rightarrow 1.$
	    \item $T_{n,I^\star}/n \rightarrow \beta$.
	\end{enumerate}
\end{lemma}

\begin{proof} Our first ingredient is a lemma showing the implications of finite measurement, and consistency when all arms are sampled infinitely often. Its proof follows standard posterior concentration arguments and is given in Appendix~\ref{app:posterior_gaussian.aux}.

\begin{lemma}[Consistency and implications of finite measurement]\label{lemma:consistency_gaussian}\ \\
	Denote with $\overline{\mathcal{I}}$ the arms that are sampled only a finite amount of times:
	\begin{align*}
	\overline{\mathcal{I}} = \{ i \in \{ 1, \ldots, k \} : \forall n, T_{n,i} < \infty \}.
	\end{align*}
	If $\overline{\mathcal{I}}$ is empty, $a_{n,i}$ converges almost surely to $1$ when $i = I^\star$ and to $0$ when $i \neq I^\star$. If $\overline{\mathcal{I}}$ is non-empty, then for every $i \in \overline{\mathcal{I}}$, we have $\liminf_{n \rightarrow \infty} a_{n,i} > 0$ a.s.
\end{lemma}

	First we show that $\sum_{n \in \mathbb{N}} T_{n,j} = \infty$ for each arm $j$. Suppose otherwise. Let $\overline{\mathcal{I}}$ again be the set of arms to which only finite measurement effort is allocated. Under \TTTS, we have
	\begin{align*}
	\psi_{n,i} = a_{n,i} \left( \beta + (1-\beta) \sum_{j \neq i} \frac{a_{n,j}}{1- a_{n,j}} \right),
	\end{align*}
	so $ \psi_{n,i}  \geq \beta a_{n,i}$. Therefore, by Lemma~\ref{lemma:consistency_gaussian}, if $i \in \overline{\mathcal{I}}$, then $\liminf a_{n,i} > 0$ implies that $\sum_n \psi_{n,i} = \infty$. By Lemma~\ref{lemma:link}, we then must have that $\lim_{n \rightarrow \infty} T_{n,i} = \infty$ as well: contradiction.  Thus, $\lim_{n \rightarrow \infty} T_{n,i} = \infty$ for all $i$, and we conclude that $a_{n,I^\star} \rightarrow 1$, by Lemma~\ref{lemma:consistency_gaussian}. 
	
	For \TTTS with parameter $\beta$ this implies that $\overline{\psi}_{n, I^\star} \rightarrow \beta$, and since we have a bound on $| T_{n,i} / n - \overline{\psi}_{n, i} |$ in Lemma~\ref{lemma:link}, we have $T_{n, I^\star} / n \rightarrow \beta$ as well.
\end{proof}

\paragraph{Step 2:  Controlling the over-allocation of sub-optimal arms.}
The convergence of $T_{n,I^\star}/n$ to $\beta$ leads to following interesting consequence, expressed in Lemma~\ref{lemma:over_allocation}: if an arm is sampled more often than its optimal proportion, the posterior probability of this arm to be optimal is reduced compared to that of other sub-optimal arms.

\begin{lemma}[Over-allocation implies negligible probability]\label{lemma:over_allocation}\footnote{analogue of Lemma 13 of \cite{russo2016ttts}}
Fix any $\xi > 0$ and $j \neq I^\star$. With probability 1, under any allocation rule, if $T_{n,I^\star}/n \rightarrow \beta$, there exist $\xi' > 0$ and a sequence $\epsilon_n$ with $\epsilon_n \rightarrow 0$ such that for any $n \in \mathbb{N}$, 
	\begin{align*}
	\frac{T_{n,j}}{n} \geq \omega_j^\beta + \xi \Rightarrow \frac{a_{n,j}}{\max_{i \neq I^\star} a_{n,i}} \leq e^{-n (\xi' + \epsilon_n)}.
	\end{align*}
\end{lemma}


\begin{proof}
	We have $\Pi_{n}(\Theta_{\cup i \neq I^\star}) = \sum_{i \neq I^\star} a_{n,i} = 1 - a_{n,I^\star}$, therefore $ \max_{i \neq I^\star} a_{n,i} \leq 1 - a_{n,I^\star}$. By Theorem~2 of \cite{qin2017ttei} we have, as $T_{n,I^\star}/n \rightarrow \beta$, 
	\begin{align*}
	\limsup_{n \rightarrow \infty} - \frac{1}{n} \log\left(\max_{i \neq I^\star} a_{n,i}\right) \leq \Gamma_{\beta}^\star.
	\end{align*}
	We also have the following from the standard Gaussian tail inequality, for $n \geq \tau$ after which $\mu_{n, I^\star} \geq \mu_{n, i}$, using that $\theta_i - \theta_{I^\star} \sim \mathcal{N}(\mu_{n,i} - \mu_{n, I^\star} , \sigma^2_{n,i} + \sigma^2_{n, I^\star} )$ and $\sigma^2_{n,i} + \sigma^2_{n, I^\star} = \sigma^2 (1/ T_{n,i} + 1/T_{n,I^\star})$,
	\begin{align*}
	a_{n,i} \leq \Pi_{n}(\theta_i \geq \theta_{I^\star})  \leq \exp \left( \frac{- (\mu_{n,i} - \mu_{n,I^\star})^2}{2\sigma^2 (1/T_{n,I^\star} +1/T_{n,i})} \right)
	= \exp \left(-n \, \frac{(\mu_{n,i} - \mu_{n,1})^2}{2\sigma^2 (n/T_{n,I^\star} +n/T_{n,i})} \right).
	\end{align*}
	
	Thus, there exists a sequence $\epsilon_n \rightarrow 0$, for which 
	\begin{align*}
	    \frac{a_{n,j}}{\max_{i \neq I^\star} a_{n,i}} 
	    \leq \ddfrac{\expp{-n \left( \frac{ (\mu_{n,j} - \mu_{n,I^\star})^2}{2\sigma^2 (n/T_{n,I^\star} +n/T_{n,j})} - \epsilon_n/2 \right) }}{\expp{-n \left( \Gamma_{\beta}^\star + \epsilon_n/2 \right)})}
	    = \expp{-n \left(\frac{(\mu_{n,j} - \mu_{n,I^\star})^2}{2\sigma^2 (n/T_{n,I^\star} +n/T_{n,j})} - \Gamma_{\beta}^\star  - \epsilon_n \right)}.
	\end{align*}
	Now we take a look at the two terms in the middle:
	\begin{align*}
	\frac{(\mu_{n,j} - \mu_{n,I^\star})^2}{2\sigma^2 (n/T_{n,I^\star} +n/T_{n,j})} - \Gamma_{\beta}^\star.
	\end{align*}
	Note that the first term is increasing in $T_{n,j} / n$. We have the definition from \cite{qin2017ttei}, for any $j \neq I^\star$,
	\begin{align*}
	\Gamma_{\beta}^\star = \frac{(\mu_{j} - \mu_{I^\star})^2}{2\sigma^2 \left(1/ \omega_{I^\star}^\beta +1/\omega_j^\beta\right)}, 
	\end{align*}
	and we have the premise
	\begin{align*}
	\frac{T_{n,j}}{n} \geq \omega_j^\beta + \xi.
	\end{align*}
	Combining these with the convergence of the empirical means to the true means (consistency, see Lemma~\ref{lemma:consistency_gaussian}), we can conclude that for all $\epsilon > 0$, there exists a time $n_0$ such that for all later times $n \geq n_0$, we have
	\begin{align*}
	\frac{(\mu_{n,j} - \mu_{n,I^\star})^2}{2\sigma^2 (n/T_{n,I^\star} +n/T_{n,j})} \geq   \frac{(\mu_{j} - \mu_{I^\star})^2}{2\sigma^2 \left(1/\beta +n/T_{n,j} \right)} - \epsilon 
	\geq  \frac{(\mu_{j} - \mu_{I^\star})^2}{2\sigma^2 \left(1/\beta +1/(\omega_j^\beta + \xi)\right)} - \epsilon
	> \Gamma_{\beta}^\star,
	\end{align*}
	where the first inequality follows from consistency, the second from monotonicity in $T_{n,j} / n$. That means that there exist a $\xi' > 0$ such that
	\begin{align*}
	\frac{(\mu_{n,j} - \mu_{n,I^\star})^2}{2\sigma^2 (n/T_{n,I^\star} +n/T_{n,j})} - \Gamma_{\beta}^\star > \xi',
	\end{align*}
	and thus the claim follows that when $\frac{T_{n,j}}{n} \geq \omega_j^\beta + \xi$, we have
	\begin{align*}
	\frac{a_{n,j}}{\max_{i \neq I^\star} a_{n,i}} \leq \exp \left\lbrace - n \left(\frac{(\mu_{n,j} - \mu_{n,I^\star})^2}{2\sigma^2 (n/T_{n,I^\star} +n/T_{n,j})} - \Gamma_{\beta}^\star  - \epsilon_n \right) \right\rbrace 
	\leq e^{-n (\xi' + \epsilon_n)}.
	\end{align*}
\end{proof}

\paragraph{Step 3: $\overline{\psi}_{n,i}$ converges to $\omega_i^\beta$ for all arms.} To establish the convergence of the allocation effort of all arms, we rely on the same sufficient condition used in the analysis of \cite{russo2016ttts}, that we recall below. 

\begin{lemma}[Sufficient condition for optimality]\label{lemma:sufficient_optimality}\footnote{Lemma 12 of \cite{russo2016ttts}}
Consider any adaptive allocation rule. If we have
	\begin{align}\label{eq:sufficient condition for optimality}
		&\overline{\psi}_{n, I^\star} \rightarrow \beta, \quad \text{ and } \quad
		\sum_{n \in \mathbb{N}} \psi_{n,j} \bm{1} \left\lbrace \overline{\psi}_{n,j} \geq \omega_j^\beta + \xi \right\rbrace < \infty, \quad \forall j \neq I^\star, \xi > 0,
	\end{align}
then $\overline{\psi}_{n} \rightarrow \psi^\beta$.
\end{lemma}

First, note that from Lemma~\ref{lemma:optimal_prop_istar_gaussian} we know that $T_{n,I^\star}/n \rightarrow \beta$, an by Lemma~\ref{lemma:link} this implies $\overline{\psi}_{n, I^\star} \rightarrow \beta$, hence we can use Lemma~\ref{lemma:sufficient_optimality} to prove convergence to the optimal proportions. Thus, we now show that \eqref{eq:sufficient condition for optimality} holds under \TTTS. Recall that $J_n^{(1)} = \argmax_{j} a_{n,j}$ and $J_n^{(2)} = \argmax_{j\neq J_n^{(1)}} a_{n,j}$. Since $a_{n,I^\star} \rightarrow 1$ by Lemma~\ref{lemma:optimal_prop_istar_gaussian}, there is some finite time $\tau$ after which for all $n > \tau$, $J_n^{(1)} = I^\star$. Under \TTTS, 
	\begin{align*}
	\psi_{n,i} =a_{n,i} \left( \beta  + (1-\beta) \sum_{j \neq i} \frac{a_{n,j}}{1-a_{n,j}} \right)
	           \leq a_{n,i} \beta  + a_{n,i} (1-\beta) \frac{\sum_{j \neq i} a_{n,j}}{1-a_{n,J_n^{(1)}}}
	           \leq  a_{n,i} \beta + a_{n,i} (1-\beta) \frac{\sum_{j \neq i} a_{n,j}}{a_{n,J_n^{(2)}}}\\
	           \leq a_{n,i}\beta  + a_{n,i} (1-\beta) \frac{1}{a_{n,J_n^{(2)}}}
	           \leq \frac{a_{n,i}}{a_{n,J_n^{(2)}}},
	\end{align*}
where we use the fact that for $j \neq J_n^{(1)}$, we have $a_{n,J_n^{(1)}} \geq a_{n,j}$ and $a_{n,J_n^{(2)}} \leq 1- a_{n,J_n^{(1)}}$. For $n \geq \tau$ this means that $\psi_{n, i} \leq a_{n,i} / \max_{j \neq I^\star} a_{n,i}$ for any $i \neq I^\star$.
	
By Lemma~\ref{lemma:over_allocation}, there is a constant $\xi' > 0$ such and a sequence $\epsilon_n \rightarrow 0$ such that
	\begin{align*}
	T_{n, i} / n \geq w_{i}^\beta + \xi \Rightarrow \frac{a_{n,i}}{\max_{j \neq I^\star} a_{n,j}} \leq e^{-n(\xi' - \epsilon_n)}.
	\end{align*}
Now take a time $\tau$ large enough, such that for $n \geq \tau$ we have $| T_{n,j} / n - \overline{\psi}_{n,j} | \leq \xi$ (which can be found by Lemma~\ref{lemma:link}). Then we have
	\begin{align*}
	    \1{\left\lbrace \overline{\psi}_{n,j} \geq \psi_j^{\beta} + \xi \right\rbrace } \leq \1{\left\lbrace \frac{T_{n,j}}{n} \geq \omega_j^\beta + 2\xi \right\rbrace }
	\end{align*}
Therefore, for all $i \neq I^\star$, we have
	\begin{align*}
	    \sum_{n \geq \tau} \psi_{n, i} 	\1{\left\lbrace \overline{\psi}_{n,j} \geq \psi_j^{\beta} + \xi \right\rbrace }
	    \leq \sum_{n \geq \tau} \psi_{n, i} \1{\left\lbrace \frac{T_{n,j}}{n} \geq \omega_j^\beta + 2\xi \right\rbrace }
	    \leq \sum_{n \geq \tau} e^{-n(\xi' - \epsilon_n)} < \infty.
	\end{align*}
Thus \eqref{eq:sufficient condition for optimality} holds and the convergence to the optimal proportions follows by Lemma~\ref{lemma:sufficient_optimality}.

\subsection{Proof of auxiliary lemmas}\label{app:posterior_gaussian.aux}

\paragraph{Proof of Lemma~\ref{lemma:consistency_gaussian}}
	
Let  $\overline{\mathcal{I}}$ be nonempty. Define
	\begin{align*}
	\mu_{\infty, n} \triangleq \lim_{n \rightarrow \infty} \mu_{n, i}, &\text{     and    } \sigma_{\infty, i}^2  \triangleq \lim_{n \rightarrow \infty} \sigma^2_{n,i},
	\end{align*}
and recall that for $i \in \cA$ for which $T_{n,i} = 0$, we have $\mu_{n_i} = \mu_{1,i} = 0$ and $\sigma^2_{n,i} = \sigma^2_{1, i} = \infty$, and if $T_{n,i} > 0$, we have
	\begin{align*}
	\mu_{n,i} = \frac{1}{T_{n,i}} \sum_{\ell = 1}^{n-1} \1{\{ I_\ell = i \}} Y_{\ell, I_\ell}, &\text{     and    } \sigma^2_{n,i} = \frac{\sigma^2}{T_{n,i}}.
	\end{align*}
For all arms that are sampled infinitely often, we therefore have $\mu_{\infty, i} = \mu_i$ and $\sigma_{\infty, i}^2 = 0$. For all arms that are sampled only a finite number of times, i.e.\ $i \in \overline{\cI}$, we have $\sigma_{\infty, i}^2 > 0$, and there exists a time $n_0$ after which for all $n \geq n_0$ and $i \in \overline{\cI}$, we have $T_{n,i} = T_{n_0,i}$. Define
	\begin{align*}
	\Pi_\infty \triangleq \cN(\mu_{\infty, 1}, \sigma_{\infty, 1}^2) \otimes \cN(\mu_{\infty, 2}, \sigma_{\infty, 2}^2) \otimes \ldots \otimes \cN(\mu_{\infty, k}, \sigma_{\infty, k}^2)
	= \bigotimes_{i \not\in \overline{\cI}} \delta_{\mu_i} \otimes \bigotimes_{i \in \overline{\cI}} \Pi_{n_0}.
	\end{align*}
Then for each $i \in \cA$ we define
	\begin{align*}
	a_{\infty, i} \triangleq \Pi_\infty \left( \theta_i > \max_{j \neq i} \theta_j \right). 
	\end{align*}
Then we have for all $i \in \overline{\mathcal{I}}$, $a_{\infty, i} \in (0,1)$, since $\sigma_{\infty, i}^2 > 0$, and thus $a_{\infty, I^\star} < 1$. 

When $\overline{\cI}$ is empty, we have $ a_{n,I^\star} = \Pi_{n} (\theta_{I^\star} > \max_{i \neq I^\star} \theta_i) $, but since $\Pi_\infty = \bigotimes_{i \in \cA} \delta_{\mu_i}$, we have $a_{\infty, I^\star} = 1$ and $a_{\infty, i} = 0$ for all $i \neq I^\star$.
 	
\hfill\BlackBox\\[2mm]

\section{Proof of Posterior Convergence for the Bernoulli Bandit}\label{app:posterior_beta}

\subsection{Preliminaries}\label{app:posterior_beta.pre}

We first introduce a crucial Beta tail bound inequality. Let $F^{\text{Beta}}_{a,b}$ denote the cdf of a Beta distribution with parameters $a$ and $b$, and $F^{\text{B}}_{c,d}$ the cdf of a Binomial distribution with parameters $c$ and $d$, then we have the following relationship, often called the `Beta-Binomial trick',
\begin{align*}
F^{\text{Beta}}_{a,b}(y) = 1 - F^{\text{B}}_{a+b-1, y} (a-1), 
\end{align*}
so that we have
\begin{align*}
\PP{X \geq x} = \PP{B_{a+b-1,x}  \leq a-1 } = \PP{B_{a+b-1,1-x} \geq b}.
\end{align*}

We can bound Binomial tails with Sanov's inequality:
\begin{align*}
    \frac{ e^{-n d \left( k / n, x \right)}  }{n+1} \leq \PP{B_{n,x} \geq k} \leq e^{-n d \left( k / n, x \right)},
\end{align*}
where the last inequalities hold when $k \geq nx$.

\begin{lemma}\label{lemma:binomial_tail}
Let $X \sim \cB eta(a, b)$ and $Y\sim \cB eta(c, d)$ with $0 < \frac{a-1}{a+b-1} < \frac{c-1}{c+d-1}$. Then we have $\PP{X > Y} \leq D e^{-C}$ where
\[
    C = \inf_{\frac{a-1}{a+b-1} \leq y \leq \frac{c-1}{c+d-1}} C_{a,b}(y)+C_{c,d}(y),
\]
and
\[
    D = 3 + \min \left( C_{a,b}\left(\frac{c-1}{c+d-1}\right), C_{c,d}\left(\frac{a-1}{a+b-1}\right) \right)\,.
\]
\end{lemma}
Note that this lemma is the Bernoulli version of Lemma~\ref{lemma:gaussiantails}.
\begin{restatable}{theorem}{restatebernoullilowerbound}\label{thm:bernoulli_lower_bound}
	Consider the Beta-Bernoulli setting. For $\beta \in (0,1)$, under any allocation rule satisfying 
	$T_{n, I^\star} / n \rightarrow \omega_{I^\star}^\beta$,
	\begin{align*}
	\lim_{n \rightarrow \infty} - \frac{1}{n} \log(1 - a_{n,I^\star}) \leq \Gamma_{\beta}^\star,
	\end{align*}
	and under any allocation rule satisfying $T_{n, i} / n \rightarrow \omega_i^\beta$ for each $i \in \cA$,
	\begin{align*}
		\lim_{n \rightarrow \infty} - \frac{1}{n} \log(1 - a_{n,I^\star}) = \Gamma_{\beta}^\star.
	\end{align*}
\end{restatable}

\begin{proof}
	Denote again with $\overline{\cI}$ again the set of arms sampled only finitely many times. For $\overline{\cI}$ empty, we thus have $\mu_{\infty, i} \triangleq \lim_{n \rightarrow \infty} \mu_{n,i} = \mu_i$. The posterior variance is
	\begin{align*}
	\sigma_{n,i}^2 &= \frac{\alpha_{n,i}\beta_{n,i}}{(\alpha_{n,i}+ \beta_{n,i})^2 (\alpha_{n,i} + \beta_{n,i} + 1)}
	= \frac{(1 + \sum_{\ell=1}^{n-1} \1{\{ I_{\ell} = i \}} Y_{\ell,I_{\ell}}) (1 + T_{n,i} - \sum_{\ell=1}^{n-1} \1{\{ I_{\ell} = i \}} Y_{\ell,I_{\ell}})  }{ (2 + T_{n,i})^2 (2 + T_{n,i} +1)}.
	\end{align*}
	We see that when $\overline{\cI}$ is empty, we have $\sigma_{\infty, i}^2 \triangleq \lim_{n \rightarrow \infty} \sigma_{n,i}^2 = 0$, i.e., the posterior is concentrated. 
	
\paragraph*{Step 1: A lower bound when some arms are sampled only finitely often.}
First, note that when $T_{n,i} = 0$ for some $i \in \cA$, the empirical mean for that arm equals the prior mean $\mu_{n,i} = \alpha_{0,i} / (\alpha_{0,i} + \beta_{0,i})$, and the variance is strictly positive: $\sigma^2_{n,i} = (\alpha_{0,i}\beta_{0,i}) / \left( (\alpha_{0,i}+ \beta_{0,i})^2 (\alpha_{0,i} + \beta_{0,i} + 1)\right) > 0$. When $\overline{\cI}$ is not empty, then for every $i \in \overline{\cI}$ we have $\sigma_{\infty, i}^2 > 0$, and $a_{\infty, i} \in (0,1)$, implying $a_{\infty, I^\star} < 1$, and thus
\begin{align*}
\lim_{n \rightarrow \infty} - \frac{1}{n} \log \left( 1- a_{n,I^\star} \right) = - \frac{1}{n} \log \left( 1- a_{\infty,I^\star} \right) = 0.
\end{align*}

\paragraph*{Step 2: A lower bound when every arm is sampled infinitely often.}
	Suppose now that $\overline{\cI}$ is empty, then we have
	\begin{align*}
	\max_{i \neq I^\star} \Pi_{n}  (\theta_i \geq \theta_{I^\star} ) \leq 1 - a_{n,I^\star} \leq \sum_{i \neq I^\star} \Pi_{n}(\theta_i \geq \theta_{I^\star}) \leq (k-1) \max_{i \neq I^\star} \Pi_{n}(\theta_i \geq \theta_{I^\star}).
	\end{align*}
	Thus, we have $1 - a_{n,I^\star} \leq (k-1) \max_{i \neq I^\star} \Pi_{n}(\theta_i \geq \theta_{I^\star})$ and also $1 - a_{n,I^\star} \doteq \max_{i \neq I^\star} \Pi_{n}(\theta_i \geq \theta_{I^\star})$.	We have 
	\begin{align*}
	\Gamma^\star &= \max_{w \in W} \min_{i \neq I^\star} C_i(\omega_{I^\star}, \omega_i),\\
	\Gamma_{\beta}^\star &= \max_{w \in W;\omega_{I^\star} = \beta} \min_{i \neq I^\star} C_i(\beta, \omega_i), \text{   with}\\
	C_i(\omega_{I^\star}, \omega_i) &= \min_{x \in \mathbb{R}} \omega_{I^\star} d(\theta_{I^\star} ; x) + \omega_i d(\theta_{i} ; x)
	= \omega_{I^\star} d(\theta_{I^\star} ; \overline{\theta} ) + \omega_i d(\theta_{i} ; \overline{\theta} ),
	\end{align*}
	where $\overline{\theta} \in [ \theta_i, \theta_{I^\star}]$ is the solution to 
	\begin{align*}
	A'(\overline{\theta}) = \frac{\omega_{I^\star} A'(\theta_{I^\star}) + \omega_i A'(\theta_i)}{\omega_{I^\star} + \omega_i}.
	\end{align*}
Since every arm is sampled infinitely often, when $n$ is large, we have $\mu_{n,I^\star} > \mu_{n,i}$. Define $S_{n,i} \triangleq \sum_{\ell=1}^{n-1} \1{\{ I_{\ell} = i \}} Y_{\ell,I_{\ell}}$. Recall that the posterior is a Beta distribution with parameters $a_{n,i} = S_{n,i} +1$ and $\beta_{n,i} = T_{n,i} - S_{n,i} + 1$. Let $\tau \in \mathbb{N}$ be such that for every $n \geq \tau$, we have $S_{n, i} / (T_{n,i} + 1) < S_{n, I^\star} / (T_{n,I^\star} + 1)$. For the sake of simplicity, we define for any $i\in\cA$ the interval
\[
    I_{i,I^\star}\eqdef \left[ \frac{S_{n,i}}{T_{n,i} + 1},  \frac{S_{n, I^\star}}{T_{n,I^\star} + 1} \right].
\] 
Then using Lemma~\ref{lemma:binomial_tail} with $a=S_{n,i}+1, b=T_{n,i}-S_{n,i}+1, c=S_{n,I^\star}+1, d=T_{n,I^\star}-S_{n,I^\star}+1$, we have
	\begin{align*}
	\Pi_{n}(\theta_i - \theta_{I^\star} \geq 0) &\leq D \expp{- \inf_{y \in I_{i,I^\star} } C_{S_{n,i}+1,T_{n,i}-S_{n,i}+1}(y)+C_{S_{n,I^\star}+1,T_{n,I^\star}-S_{n,I^\star}+1}(y)}.
	\end{align*}
This implies
	\begin{align*}
	\frac{1}{n} \log \left( \frac{\Pi_{n}(\theta_i \geq \theta_{I^\star})}{\expp{ - \inf_{y \in I_{i,I^\star} } C_{S_{n,i}+1,T_{n,i}-S_{n,i}+1}(y)+C_{S_{n,I^\star}+1,T_{n,I^\star}-S_{n,I^\star}+1}(y) }} \right) \leq \frac{1}{n} \log(D),
	\end{align*}
	which goes to zero as $n$ goes to infinity. Indeed replacing $a,b,c,d$ by their values in the definition of $D$ we get
	\begin{align*}
	    D&\leq 3 + (T_{n,i}-1) kl \left( \frac{S_{n,i}}{T_{n,i}+1};   \frac{S_{n,I^\star}}{T_{n,I^\star}+1}\right)\\
	    &\leq 3 + (n+1) kl\left(0; \frac{n}{n+1} \right) = (n+1)\log(n+1)\,. 
	\end{align*}
Hence,
\[
\Pi_{n}(\theta_i \geq \theta_{I^\star}) \doteq \expp{ - \inf_{y \in I_{i,I^\star}}  C_{S_{n,i}+1,T_{n,i}-S_{n,i}+1}(y)+C_{S_{n,I^\star}+1,T_{n,I^\star}-S_{n,I^\star}+1}(y) }.
\]
We thus have for any $i$,
\begin{align*}
1 - a_{n,i} &\doteq \max_{j \neq I^\star} \Pi_{n}\left[\theta_j \geq \theta_{I^\star}\right] \\
&\doteq \max_{j \neq I^\star} \expp{ - \inf_{y \in I_{j,I^\star} } C_{S_{n,j}+1,T_{n,j}-S_{n,j}+1}(y)+C_{S_{n,I^\star}+1,T_{n,I^\star}-S_{n,I^\star}+1}(y)}\\
&\doteq  \expp{ - n \min_{j \neq I^\star} \inf_{y \in I_{j,I^\star} } \frac{T_{n,j} +1}{n} kl\left(\frac{S_{n,j}}{T_{n,j}+1}; y\right) + \frac{T_{n,I^\star} +1}{n} kl\left(\frac{S_{n,I^\star}}{T_{n,I^\star}+1} ; y\right) }\\
&\geq \expp{ - n \max_{\bomega} \min_{j \neq I^\star} \inf_{y \in I_{j,I^\star} } \omega_i kl\left (\frac{S_{n,j}}{T_{n,j}+1}; y\right) + \omega_{I^\star} kl \left(\frac{S_{n,I^\star}}{T_{n,j}+1} ; y\right) }.
\end{align*}
Fix some $\epsilon > 0$, then there exists some $n_0(\epsilon)$ such that for all $n \geq n_0(\epsilon)$, we have for any $j$, 
\begin{align*}
I_{j,I^\star} = \left[\frac{S_{n,j}}{T_{n,j} + 1}, \frac{S_{n,I^\star}}{T_{n,I^\star} + 1}, \right] \subset \left[ \mu_j + \epsilon, \mu_{I^\star} - \epsilon \right] \triangleq I^\star_{j,\epsilon},  
\end{align*}
and because \texttt{KL}-divergence is uniformly continuous on the compact interval $I^\star_{j,\epsilon}$, there exists an $n_1$ such that for every $n \geq n_1$ we have 
\begin{align*}
kl\left(\frac{S_{n,j}}{T_{n,j} + 1}; y  \right) \geq (1-\epsilon) kl \left( \mu_j ; y \right),
\end{align*}
for any $y$ and for all $j \in \cA$. Therefore, we have
\begin{align*}
1 - a_{n,i} &\doteq \expp{ - n \max_{\bomega} \min_{j \neq I^\star} \inf_{y \in I_{j,I^\star} } \omega_j kl\left(\frac{S_{n,j}}{T_{n,j}+1}; y\right) + \omega_{I^\star} kl\left(\frac{S_{n,I^\star}}{T_{n,I^\star}+1} ; y\right) }\\
&\geq \expp{ - n \max_{\bomega} \min_{i \neq I^\star} \inf_{y \in I^\star_{j,\epsilon}} \omega_i kl (\mu_j; y) + \omega_{I^\star} kl (\mu_{I^\star} ; y) }.
\end{align*}
Therefore, we have
\begin{align*}
\limsup_{n \rightarrow \infty} - \frac{1}{n} \log (1 - a_{n,i}) \leq \Gamma^\star.
\end{align*}
If $T_{n,i} / n \rightarrow \omega_i^\star$ for each $i \in \cA$, we have
\begin{align*}
    &\lim_{n \rightarrow \infty} \inf_{y \in I_{i,I^\star} } \frac{T_{n,i} + 1}{n} kl\left(\frac{S_{n,i}}{T_{n,i}+1}; y\right)+ \frac{T_{n,I^\star}  + 1}{n} kl\left(\frac{S_{n,I^\star}}{T_{n,i}+1} ; y\right)\\
    &= \inf_{y \in \left[ \mu_i, \, \mu_{I^\star} \right] } \omega_i^\star kl (\mu_{i}; y) + \omega_{I^\star}^\star kl (\mu_{I^\star} ; y) \\
    &= \Gamma^\star,
\end{align*}
and thus
\begin{align*}
1 - a_{n,i} &\doteq \expp{ - n \max_{\bomega} \min_{j \neq I^\star} \inf_{y \in I^\star_\epsilon} \omega_i kl (\mu_j; y) + \omega_{I^\star} kl (\mu_{I^\star} ; y) } \\
&\doteq \expp{ - n \Gamma^\star },
\end{align*}
implying
\begin{align*}
\lim_{n \rightarrow \infty} - \frac{1}{n} \log \left( 1 - a_{n,i} \right) = \Gamma^\star.
\end{align*}

Everything goes similarly when $\omega_{I^\star} = \beta \in (0,1)$, so under any sampling rule satisfying $T_{n,I^\star} / n \rightarrow \beta$ we have
\begin{align*}
\limsup_{n \rightarrow \infty} - \frac{1}{n} \log (1 - a_{n,i}) \leq \Gamma_{\beta}^\star
\end{align*}
and under any sampling rule satisfying $T_{n,i} / n \rightarrow \omega_i^\beta$ for each $i \in \cA$, we have
\begin{align*}
\lim_{n \rightarrow \infty} - \frac{1}{n} \log (1 - a_{n,i}) = \Gamma_{\beta}^\star.
\end{align*}

\end{proof}

\subsection{Proof of Theorem~\ref{thm:posterior_bernoulli}}\label{app:posterior_beta.main}

\restateposteriorbernoulli*

From Theorem~\ref{thm:bernoulli_lower_bound} we know that under any allocation rule satisfying $T_{n,i} / n \rightarrow \omega_i^\beta$ for every $i \in \cA$, we have
	\begin{align*}
		\lim_{n \rightarrow \infty} - \frac{1}{n} \log \left( 1 - a_{n,I^\star} \right) = \Gamma_{\beta}^\star.
	\end{align*}
Thus, we only need to prove that under \TTTS, for all $i \in \cA$, we have
	\begin{align*}
		\lim_{n \rightarrow \infty} \frac{T_{n,i}}{n} \overset{a.s}{=} \omega_i^\beta.
	\end{align*}
Just as for the proof of the Gaussian case, we can use Lemma~\ref{lemma:link} (proof in Appendix~\ref{app:posterior_gaussian.aux}), which implies 
\[
    \lim_{n\rightarrow \infty} \frac{T_{n,i}}{n}  \overset{a.s}{=} \omega_i^\beta \ \ \Leftrightarrow \ \ \ \lim_{n\rightarrow \infty} \frac{\Psi_{n,i}}{n}  \overset{a.s}{=} \omega_i^\beta.
\]
Therefore, it suffices to show convergence for $\overline{\psi}_{n,i} = \Psi_{n,i}/n$ to $\omega_i^\beta$, which we will do next, following the same steps as in the proof for the Gaussian case. 
	
\paragraph{Step 1: \TTTS draws all arms infinitely often and satisfies $T_{n,I^\star}/n \rightarrow \beta$.} We prove the following lemma. 
	
\begin{lemma}\label{lemma:optimal_prop_istar_bernoulli}
	Under \TTTS, it holds almost surely that
	\begin{enumerate}
		\item for all $i \in \cA$, $\lim_{n\rightarrow \infty} T_{n,i} = \infty.$
		\item $a_{n,I^\star} \rightarrow 1.$
		\item $\frac{T_{n,I^\star}}{n} \rightarrow \beta$.
	\end{enumerate}
\end{lemma}

\begin{proof}
First, we give a lemma showing the implications of finite measurement, and consistency when all arms are sampled infinitely often, which provides a proof for $2.$ The proof of this lemma follows from the proof of Theorem~\ref{thm:bernoulli_lower_bound}, and is given in Appendix~\ref{app:posterior_beta.aux}. 

\begin{lemma}[Consistency and implications of finite measurement]\label{lemma:consistency_bernoulli}\ \\
	Denote with $\overline{\mathcal{I}}$ the arms that are sampled only a finite amount of times:
	\begin{align*}
	\overline{\mathcal{I}} = \{ i \in \{ 1, \ldots, k \} : \forall n, T_{n,i} < \infty \}.
	\end{align*}
	If $\overline{\mathcal{I}}$ is empty, $a_{n,i}$ converges almost surely to $1$ when $i = I^\star$ and to $0$ when $i \neq I^\star$. If $\overline{\mathcal{I}}$ is non-empty, then for every $i \in \overline{\mathcal{I}}$, we have $\liminf_{n \rightarrow \infty} a_{n,i} > 0$ a.s.
\end{lemma}

Now we can show $1.$ of Lemma~\ref{lemma:optimal_prop_istar_bernoulli}: we show that under \TTTS, for each $j \in A$, we have $\sum_{n \in \NN} T_{n,j} = \infty$. The proof is exactly equal to the proof for Gaussian arms. 

Under \TTTS, we have
\begin{align*}
\psi_{n,i} = a_{n,i} \left( \beta + (1-\beta) \sum_{j \neq i} \frac{a_{n,j}}{1- a_{n,j}} \right),
\end{align*}
so $ \psi_{n,i}  \geq \beta a_{n,i}$, therefore, by Lemma~\ref{lemma:consistency_gaussian}, if $i \in \overline{\mathcal{I}}$, then $\liminf a_{n,i} > 0$ implies that $\sum_n \psi_{n,i} = \infty$. By Lemma~\ref{lemma:link}, we then must have that $\lim_{n \rightarrow \infty} T_{n,i} = \infty$ as well: contradiction.  Thus, $\lim_{n \rightarrow \infty} T_{n,i} = \infty$ for all $i$, and we conclude that $a_{n,I^\star} \rightarrow 1$, by Lemma~\ref{lemma:consistency_gaussian}. 

Lastly we prove point $3.$ of Lemma~\ref{lemma:optimal_prop_istar_bernoulli}. For \TTTS with parameter $\beta$, the above implies that $\overline{\psi}_{n, I^\star} \rightarrow \beta$, and since we have a bound on $| T_{n,i} / n - \overline{\psi}_{n, i} |$ in Lemma~\ref{lemma:link}, we have $T_{n, I^\star} / n \rightarrow \beta$ as well.

\end{proof}

\paragraph{Step 2: Controlling the over-allocation of sub-optimal arms.}
Following the proof for the Gaussian case again, we can establish a consequence of the convergence of $T_{n,I^\star} / n$ to $\beta$ : if an arm is sampled more often than its optimal proportion, the posterior probability of this arm to be optimal is reduced compared to that of other sub-optimal arms. We can prove this by using ingredients from the proof of the lower bound in Theorem~\ref{thm:bernoulli_lower_bound}.

\begin{lemma}[Over-allocation implies negligible probability]\label{lemma:over_allocation_bernoulli}\footnote{analogue of Lemma 13 of \cite{russo2016ttts}}\ \\
	Fix any $\xi > 0$ and $j \neq I^\star$. With probability 1, under any allocation rule, if $T_{n,I^\star}/n \rightarrow \beta$, there exist $\xi' > 0$ and a sequence $\epsilon_n$ with $\epsilon_n \rightarrow 0$ such that for any $n \in \mathbb{N}$, 
	\begin{align*}
	\frac{T_{n,j}}{n} \geq \omega_j^\beta + \xi \implies \frac{a_{n,j}}{\max_{i \neq I^\star} a_{n,i}} \leq e^{-n (\xi' + \epsilon_n)}.
	\end{align*}
\end{lemma}

\begin{proof}
By Theorem~\ref{thm:bernoulli_lower_bound}, we have, as $T_{n, I^\star} / n \rightarrow \beta$, 
\begin{align*}
\limsup_{n \rightarrow \infty} - \frac{1}{n} \log \left(\max_{i \neq I^\star} a_{n,i} \right) \leq \Gamma_{\beta}^\star,
\end{align*}
since $\max_{i \neq I^\star} a_{n,i} \leq 1 - a_{n,I^\star}$. We also have from Lemma~\ref{lemma:binomial_tail} a deviation inequality, so that we can establish the following logarithmic equivalence:
\begin{align*}
a_{n,j} \leq \Pi_{n}(\theta_j \geq \theta_{I^\star} ) \doteq \exp \left\lbrace - n C_{j} \left(w_{n, I^\star}, \omega_{n,j} \right) \right\rbrace \doteq \exp \left\lbrace - n C_{j} \left(\beta, \omega_{n,j} \right) \right\rbrace,
\end{align*}
where we denote $\omega_{n,j} \triangleq \frac{T_{n,j}}{n}$.
We can combine these results, which implies that there exists a non-negative sequence $\epsilon_n \rightarrow 0$ such that
\begin{align*}
\frac{a_{n,j}}{\max_{i \neq I^\star} a_{n,i}} \leq \frac{ \exp \left\lbrace -n C_{j} \left(\beta, \omega_{n,j} \right) - \epsilon_n / 2 \right\rbrace}{\exp \left \lbrace - n ( \Gamma_{\beta}^\star + \epsilon / 2 ) \right\rbrace} 
= \exp \left\lbrace -n \left( C_{j} \left(\beta, \omega_{n,j} \right) - \Gamma_{\beta}^\star \right) - \epsilon_n \right\rbrace.
\end{align*}
We know that $C_{j} \left(\beta, \omega_j^\beta \right)$ is strictly increasing in $\omega^\beta_j$, and $C_{j} \left(\beta, \omega_j^\beta \right) = \Gamma_{\beta}^\star$, thus, there exists some $\xi' > 0$ such that 
\begin{align*}
\omega_{n,j} \geq \omega_j^\beta + \xi \implies C_{j} \left(\beta, \omega_{n,j} \right) - \Gamma_{\beta}^\star > \xi'.
\end{align*}
\end{proof}

\paragraph{Step 3: $\overline{\psi}_{n,i}$ converges to $\omega_i^\beta$ for all arms.} To establish the convergence of the allocation effort of all arms, we rely on the same sufficient condition used in the analysis of~\cite{russo2016ttts}, restated above in Lemma~\ref{lemma:sufficient_optimality}, and we will restate it here again for convenience.
	
\begin{lemma}[Sufficient condition for optimality]\ \\
	Consider any adaptive allocation rule. If 
		\begin{align}
		&\overline{\psi}_{n, I^\star} \rightarrow \beta, \,\,\, \text{ and } \,\,\,
		\sum_{n \in \mathbb{N}} \psi_{n,j} \bm{1} \left\lbrace \overline{\psi}_{n,j} \geq \omega_j^\beta + \xi \right\rbrace < \infty, \,\, \forall j \neq I^\star, \xi > 0,
		\end{align}
	then $\overline{\psi}_{n} \rightarrow \psi^\beta$.
\end{lemma}

First, note that from Lemma~\ref{lemma:optimal_prop_istar_bernoulli} we know that $\frac{T_{n,I^\star}}{n} \rightarrow \beta$, and by Lemma~\ref{lemma:link} this implies $\overline{\psi}_{n, I^\star} \rightarrow \beta$, hence we can use the lemma above to prove convergence to the optimal proportions. This proof is already given in Step~3 of the proof for the Gaussian case, and since it does not depend on the specifics of the Gaussian case, except for invoking Lemma~\ref{lemma:consistency_gaussian} (consistency), which for the Bernoulli case we replace by Lemma~\ref{lemma:consistency_bernoulli}, it gives a proof for the Bernoulli case as well. We conclude that \eqref{eq:sufficient condition for optimality} holds, and the convergence to the optimal proportions follows by Lemma~\ref{lemma:sufficient_optimality}.

\subsection{Proof of auxiliary lemmas}\label{app:posterior_beta.aux}

\paragraph{Proof of Lemma~\ref{lemma:binomial_tail}}

	\begin{align*}
		\PP{X > Y} = \EE{ \PP{X > Y | Y}}
		&\leq \EE{\1{\{ Y < \frac{a-1}{a+b-1}  \} } + \1{\{ Y \geq \frac{a-1}{a+b-1} \} } \PP{X > Y | Y}  } \\
		&\leq \expp{- (c+d-1) kl\left(\frac{c-1}{c+d-1}; \frac{a-1}{a+b-1}\right) } \\
		&+ \EE{ \expp{-(a + b -1) kl\left(\frac{a-1}{a+b-1}; Y\right)}\1{\{ Y \geq \frac{a-1}{a+b-1} \} } },
	\end{align*}
	Using the Beta-Binomial trick in the second inequality.	Then we have (call the second half A)
	\begin{align*}
	A &\leq \EE{\1{\{ \frac{a-1}{a+b-1} \leq Y \leq \frac{c-1}{c+d-1}\}}}  \expp{-(a + b -1) kl\left(\frac{a-1}{a+b-1}; Y\right)}\\
	&+ \expp{-(a+b-1) kl\left( \frac{a-1}{a+b-1} ; \frac{c-1}{c+d-1} \right)}\\
	\end{align*}
	(call the first half B). Denote with $f$ the density of $Y$, then
	\begin{align*}
	B &= \int_{\frac{a-1}{a+b-1}}^{\frac{c-1}{c+d-1}} \expp{-(a+b-1) kl\left( \frac{a-1}{a+b-1}; y\right)} f(y) \diff{y}.
	\end{align*}
	Via integration by parts we obtain
	\begin{align*}
	B &= \left[ \expp{-(a+b-1) kl\left( \frac{a-1}{a+b-1}; y\right)} \PP{Y \leq y} \right]_{\frac{a-1}{a+b-1}}^{\frac{c-1}{c+d-1}}\\
	&+ \int_{\frac{a-1}{a+b-1}}^{\frac{c-1}{c+d-1}} (a+b-1) \frac{\d{}}{\d{y}} kl\left(\frac{a-1}{a+b-1} ; y\right) \expp{-C_{a,b}(y) } P(Y \leq y) \diff{y}\\
	&\leq \int_{\frac{a-1}{a+b-1}}^{\frac{c-1}{c+d-1}} (a+b-1) \frac{\d{}}{\d{y}} kl\left(\frac{a-1}{a+b-1} ; y\right)
	\expp{-(C_{a,b}(y)+C_{c,d}(y)) } \diff{y}\\
	&+ \expp{-(a+b-1)kl\left(\frac{a-1}{a+b-1}; \frac{c-1}{c+d-1}\right)},
	\end{align*}
	where the first inequality uses the Binomial trick again. Let
	\begin{align*}
	C &= \inf_{\frac{a-1}{a+b-1} \leq y \leq \frac{c-1}{c+d-1}} (a+b-1) kl\left(\frac{a-1}{a+b-1}; y\right) + (c+d-1) kl\left(\frac{c-1}{c+d-1} ; y\right)  = \inf_{\frac{a-1}{a+b-1} \leq y \leq \frac{c-1}{c+d-1}} C_{a,b}(y)+C_{c,d}(y),
	\end{align*}
	then note that in particular we have
	\begin{align*}
	C &\leq \min \left( (a+b-1) kl\left(\frac{a-1}{a+b-1}; \frac{c-1}{c+d-1}\right), (c+d-1) kl \left(\frac{c-1}{c+d-1} ; \frac{a-1}{a+b-1}\right) \right)\\
	  &= \min \left(C_{a,b}\left(\frac{c-1}{c+d-1}\right), C_{c,d}\left(\frac{a-1}{a+b-1}\right)\right).
	\end{align*}
	Then
	\begin{align*}
	B &\leq e^{-C} \int_{\frac{a-1}{a+b-1}}^{\frac{c-1}{c+d-1}} (a+b-1) \frac{\diff{}}{\diff{y}} kl (\frac{a-1}{a+b-1}; y) \diff{y}  + e^{-C}
	= \left[ (a+b-1) kl\left(\frac{a-1}{a+b-1}; \frac{c-1}{c+d-1}\right)  + 1 \right] e^{-C}.
	\end{align*}
	Thus we have
	\begin{align*}
	\PP{X > Y} \leq \left( 3 + (a+b-1) kl\left(\frac{a-1}{a+b-1}; \frac{c-1}{c+d-1}\right)  \right) e^{-C}.
	\end{align*}
	By symmetry, we have
	\begin{align*}
	\PP{X > Y} \leq \left(3 + \min \left(C_{a,b}\left(\frac{c-1}{c+d-1}\right), C_{c,d}\left(\frac{a-1}{a+b-1}\right)\right)\right) e^{-C},
	\end{align*}
	where
	\[
	    C = \inf_{\frac{a-1}{a+b-1} \leq y \leq \frac{c-1}{c+d-1}} (a+b-1) kl\left(\frac{a-1}{a+b-1}; y\right) + (c+d-1) kl\left(\frac{c-1}{c+d-1} ; y\right).
	\]

\hfill\BlackBox\\[2mm]

\paragraph{Proof of Lemma~\ref{lemma:consistency_bernoulli}} 

Let $\overline{\cI}$ be empty, then we have $\mu_{\infty, i} \triangleq \lim_{n \rightarrow \infty} \mu_{n,i} = \mu_i$. The posterior variance is
\begin{align*}
\sigma_{n,i}^2 &= \frac{\alpha_{n,i}\beta_{n,i}}{(\alpha_{n,i}+ \beta_{n,i})^2 (\alpha_{n,i} + \beta_{n,i} + 1)} 
= \frac{(1 + \sum_{\ell=1}^{n-1} \1{\{ I_{\ell} = i \}} Y_{\ell,I_{\ell}}) (1 + T_{n,i} - \sum_{\ell=1}^{n-1} \1{\{ I_{\ell} = i \}} Y_{\ell,I_{\ell}})  }{ (2 + T_{n,i})^2 (2 + T_{n,i} +1 )   },
\end{align*}
We see that when $\overline{\cI}$ is empty, we have $\sigma_{\infty, i}^2 \triangleq \lim_{n \rightarrow \infty} \sigma_{n,i}^2 = 0$, i.e., the posterior is concentrated. \\

When $T_{n,i} = 0$ for some $i \in \cA$, the empirical mean for that arm equals to the prior mean $\mu_{n,i} = \alpha_{1,i} / (\alpha_{1,i} + \beta_{1,i})$, and the variance is strictly positive: $\sigma^2_{n,i} = (\alpha_{n,i}\beta_{n,i}) / \left((\alpha_{1,i}+ \beta_{1,i})^2 (\alpha_{1,i} + \beta_{1,i} + 1) \right) > 0$. When $\overline{\cI}$ is not empty, then for every $i \in \overline{\cI}$ we have $\sigma_{\infty, i}^2 > 0$, and $\alpha_{\infty, i} \in (0,1)$, implying $\alpha_{\infty, I^\star} < 1$, hence the posterior is not concentrated. 

\hfill\BlackBox\\[2mm]

\end{document}